\documentclass{article}
\usepackage{graphicx}
\usepackage[utf8]{inputenc}
\usepackage{amsmath}
\usepackage{amsthm}
\usepackage{amsfonts}
\usepackage{amssymb, bbold}
\usepackage{xcolor}
\usepackage{xfrac}
\usepackage{hyperref}
\usepackage{thmtools}
\usepackage{caption}

\graphicspath{{figures/}}
\usepackage{color}   %May be necessary if you want to color links
\newcommand{\mumon}{\mu^{\mathrm{separated}}}
\newcommand{\mubig}{\mu^{\mathrm{S_n}}}
\newcommand{\muTwo}{\mu^{SO(2)}}
\newcommand{\muThree}{\mu^{SO(3)}}
\newcommand{\muThreeStable}{\mu^{SO(3) \mathrm{ stable}}}

\newcommand{\muF}{\mu_{[\cdot]}}
\newcommand{\CWW}{robust~} % changed from continuous
\newcommand{\cCWW}{Robust~}

\newcommand{\stab}[1]{G_{#1}}

\newcommand{\ZZ}{\mathbb{Z}}

\newcommand{\ty}{\tilde{y}}

\newtheorem{thm}{Theorem}[section]
\newtheorem{cor}[thm]{Corollary}

\newtheorem{prop}[thm]{Proposition}

\theoremstyle{definition}
\newtheorem{remark}[thm]{Remark}
\newtheorem{defn}[thm]{Definition}

\newtheorem{example}[thm]{Example}

\newcommand{\I}{\mathcal{I}}
\newcommand{\N}{\mathcal{N}}

\newcommand{\F}{\mathcal{F}}
\newcommand{\E}{\mathcal{E}}

\newcommand{\RR}{\mathbb{R}}
\newcommand{\CC}{\mathbb{C}}
\newcommand{\MM}{\mathbb{M}}

\newcommand{\twopartdef}[4]
{
	\left\{
		\begin{array}{ll}
			#1 & \mbox{if } #2 \\
			#3 & \mbox{if } #4
		\end{array}
	\right.
}

\newcommand{\Iavg}{\I_{\mathrm{avg}}}
\newcommand{\Eavg}{\E_{\mathrm{avg}}}
\newcommand{\Ican}{\I_{\mathrm{can}}}
\newcommand{\Ecan}{\E_{\mathrm{can}}}
\newcommand{\Iframe}{\I_{\mathrm{frame}}}
\newcommand{\Eframe}{\E_{\mathrm{frame}}}
\newcommand{\Iw}{\I_{\mathrm{weighted}}}
\newcommand{\Ew}{\E_{\mathrm{weighted}}}

\newcommand{\Rdistinct}{\RR^{d\times n}_{\mathrm{distinct}} }

\newcommand{\Cequi}{C_{\mathrm{equi}}}
\newcommand{\Fequi}{F_{\mathrm{equi}}}
\renewcommand{\Finv}{F_{\mathrm{inv}}}
\newcommand{\Cinv}{C_{\mathrm{inv}}}
\usepackage{microtype}
\usepackage{graphicx}
\usepackage{subfigure}
\usepackage{booktabs} % for professional tables

\usepackage{hyperref}

 \usepackage[accepted]{icml2024}

\usepackage{amsmath}
\usepackage{amssymb}
\usepackage{mathtools}
\usepackage{amsthm}

\usepackage[capitalize,noabbrev]{cleveref}

\usepackage[textsize=tiny]{todonotes}

\icmltitlerunning{Equivariant Frames and the Impossibility of Continuous Canonicalization} % to replace
\begin{document}

\twocolumn[
\icmltitle{Equivariant Frames and the Impossibility of Continuous Canonicalization}
% Towards Robust, Equivariant Frames: Possibilities and Limitations

% Alternative shorter title: Weighted Frames for Smooth Equivariance
% or, medium length: Weighted Frames for Equivariant Learning: Possibilities and Limitations



\icmlsetsymbol{equal}{*}

\begin{icmlauthorlist}
\icmlauthor{Nadav Dym}{equal,aa}
\icmlauthor{Hannah Lawrence}{equal,bb}
\icmlauthor{Jonathan W Siegel}{equal,cc}
\end{icmlauthorlist}

\icmlaffiliation{aa}{Faculty of Mathematics, Faculty of Computer Science, Technion, Israel}
\icmlaffiliation{bb}{Department of Electrical Engineering and Computer Science, MIT, MA, USA}
\icmlaffiliation{cc}{Department of Mathematics, Texas A\&M University, TX, USA}

\icmlcorrespondingauthor{Nadav Dym}{nadavdym@technion.ac.il}

\icmlkeywords{Machine Learning, ICML}

\vskip 0.3in
]



\printAffiliationsAndNotice{\icmlEqualContribution} % otherwise use the standard text.

\begin{abstract}

Canonicalization provides an architecture-agnostic method for enforcing equivariance, with generalizations such as frame-averaging recently gaining prominence as a lightweight and flexible alternative to equivariant architectures. Recent works have found an empirical benefit to using probabilistic frames instead, which learn weighted distributions over group elements. 
In this work, we provide strong theoretical justification for this phenomenon: for commonly-used groups, there is no efficiently computable choice of frame that preserves continuity of the function being averaged. In other words, unweighted frame-averaging can turn a smooth, non-symmetric function into a discontinuous, symmetric function. To address this fundamental robustness problem, we formally define and construct \emph{weighted} frames, which provably preserve continuity, and demonstrate their utility by constructing efficient and continuous weighted frames for the actions of $SO(d)$, $O(d)$, and $S_n$ on point clouds. 
%Finally, we generalize weighted frames to efficiently handle self-symmetries. 

% \hl{Next sentence may be slightly out of place - depends if we want to include these results} Theoretically, we use such weighted frame ideas to prove improved approximation rates for smooth, invariant functions. 
%are often not smooth in practice. 

%Empirically, such approaches outperform vanilla, unweighted frame-averaging. % / their non-distributional counterparts. 

\end{abstract}
%\hl{We also define an appropriate notion of } % symmetry breaking sentence?

\section{Introduction}
%\nd{@Hannah or self: Explain re equivariance importance, continuity, and universality and end with}

\iffalse
\begin{figure*}
     \centering
     \includegraphics[width=0.6\linewidth]{canon_and_fa_fig.pdf}
     \caption{Canonicalization and its generalization to frame averaging. Under group transformation of the input, either its canonicalized version, or the set of inputs transformed by the frame, are invariant. }
        \label{fig:main}
     \end{figure*}
    \fi

Equivariance has emerged in recent years as a cornerstone of geometric deep learning, with widespread adoption in domains including biology, chemistry, and graphs \citep{jumper2021highly, corso2023diffdock, liao2023equiformerv2, satorras2021en, frasca2022subgraph}. This fundamental idea --- incorporating known data symmetries into a learning pipeline --- often enables improved generalization and sample complexity, both in theory \citep{petrache2024approximation,mei2021learning, bietti2021on, elesedy2021provably} and in practice \citep{batzner2022e3, liao2023equiformerv2}.

The genesis of equivariant learning focused on equivariant \emph{architectures}, i.e.  custom parametric function families containing only functions with the desired symmetries. However, equivariant architectures must be custom-designed for each group action, which reduces the transferability of 
%any existing 
engineering best practices. % significantly less applicable. 
Moreover, the building blocks of many equivariant architectures (such as tensor products) are computationally intensive \citep{passaro2023reducing}. %, and only recently have nascent works begun to develop custom algorithmic and/or GPU accelerations \hl{cite saro + gpu paper}. 

In light of these difficulties, the more lightweight and modular approach of frame-averaging has received renewed attention. Frame-averaging \citep{puny2021frame} %(which includes both canonicalization and group averaging as special cases) 
harnesses a generic neural network $f$ to create an equivariant framework by averaging the network's output over input transformations. It is a direct extension of group-averaging (also known as the Reynolds operator), whereby $f$ is made invariant by averaging over all input transformations from a group $G$:
\begin{equation}\label{eq:invar_renolds}
    \Iavg[f](v) = \int_G f(g^{-1}v)dg.
\end{equation}
However, while group averaging scales with $|G|$ (which can be large or infinite), frame-averaging can enjoy computational advantages by % $|G|$ is large or infinite, prompting alternatives like frame-averaging. Here, instead of averaging over all of $G$, one 
averaging over only an input-dependent subset of $G$. %, $\mathcal{F}(x)$. 
% $$ $$
% If $\mathcal{F}(x)$ is chosen appropriately, the result of this operation is still invariant.
A notable special case of frame-averaging is canonicalization (Figure \ref{fig:canon_and_venn}), which ``averages'' over a \emph{single, canonical} group transformation per point. 
% \begin{equation}\label{eq:canonicalization} \Ican[f](v)=f(g(v)^{-1}v).\end{equation} 
Intuitively, canonicalization is %the canonicalization $g(v)^{-1}v$ is 
a standardization of the input data, such as centering a point cloud with respect to translations. %Such an idea has existed naturally in science for decades, such as the concept of ``conventions'' from crystallography \hl{add ref + double check; others?}. 
Frame-averaging methods are universal (so long as $f$ is universal), in the sense that they can approximate all continuous equivariant functions, and are projections\footnote{Recall that a projection $\mathcal{P}$ need only satisfy $\mathcal{P}(\mathcal{P}f)=\mathcal{P}f$.} % A stronger property is \emph{orthogonal} projection, which posits that $\mathcal{P}f$ minimizes $\|g - f\|$ (for some functional norm) over all $g$ in the vector space onto which one is projecting. The Reynolds operator is an orthogonal projection for the $L_2$-norm \hl{any natural choice of functional norm}, whereas the frame-based projections are not.}
onto the space of invariant functions. 
% which we explain in more detail in \hl{ref section}. 
% \hl{actually, will define frames here, mention canon/reynolds/efficiency in 1-2 sentences instead of longer later section}
%Group averaging is
%\hl{where to put canonicalization related work?}
Recent approaches include using both fixed \citep{duval2023fae, du2022complete, du2023new},  and learned \citep{zhang2019permutations, 
kaba2022learned, luo2022pointcloud} frames and canonicalizations. %, and of particular recent interest is their compatibility with large foundation models \citep{mondal2023adaptation, kim2023probabilistic}. 
Despite their simplicity, however, it seems that such approaches have generally not yet supplanted %performed well enough to supplant 
popular equivariant architectures in applications. 

In this work, we unearth an insidious problem with frames, which may shed light on their slow adoption in applications: they very often induce discontinuity. The absence of a continuous canonicalization for permutations was already observed by \citet{zhang2020fspool}. We prove that there are no continuous canonicalization for rotations, either. Moreover, we  show that the only continuity preserving frame for permutations is the $(n!)$-sized Reynolds operator, and that there is no continuity preserving frame with finite size for rotations in two dimensions.
In other words, even if the generic network $f$ is continuous, the frame-averaged function may not be! This constitutes a significant lack of robustness, in which one can slightly perturb an input, and have the predicted output entirely change.

To address this issue, we generalize \citet{puny2021frame}'s  definition of frames %in  %Puny's definition of frames 
in two ways. First, we define \emph{weighted frames}, where group elements are assigned non-uniform, input dependent weights. Second, we observe that their definition of frame equivariance necessitates very large frames at points with large stabilizer. To avoid this, we define weak equivariance, which relaxes the notion of equivariance at points with non-trivial stabilizers. Finally, we define a natural notion of continuity for these generalized frames, and name the resulting frames \textbf{\CWW frames}. We show that invariant projection operators induced by robust frames \emph{always preserve continuity}. Finally, we show that robust frames of  moderate size can be constructed for group actions of interest, including the action of permutations (where small unweighted frames cannot preserve continuity). 

  Serendipitously, our results %thus 
  coincide with several quite recent works, which demonstrate a significant benefit to using probabilistic or weighted frames quite similar to ours \citep{kim2023probabilistic, mondal2023adaptation, pozdnyakov2023smooth}. Thus, our results %for weighted frames 
provide both a theoretical framework and strong justification for empirically successful approaches which learn distributions over the group, while also suggesting future avenues in practice. 

%Figure \ref{fig:canon_and_venn} summarizes %gives a visual summary of 
%some of our discussion so far. 
We visualize canonicalization and frame-averaging on the left of Figure \ref{fig:canon_and_venn}, while the right %The left hand side visualizes canonicalization and frame averaging. The right hand side 
shows the % summarizes the %containment 
relation %ship 
between canonicalization, averaging, frames, and weighted frames.

  \begin{figure*}
     \centering
     \includegraphics[width=0.75\linewidth]{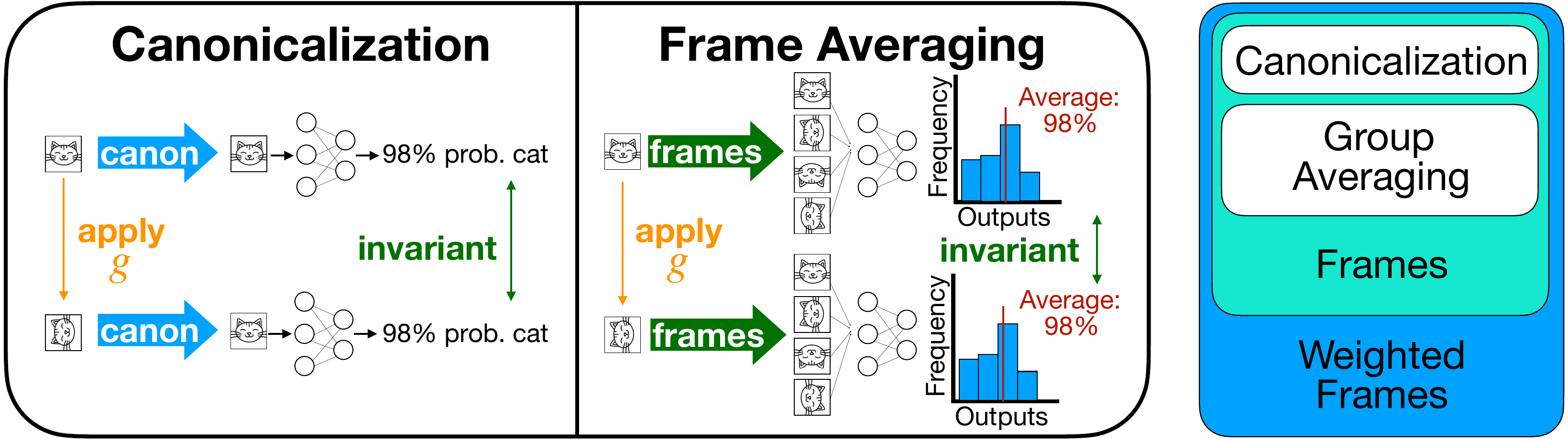}
     \caption{\textbf{Left:} Canonicalization and its generalization to frame averaging. Under group transformation of the input, both its canonicalization and the set of inputs transformed by the frame are invariant. \textbf{Right:} %Categorization of the concepts discussed in this paper. 
     Group averaging and group canonicalization are special cases of frames of maximal and minimal size, respectively. Frames, in turn, are a special case of weighted frames.}
        \label{fig:canon_and_venn}
     \end{figure*}

\subsection{Our Contributions} 
% In \cref{sub:requirements}, we described three common-sense requirements for equivariant models.
% While the models obtained from frame-averaging are typically equivariant and universal, 
% %posess two of these requirements (equivariance and universality), 
% continuity has not been rigorously discussed.
% In many instances, the frame and canonicalization constructions in the literature in fact \emph{do not} preserve continuity, as noted in \nd{add refs}, and it is not clear whether such continuity-preserving constructions exist. 
% The goal of this paper is to systematically study this question, and to suggest practical alternatives when the answer is no. Our main results are:
To summarize, % our discussion so far, 
our main contributions are: 
\begin{enumerate}
\item \textbf{Limitations of canonicalization:} Using tools from algebraic topology, we show that %in many cases of interest to the equivariant learning community (\hl{we could enumerate some cases here}), 
\emph{a continuous canonicalization does not exist} for $S_n$, $SO(d)$, and $O(d)$. 
\item \textbf{Limitations of frames:}
We show that for finite groups acting freely on a connected space, the only  frame which preserves continuity is the %full 
Reynolds operator. In particular, this is the case for the group of permutations $S_n$ acting on $d\geq 2$ dimensional point clouds. For the infinite group $SO(2)$ we show that there is no continuity preserving frame of finite cardinality. %\nd{only $d=2$?}. %frame. 
\item \textbf{\cCWW frames:} We define weighted frames as an alternative to standard (unweighted) frames, and introduce notions of weak equivariance  and continuity. We call weighted frames satisfying both criteria \emph{\CWW frames}. %We show that, 
Unlike unweighted frames, the invariant projection operators 
induced by \CWW frames \emph{always} preserve continuity. With some care, this can be generalized to the equivariant case as well. %\hl{mention symmetry-breaking}
\item \textbf{Examples of \CWW frames:} We give constructions of \CWW frames of moderate, polynomial size for $S_n$, $SO(d)$, and $O(d)$. %concrete examples. 
For $S_n$, we also provide complementary lower bounds on the size of any \CWW frame, which is directly related to the computational efficiency of implementation.
\end{enumerate}

%\subsection{Point cloud examples}
To give a more concrete picture, we summarize our results  restricted to a large family of examples: the various group actions of ``geometric groups'' on point clouds. A point cloud is a matrix $X\in \RR^{d\times n}$, with columns denoted $(x_1,\ldots,x_n) $. Point clouds arise in many applications, ranging from graphics and computer vision (as representations of objects and scenes in the physical world), to chemistry and biology (as representations of molecular systems). Group actions of interest here include the application of translation, rotations, orthogonal transformations, or permutations. %, to the point clouds. 

Our results for point clouds and these group actions are described in  Table \ref{tab:main_results}. The action of translations has trivial continuous canonicalizations. When $n,d>1$, the action of permutations admits no continuous canonicalization, nor any unweighted frame which preserves continuity besides the full Reynolds frame. In contrast, robust frames whose complexity is linear in the input dimension can be constructed, at least for % when focusing only on 
the subset $\Rdistinct$
 (on which the permutation group acts freely). When considering robust frames defined on the whole space, we require a larger cardinality, but it is still significantly smaller than the full $n!$ cardinality required by unweighted frames.

 For the action of rotations on $\RR^{d\times n}$, we prove that there is no continuous canonicalization when $n \geq d$. Moreover, when $d=2$, we show that no continuity preserving frame of finite cardinality exists, either. In contrast, robust frames exist with a cardinality which (when $d$ is small and $n\rightarrow \infty$) scales like $ \sim n^{d-1}$. 

\begin{table*}[t]
    \centering
    \begin{tabular}{|c|c|c|c|c|c|} \hline 
         \textbf{Action}&  Translation&    Permutation & Permutation&  $SO(d)$ & $O(d)$\\ \hline 
         \textbf{Domain}&  $\RR^{d\times n}$ &  $\Rdistinct$ &$\RR^{d\times n}$&  $\RR^{d\times n} $& $\RR^{d\times n}$\\ \hline 
         Canonicalization&  yes&  no &no&  no if $n\geq d$& no if $n >  d$\\ \hline 
         Frame&  1&    $N=n!$&$N=n!$&  $N = \infty$ if $n \geq d$ & $N > 1$ if $n > d$\\ \hline 
 Weighted frame& 1&  $N\leq (n-1)d$&$\frac{n}{2}< N\leq  n^{2(d-1)} $& $N\leq (d-1)!\binom{n}{d-1}$ & $N\leq 2(d-1)!\binom{n}{d-1}$\\ \hline
    \end{tabular}
    \caption{Summary of main results. For various group actions, we show lower and upper bounds on the minimal cardinality $N$ for which a continuity-preserving frame or weighted frame exists, and  whether a continuous canonicalization exists (in which case $N=1$). The $N=\infty$ result for unweighted $SO(d)$ frames is proven only when $d=2$ (although we conjecture it holds for general $d \geq 2$ as well).  }
    \label{tab:main_results}
\end{table*}
%  \hl{Someone reading this may wonder about all three of these symmetries (translation, rotation, permutation) at once -- we should say something about what our results imply for this case.}

\subsection{Paper structure}
In Section \ref{sec:projections_canonicalization_and_frames}, we 
introduce useful mathematical preliminaries, including criteria for an equivariant projection operator to preserve continuity. %the notion of a bounded, equivariant, continuity-preserving projection operator.
% Section \ref{sec:canonicalization_and_frames} establishes 
We then establish impossibility results for continuous canonicalizations and frames. As a solution, Section \ref{sec:weighted_frames} defines weighted frames, and establishes criteria under which they are \CWW (i.e. preserve continuity). The following two sections, Section \ref{sec:weighted_frames_for_permutations} and Section \ref{sec:weighted_frames_for_rotations}, give explicit continuity-preserving \CWW frame constructions for invariance under $S_n$ and $SO(d)$, respectively. Finally, Section \ref{sec:inv2eq} discusses extensions to equivariance.

\section{Projections, canonicalization, and frames}\label{sec:projections_canonicalization_and_frames}
%\section{Projections and universality}\label{sec:projections_and_universality}
%\subsection{Mathematical Preliminaries}
\paragraph{Preliminaries} Unless stated otherwise, throughout the paper we consider compact groups $G$ acting linearly and continuously on (typically) finite dimensional real vector spaces $V$ and $W$, or else on subsets of $V$ and $W$  closed under the action of $G$ (see Appendix \ref{sec:additional_background} for a formal definition) . We will work with the groups $S_n$ of $n$-dimensional permutations, $O(d)$ of orthogonal matrices ($M\in\mathbb{R}^{d\times d }\: s.t. \: M^TM = I$), and $SO(d)$ of rotations ($M\in O(d)\: s.t. \: \text{det}(M)=1$).%  $d \times d$ matrices of determinant $1$),  %$d \times d$ matrices of determinant $\pm 1$). 

A function $f:V \rightarrow W$ is equivariant if $f(gv)=gf(v) \: \forall v \in V, g \in G$, and invariant (a special case where $G$ acts trivially on $W$) if $f(gv)=f(v) \forall v \in V, g \in G$. A compact group $G$ admits a unique Borel probability measure which is both left and right invariant. As in \eqref{eq:invar_renolds},  we denote integration according to this measure by $\int dg $. The orbit of $v\in V$ is 
$[v] := \{ gv: \: g \in G  \}$, and the stabilizer is $\stab{v}=\{g\in G| \, gv=v \}$. 
%The stabilizer of $v\in V$ with respect to the action of $G$ is 
% $$\stab{v}=\{g\in G| \, gv=v \} .$$
We say that $v$ has a trivial stabilizer if $\stab{v}=\{e\}$, where $e \in G$ denotes the identity element. $G$ acts on $V$ \emph{freely} if $\stab{v}=\{e\} \: \forall v \in V$. %all elements in $\M$ have a trivial stabilizer.
For $H$ a subgroup of $G$, $G / H$ denotes the left cosets of $H$ in $G$. 

 %Let $(G,V)$ and $(G,W)$ be modules. 
For $G$ acting on spaces $V$ and $W$ as described, let $F(V,W)$ denote the space of functions from $V$ to $W$, and let $C(V,W)$ denote the subset of these functions which are also continuous. Let $\Fequi(V,W)$ (and $\Cequi(V,W)$) denote (continuous) functions which are also \emph{equivariant}. Notions of denseness and boundedness in $C(V,W)$,  described in this paper, are with respect to the topology of uniform convergence on compact subsets of $V$. %, meaning that $f(gv)=g(fv)$ for all $v\in V$ and $g\in G$.  
Full proofs of all claims in the paper are given in the Appendix. 
\subsection{Projection operators} % via averaging}
Group averaging, as well as the cheaper alternatives summarized in Figure \ref{fig:canon_and_venn}, achieve equivariant models via \emph{equivariant projection operators}, a notion we now define formally.

%\hl{Seems like the defn environment fails when cleveref is imported}
\begin{defn}[BEC operator]\label{def:proj}
Let $\E:F(V,W)\to F(V,W) $ be a linear operator. We say that $\E$ is a
\begin{enumerate}
	\item \textbf{Bounded operator} If for every compact $K$,  there exists a positive constant $C_K$ such that  
	$$\max_{v\in K} \|\E[f](v)\|\leq C_K\max_{v\in K}\|f(v)\| $$
	\item \textbf{Equivariant projection operator} for every $f\in F(V,W)$,  $\E(f)$ is equivariant, and moreover if $f$ is equivariant then $\E[f]=f$.
	\item \textbf{Preserves continuity} If $f:V\to W$ is continuous,  $\E[f]$ will also continuous. 
\end{enumerate}
\end{defn}

%The next proposition shows that 

If $\E$ satisfies all three conditions, we call it a BEC %projection 
operator. BEC operators can be used to define universal, equivariant models which preserve continuity. %satisfying all the requirements from Subsection \ref{sub:requirements}

\begin{restatable}{prop}{projuniv}\label{prop:proj2univ}
Let %$(G,V)$ and $(G,W)$ be modules, and let 
$\E:F(V,W)\to F(V,W)$ be a BEC operator, and  $Q\subseteq C(V,W) $ a dense subset . Then 
$\E(Q)=\{\E(q)| ~ q \in Q\}$
contains only continuous equivariant functions, and is dense in $\Cequi(V,W) $ .
\end{restatable}

%\begin{restatable}{prop}{projuniv}\label{prop:proj2univ}
%Let %$(G,V)$ and $(G,W)$ be modules, and let 
%$\E:F(V,W)\to F(V,W)$ be a BEC %operator. Let $K\subseteq V$ be a %compact set and $Q\subseteq C(V,W) $ %be a dense subset of functions with %respect to uniform convergence %\hl{put unif conv. in prelims?} on %$K$. Then $\E$ preserves continuity, %in the sense that the function space 
%$\E(Q)=\{\E(q)| \quad q \in Q\} $
%contains only continuous equivariant %functions, and is dense in %$\Cequi(V,W) $ with respect to %uniform convergence on $K$.
%\end{restatable}

%We note that invariant functions are a special case of equivariant functions where the action of $G$ on $W$ is trivial ($gw=w$ for all $g\in G$, $w\in W$). 

We will often handle the invariant case separately, denoting $\I$ instead of $\E$ and $\Cinv$ or $\Finv$ instead of $\Cequi$ or $\Fequi$. 

We note that equivariant universality can be obtained using other methods, based on approximating all equivariant polynomials \cite{yarotsky2022universal,dym2020universality}, or exploiting separating invariants \cite{gortler_and_dym,hordan2,scalars}, see also \cite{kurlin2023polynomial,widdowson2023recognizing,cahill2024group}. Our focus in this paper is on projection based methods.

Proposition \ref{prop:proj2univ} motivates our search for BEC operators. %Generally speaking, 
Frame-like constructions typically produce bounded and equivariant operators, but preserving continuity is more challenging, and thus is the main focus of this paper.

One simple method to obtain
BEC operators is by group averaging. In the invariant setting this was defined in  Eq. \ref{eq:invar_renolds}, and we can define an equivariant projection operator $\Eavg$ as $\Eavg[f](x)=\int_G gf(g^{-1}x)dg$. It is not difficult to show that $\Eavg$ is a BEC operator, and therefore preserves continuity. However, the disadvantage of group averaging is the complexity of computing such an operator when the group is large. Thus our goal is designing \emph{efficient} \CWW frames. %to design an operator which preserves continuity. 

% \paragraph{Group averaging}
% For modules $(G,V)$ and $(G,W)$, we can define an equivariant projection operator $\Eavg$ via group averaging as follows
% $$ \Eavg[f](x)=\int_G gf(g^{-1}x)dg . $$
% The special case where the action on $W$ is trivial leads to the invariant projection operator defined in \eqref{eq:inv_averaging}. 

% It is not difficult to show that $\Eavg$ is a BEC operator, and as a result that it induces equivariant models which fulfill all the requirements in Subsection \ref{sub:requirements}. As mentioned above, the disadvantage of group averaging is that it can be expensive to compute or reliably approximate when $G$ is a large or infinite group. Henceforth we will discuss the ``cheaper'' alternatives summarized in Figure \ref{fig:ven}, and investigate whether they can be used to define BEC operators. As we will see the main challenge is typically to design an operator which preserves continuity. 

%\section{Canonicalization and frames}\label{sec:canonicalization_and_frames}
\subsection{Canonicalization}
%The idea of \emph{canonicalization} is simple: 
Instead of defining projection operators by averaging over \emph{all} members in the orbit of a point $v$, an \emph{orbit canonicalization} $y:V\to V $ maps all elements in any given orbit to a \emph{unique} orbit element, and ``averages'' over it. %We term this an \emph{orbit canonicalization}.
\begin{defn}[Orbit canonicalization]
%Let $G$ be a group acting  on a set $V$. We say that a 
A function $y:V\to V$ is an orbit canonicalization if 
\begin{enumerate}
\item $y$ is $G$ invariant: 
$y_{gv}=y_v, \quad \forall g\in G, v\in V $
\item $y$ maps $v$ to a member of its orbit: $y_v \in [v], \quad \forall  v \in V$
%$[y_v]=[v], \forall v\in V$
\end{enumerate}
\end{defn}
In the case where the action of $G$ on $V$ is \emph{free},  
the orbit canonicalization naturally induces  a group canonicalization $h:V \to G $. Namely, if $y_v$ is a canonical element in the orbit $[v]$, then the fact that $v$ has a trivial stabilizer implies that there is a unique $h_v \in G$ such that 
$y_v=h_v^{-1}v $. %\footnote{However, if $v \in V$ has non-trivial stabilizer, no group canonicalization exists. We elaborate further on this in Section \hl{add ref}. \nd{not sure I agree} \hl{I was assuming we wanted the group canonicalization to be a frame, i.e. equivariant. If $gv=v$, then by equivariance $gh(v)=h(v) \rightarrow g = e$, so it can't exist. If instead we only ask for a group canonicalization to satisfy Definition 2.4 when you apply the output g to the input v, then I agree it certainly exists. Maybe we need to be careful about this subtlety when we say things like ``canonicalization is a special case of frames''?}} 
%\begin{defn}
%Let $G$ be a group acting  on a set $X$. We say that a function $h:X\to G$ is a group canonicalization function if 
%the map $x\mapsto h_x^{-1}x$ is a canonicalization map. 
%\end{defn}  
%Note that if $h$ is a group canonicalization, the invariance of the induced canonicalization map implies that for every $x\in X$ and $g\in G$,
%$$h^{-1}_x x=h^{-1}_{gx}(gx)=[h^{-1}_{gx}\cdot g]x  $$
%implying that 
%$$g^{-1}\cdot h_{gx}\cdot h^{-1}_x \in \stab{x}. $$
%In particular if the stabilizer of $x$ is trivial we obtain
%\begin{equation}\label{eq:pre_equi}
%h_{gx}=g\cdot h_x, \forall g\in G \text{ and } x\in V \text{ with } \stab{x}=\{e\}.  
%\end{equation}
\begin{example}\label{example:translation}
For the (free) action of $\RR^d$ on $\RR^{d\times n}$ by translation defined above, a simple orbit canonicalization is the map from $X\in \RR^{d\times n}$ to the unique $Y$ whose first coordinate is zero, that is 
$(0,x_2-x_1,\ldots,x_n-x_1)$.
The associated group canonicalization is the map $h_X=x_1 $. 
% For the (free) action of $\RR^d$ on $\RR^{d\times n}$ by translation defined above, a simple orbit canonicalization $Y_X $ is the map that takes every $X\in \RR^{d\times n}$ to the unique $Y$ whose first coordinate is zero, that is 
% $$Y_X=(0,x_2-x_1,\ldots,x_n-x_1) . $$
% The associated group canonicalization is the map $h_X=x_1 $. %For every group element $t\in \RR^d$ we have 
%$$h_{tX}=t+x_1=t+h_X . $$
\end{example}

\paragraph{Canonicalization-based equivariant projection operators.}
%If $(V,G)$ is a module, and 
If $y_v$ is a canonicalization, we can define an invariant projection operator on functions $f:V\to \RR $ via
$$\Ican[f](v)=f(y_v) .$$

 One can easily check that $\Ican[f]$ is invariant, and that if $f$ already is invariant then $\Ican[f]=f $. 

Similarly, if $G$ acts on $V$ freely so that $h_v$ is well-defined, then % and $(W,G)$ is another module,   
we can also write $\Ican[f](v)=f(h_v^{-1}v)$, which shows how canonicalizations are frames of cardinality one. Moreover, we can  define an  equivariant projection operator via  
$\Ecan[f](v)=h_vf(h_v^{-1}v)$.
In general, it is not difficult to see that $\Ican $ and $\Ecan$ are bounded projection operators in the sense of Definition \ref{def:proj}. However, below we will show that even the invariant projection operator often cannot achieve  continuity preservation. %  in many examples of interest.

From the definition of $\Ican$, it is clear that if the canonicalization $h_v$ is continuous, then $\Ican$ preserves continuity. The following %proposition 
shows that these notions are in fact equivalent. %\hl{need linked ref to proof}
\begin{restatable}{prop}{canon}
%Let $(V,G)$ be a module and 
Let $y:V\to V$ be a canonicalization. Then $\Ican:F(V,\RR)\to F(V,\RR)$ preserves continuity if and only if $y $ is continuous. 
\end{restatable} 

%Looking at the definition of the canonicalization equivariant operator, it is clear that it will preserve continuity if the group canonicalization function $h_x$ itself is continuous. However, this will essentially never happen if the action is not free (i.e., if there exist points with a non-trivial stabilizer). 

%\nd{define these assumptions above?}
%\nd{add the concept of non-euclidean sub domains of $V$ and in particular $V_{free}=\{v\in V| \quad \stab{v}=\{e\} \} $}
%\begin{prop}
%Let $G$ be a topological group acting continuously on a topological space $X$. Let $h:X\to G$ be a group canonicalization map. If $x\in X$ has a non-trivial stabilizer, but is an accumulation point of a sequence $x_n \rightarrow x $ such that $\stab{x_n}=\{e\}$ for all $n$, then $h$ is not continuous at $x$.  
%\end{prop}
%\begin{proof}
%If   $h_{x_n}$ does not converge to $h_x$ we are done. Thus we can assume that $h_{x_n} \rightarrow h_x $. Let $g\neq e$ be an element in $\stab{x}$. Then $g.x_n$ converges to $g.x=x$ but $h_{gx_n} $ does not converge to $h_x$ because, using \eqref{eq:pre_equi} we have 
%$$\lim_{n \rightarrow \infty}h_{gx_n}=\lim_{n\rightarrow } gh_{x_n}=g\lim_{n\rightarrow } h_{x_n}=g\cdot h_x\neq h_x .$$
%\end{proof}

For some relatively simple examples, continuous canonicalizations are available: %\hl{seem obvious, but do we need to prove them?}\nd{no IMHO}
% \begin{example}[Continuous orbit canonicalizations]\label{ex:continuous_canons}
% \text{}
\begin{enumerate}
%\item 
%For the action of $\{-1,1\}$ on $\RR$, $y(x)= |x|$.  % has a continuous orbit canonicalization 
% $x\mapsto |x|$. 
%is a continuous orbit canonicalization map. 
\item For the action of $O(d)$ on $\RR^d$, $y(x)=\|x\|e_1$. % has a continuous orbit canonicalization $c(x)=\|x\|e_1$.
\item For the action of $S_n$ on $\RR^n$, $y(x)=sort(x)$. % has a continuous orbit canonicalization $c(x)=sort(x)$.
\end{enumerate}
%\end{example}
\begin{figure*}
     \centering
     \includegraphics[width=\linewidth]{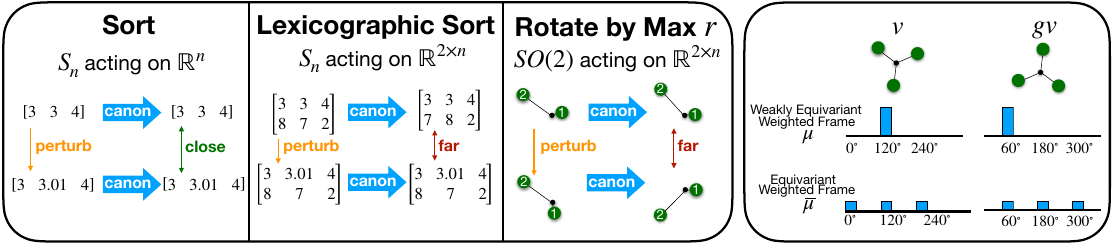}
     \caption{\textbf{Left: }Examples of natural canonicalizations. Canonicalizing $\mathbb{R}^n$ with respect to $S_n$ via sorting is continuous, but the lexicographic generalization to $\mathbb{R}^{2 \times n}$ is not. Similarly, one can canonicalize $2D$ ordered point clouds with respect to $SO(2)$ by applying the rotation that aligns the first node with the positive $x$-axis, but this is discontinuous. \textbf{Right: }Visualization of a weakly equivariant frame $\mu$ for $SO(2)$ acting on an unordered point cloud $v$ with $120^\circ$ self-symmetry, evaluated both at $v$ and at $gv$, a $60^\circ$ rotation of $v$. $\mu_{gv}$ is not a $60^\circ$ rotation of $\mu_v$, but $\overline{\mu}$ is equivariant by definition, so $\overline{\mu}_{gv}$ \emph{is} a $60^\circ$ rotation of $\overline{\mu}_v$. Thanks to the self-symmetry of $v$, $v$ is exactly the same whether rotated by $0^\circ$, $120^\circ$, or $240^\circ$. Thus, invariant projection by $\mu$ or $\overline{\mu}$ has the same result; $\mu$ is simply $3x$ more efficient.}
        \label{fig:discontinuous_canons}
     \end{figure*}
In contrast, the following are ``natural'' orbit canonicalizations, but are not continuous:
% \begin{example}[Discontinuous orbit canonicalizations]\label{ex:disc}
% \text{}
\begin{enumerate}
\item The action of $S_n$ on $\RR^{d\times n}$ with $d>1$ has a canonicalization defined by lexicographical sorting. As shown in Figure \ref{fig:discontinuous_canons}, this is not continuous.
\item The action of $SO(2) \cong S^1$ on $\CC^n$ has a natural orbit canonicalization by rotating $z\in \CC^n $ so that its first coordinate $z_1$ is a real positive number. However, this is not uniquely defined when $z_1=0$, which induces a discontinuity (see also Figure \ref{fig:discontinuous_canons}). 
%\item Several works learn $SO(3)$ canonicalizations by 
\end{enumerate}
% \end{example}
We will show that,  not only are the canonicalizations above %in Example \ref{ex:disc} 
not continuous, but there do not exist \emph{any} continuous canonicalizations for these actions (Theorem \ref{thm:onlyReynold}). 

% For the action of $S_n$ on $\RR^{d\times n}$, this will be a special case of a more general theorem Theorem \ref{thm:onlyReynold} which will also show that there are no frames which lead to a continuty preserving projection operator, except for the full frame. This theorem will be discussed after the introduction of frames. The next subsection will be devoted to the lack of a continuous orbit canonicalization for the action of $SO(d) $ on $\RR^{d\times n}$.

To give a sense for the practical utility of this fact, consider that learned canonicalization methods for $SO(3)$ \citep{kaba2022learned, kim2023probabilistic} involve orthogonalizing learned features via the Gram-Schmidt process. A priori, depending on the learned network, this may or may not be discontinuous (if the learned features are linearly dependent or nearly zero).  
Our results answer definitively that %show that the answer to this question is yes, as %negative since 
no continuous canonicalization is possible. %that there is no continuity-preserving canonicalization anyway.

\paragraph{Examples where there is no continuous canonicalization.}
We now detail a strategy for proving that a continuous canonicalization does not exist, and apply this method to numerous examples of interest.

Our general proof strategy is as follows: we begin by noting that a continuous canonicalization $y:V \rightarrow V$ really defines a continuous map on the \emph{quotient}, $\ty:V/G\rightarrow V$, which is a right inverse of the quotient map $q:V\rightarrow V/G$. % (i.e. $q \circ y = \text{id}$). 
Specifically, continuity and $G$-invariance of $y$ imply that $y$ induces a continuous map on the quotient space $V/G$. Further, if $y$ maps each $v$ to a member of its orbit, it follows that $q\circ \ty = I$ is the identity map on the quotient space $V/G$. 

A canonicalization on $V$ will also be a canonicalization when restricted to a $G$-stable subset $Y\subseteq V $. Thus, for any $G$-stable subset $Y$, a continuous canonicalization $y$ yields a right inverse of the quotient map $Y\rightarrow Y/G$. Our strategy will be to show that for some $G$-stable set $Y$, the quotient map $Y\rightarrow Y/G$ cannot have a right inverse. We will prove this by finding topological invariants, specifically homotopy and homology groups (see for instance \citet{hatcher2002}), which contradict the existence of such a right inverse $\ty$.

We now apply this method to prove the impossibility of continuous canonicalization for several groups. We begin with a well known example which will be helpful for illustrating our methodology.

%, beginning with a simple example.
% \hl{modified this (commented out original version) because I don't think we have other results using this method? if we do, we should ref them here}
% Let us apply this method to give several examples where canonicalization cannot occur. We begin with a simple example.
% \nd{perhaps leave this proof in main text?}
\begin{prop}
The action of $\ZZ$ on $\RR$ by addition does not have a continuous canonicalization.
\end{prop}
\begin{proof}
For this example, we use the fundamental group as an obstruction. The fundamental group is an invariant of topological spaces which essentially consists of loops in the space (see \citet{hatcher2002}, Chapter 1). %The fundamental group is functorial in the sense that 
Maps between spaces induce corresponding maps between their fundamental groups, so %. Consequently, 
the fundamental group can be used as an obstruction to show that certain maps cannot exist. 

The quotient $\RR/\ZZ $ is isomorphic to $S^1$, whose fundamental group is $\ZZ$, while the fundamental group of $\RR$ is trivial. A continuous canonicalization would give a map $\ty:S^1\rightarrow \mathbb{R}$ such that
$
    S^1\xrightarrow{\ty} \RR\xrightarrow{q} S^1
$
is the identity on $S^1$. Thus, the induced maps on the fundamental group would satisfy
$
    \ZZ\xrightarrow{\ty^*} \{0\}\xrightarrow{q^*} \ZZ
$
is the identity on $\ZZ$, which is impossible.
\end{proof}
% \begin{restatable}{prop}{Ztwo}\label{pm1-action-example}
% The module $(\RR^2,G=\{-1,1\}) $ with the action of multiplication does not have a continuous canonicalization.
% \end{restatable}

% \begin{remark}
% The above example of lack of continuity for $G=\{-1,1\}$ is also a special case of \cref{thm:onlyReynold} we will prove in the next section which will show that under mild conditions finite (unweighted) frames which are not Reynolds operators cannot preserve continuity.  
% \end{remark}

%Next, we consider the more interesting examples of $O(d)$ and $SO(d)$ acting on ordered point clouds.

% \begin{restatable}{prop}{SOtwo}\label{prop:SOno}
%    If $n \geq d \geq 2$, $SO(d)$ acting on $\RR^{d\times n}$ %the module $(\RR^{d\times n},SO(d)) $  
%    does not have a continuous canonicalization.
% \end{restatable}

% \begin{restatable}{prop}{Otwo}\label{prop:Ono}
%    If $n > d \geq 1$, then $O(d)$ acting on $\RR^{d\times n}$ %$(\RR^{d\times n},O(d))$ 
%    does not have a continuous canonicalization.
% \end{restatable}

\begin{restatable}{prop}{OtwoandSOtwo}\label{prop:SOno_and_Ono}
   Consider $O(d)$ and $SO(d)$ acting on $\RR^{d\times n}$. If $n > d \geq 1$ for $O(d)$, or $n \geq d \geq 2$ for $SO(d)$, %then $O(d)$ and $SO(d)$ acting on $\RR^{d\times n}$ %$(\RR^{d\times n},O(d))$ 
   %do not have continuous canonicalizations.
   then there is no continuous canonicalization.
\end{restatable}

\begin{proof}[Proof idea]
    % The detailed proofs are given in the Appendix \ref{proofs-appendix}. Let us briefly describe the idea for $SO(2)$. It is possible to reduce to the case $d=n=2$. We consider the subspace $Y = S^3\subset \mathbb{R}^{2\times 2}$, and show that $Y/G\eqsim S^2$ (the quotient map $q:Y\rightarrow Y/G$ is in fact the famous Hopf fibration \cite{hopf1931abbildungen}). Thus, a continuous canonicalization would give a map $y:S^2\rightarrow S^3$ such that
    % $$
    % S^2\xrightarrow{y} S^3\xrightarrow{q} S^2
    % $$
    % is the identity on $S^2$. Calculating the homology groups (see \cite{hatcher2002}, Chapter 2) of these spaces shows that this is not possible.
%In the $SO(d)$ proof we provide an elementary reduction 
For $SO(d)$, we reduce to %the case where 
$n=d=2$ and use an %. In this case we use an 
algebraic topology argument similar to the one before, %described above, but %this time 
but with homology groups. The proof for $O(d)$ is similar.
\end{proof}

We remark that this result is sharp, i.e. a continuous canonicalization exists for $O(d)$ when $n\leq d$ and for $SO(d)$ when $n < d$. We construct such canonicalizations in Appendix \ref{app:low}.

%We remark that the existence of a continuous canonicalization can be somewhat subtle. For instance, we have the following result showing that the group in the above example is changed to $O(2)$ instead of $SO(2)$, then a continuous canonicalization exists when $n=1$.

% Currently we do not have a general result for the modules $(\RR^{d\times n},SO(d)) $ and $(\RR^{d\times n},O(d)) $. We conjecture that the lack of continuous canonicalization, obtained above for $d=2$ and $n>1$, will be true for general $d$ and large enough $n$. \nd{actually I think we can prove this. Need to discuss. This is the claim}

\subsection{Frames}
% The term frame, or a \emph{moving frame} was first defined in 1937 by mathematician Élie Cartan \nd{@Hannah ref?} to denote what we called in this paper a group canonicalization. We will use the term frame as introduced more recently by \citet{puny2021frame}. 

%The canonicalization approach averages over a single group transformation per point $x$. 
\citet{puny2021frame} generalize canonicalization %this framework 
to allow for \emph{averaging} over an equivariant set of points %on each orbit
instead. Concretely, they define a frame $\mathcal{F}$  as %a set-valued function, 
$\mathcal{F}:V \rightarrow 2^G \backslash \emptyset$, which is equivariant: $\mathcal{F}(gv) = g\mathcal{F}(v)$,
% \begin{align*}
%     \mathcal{F}(gm) = g\mathcal{F}(m),
% \end{align*}
where the equality is between sets. %If $(V,G)$ are a module,  
The invariant projection operator $\Iframe $ induced by the frame is defined for every  $f:V \to \RR$ as 
\begin{align*}
 \Iframe[f](v) := \frac{1}{|\mathcal{F}(v)|}\sum_{g \in \mathcal{F}(v)} f(g^{-1}v).
\end{align*}
%If $(G,V)$ and $(G,W)$ are two modules,  we can 
The equivariant projection operator $\Eframe$ is similar, but with summand $gf(g^{-1}v)$ instead. 

% Similarly, we obtain an equivariant projection operator defined on functions $f:V\to W$ via
% \begin{align*}
%    \Eframe[f](v) := \frac{1}{|\mathcal{F}(v)|}\sum_{g \in \mathcal{F}(v)} gf(g^{-1}v).
% \end{align*}
We note that, when the frame maps to $G \in 2^G$ for all elements of the input space $\mathcal{M}$, then frame-averaging reduces to group-averaging. %the  Reynolds operator for group-averaging functions. %(which projects a given function to the closest equivariant function in an $L_2$ sense, see \emph{e.g.}, \cite{elesedy2021provably}). 
On the other extreme, a frame where  $|\mathcal{F}(v)|=1 \: \forall v$ %always maps to a set containing exactly one group element 
is exactly a group canonicalization.

\citet{puny2021frame} provide several %interesting 
examples where frames seem a natural alternative to canonicalization. 
Nonetheless, in many cases (such as PCA), these frames %clearly %have singularities and 
do not preserve continuity. %\footnote{\hl{We only proved that preserving continuity is the same as having a continuous frame for the canonicalization case -- for frames, there might be a singularity but still preserve continuity, right?} }
It turns out that this can be unavoidable: %significant issue.
\begin{restatable}{thm}{onlyReynold}\label{thm:onlyReynold}
Let $G $ be a finite group acting  continuously on a  metric space $V$, let $V_{free}$ denote the points in $V$ with trivial stabilizer, and let  $\F$ be a frame which preserves continuity on $V_{free}$. If $V_{free}$ is connected, then  $\F(v)=G$ for all $v$ in the closure of   $V_{free}$.
\end{restatable}
\begin{proof}[Proof idea]
%The proof has three steps: 
First, we show that for finite groups acting freely on $V_{free}$, a frame which preserves continuity must be \emph{locally} constant. Next, due to connectivity, this can be extended to a \emph{global} result. Finally, we combine the fact that the frame is globally constant with equivariance ($\F(gv)=g\F(v)$), and extend this to the closure of $V_{free}$. % completes the proof. %to show that all group elements are in $\F(v)$ for all $v\in V$.
\end{proof}
% We now consider the action of $S_n$ %the finite permutation group 
% on point clouds. We assume that we have $n>1$ points in $\RR^d$ with $d>1$, and consider the set $\Rdistinct$ of point clouds with trivial stabilizer (equivalently, point clouds $X $ with  $x_i\neq x_j \: \forall i\neq j$).

%  As a corollary, %we show that 
% a frame which preserves continuity for the action of $S_n$ restricted to the set $\Rdistinct$ of point clouds with trivial stabilizer (equivalently, $X \in \mathbb{R}^{d \times n}$ with  $x_i\neq x_j \: \forall i\neq j$) is necessarily the Reynolds operator.
%\hl{I don't think we need both this paragraph and the formal corollary, as they say exactly the same thing}
\begin{restatable}{cor}{snReynold}\label{cor:snReynold}
Let $d, n > 1$, and consider $S_n$ acting %the finite permutation group 
on  $\RR^{d\times n}$. If $\F$ is a continuity-preserving frame, %for this group action
%which preserves continuity, 
then $\F(X)=S_n$ for all $ X\in \RR^{d\times n}$.
% We assume that we have $n>1$ points in $\RR^d$ with $d>1$, and consider the set $\Rdistinct$ of point clouds with trivial stabilizer (equivalently, point clouds $X $ with  $x_i\neq x_j \: \forall i\neq j$).
% Let $d>1,n>1\in \mathbb{N}$ and 
% Let $\F$ be a frame for this group action. % with respect to the module  $(\Rdistinct,S_n) $. 
% If $\F$ preserves continuity, then $\F(X)=S_n$ for all $X\in \Rdistinct$.
\end{restatable}
The corollary implies that there is no continuity preserving frame of reasonable cardinality for permutations. Interestingly, Theorem \ref{thm:onlyReynold} and Corollary \ref{cor:snReynold} imply that the failure of a small $S_n$ frame to preserve continuity is not a result of self-symmetries alone: even continuity only at points with trivial stabilizer implies frame size $n!$.

We next show that, for the infinite group $SO(2)$, there is no  unweighted frame of any finite cardinality that preserves continuity under the action on pairs of points.
\begin{restatable}{thm}{finiteframeSOthm}
    Consider $SO(2)$ acting on $\RR^{2\times n}$ with $n \geq 2$. If $\F$ is a continuity-preserving frame, then 
    $$
        \sup_{X\in \RR^{d\times n} \backslash \mathbf{0}} |\F(X)| = \infty,
    $$ 
    i.e. there does not exist a finite (unweighted) frame which preseves continuity. 
\end{restatable}
% \begin{proof}[Proof idea]
% %The proof has three steps: 
% \textcolor{red}{To fill out still} \nd{I think we can skip proof idea}
% \end{proof}
Note that the result above holds at $X$ with trivial stabilizer, since the supremum excludes $X=0$  (at which any unweighted frame trivially has infinite size).

%\nd{Add: generalization to $O(d)\times S_n $?}

\section{Weighted frames}\label{sec:weighted_frames} 
%We have seen that 
In some cases, it is therefore impossible to define frames of reasonable size which preserve continuity. To address this issue, we suggest a generalization of frame-averaging to \emph{weighted frames}, and determine conditions under which averaging over a weighted frame preserves continuity. We furthermore obtain advantages in size by defining a weaker notion of equivariant frames for points $v$ with $\stab{v}\neq \{e\}$. %with non-trivial stabilizer.  %\hl{possibly delete this last sentence}
% generalize the notions from Puny's paper by 
% \begin{enumerate}
%     \item Allowing for weighted frames.
%     \item Using a weaker notion of equivariant frame for points with a non-trivial stabilizer.
%     \item Determining conditions under which averaging over a weighted frame preserves continuity.
% \end{enumerate}

It will first be helpful to recall some %notions from 
measure theory. If $\mu$ is a Borel probability measure on a topological group $G$, we can, for $g\in G$, define a ``pushforward'' measure $g_*\mu$ which assigns to %each Borel set 
$A\subseteq G $ the measure $g_*\mu(A)=\mu(g^{-1}(A)) $. 

For every $v\in V$ and Borel probability measure $\tau$ on $G$,  %\hl{using both $v$ and $\nu$ was a bit hard to read: switching to $\tau$ from $\nu$}, 
 we let  $\langle \tau \rangle_v$ be the measure defined for every %Borel set 
 $A\subseteq G$ by 
	$$ \langle \tau \rangle_v(A)=\int_{s \in \stab{v}}s_*\tau(A)ds,$$
	where the integral is over the Haar measure of %the compact group 
 $\stab{v}$. This can be thought of as %an operation which averages 
 averaging over the stabilizer of $v$. For a frame $\muF$, we use the shortened notation 
$\bar \mu_v :=\langle \mu_v \rangle_v$.
 %(One can check that $\bar \mu_v$ remains a probability measure, i.e. $\bar \mu_v(G)=1$ for all $v$.) 

\begin{defn}[Weighted frames]\label{def:wFrame}
	%Let $(V,G)$ be a module. 
 A weighted frame $\muF$  is a mapping $v\mapsto \mu_v $ from %the vector space 
 $V$ to the space of Borel probability measures on $G$. $\muF$ is \emph{equivariant} at $v\in V$ if $ \forall g\in G$, $\mu_{gv}=g_*\mu_v$.  
 $\muF$ is \emph{weakly equivariant} at $v\in V$ if $\forall  g\in G$, $$\bar{\mu}_{gv} = \langle{\mu}_{gv}\rangle_{gv}=g_*\langle \mu_v \rangle_v = g_* \bar{\mu}_v, \forall g\in G,$$
 We say that a frame 
 $\muF$ is (weakly) equivariant if it is (weakly) equivariant at all $v\in V$. See Figure \ref{fig:discontinuous_canons} for a visual.
\end{defn}
% \nd{see if you like notation $\muF$ above for a frame. Macro is backslask muF }

\begin{remark}%[Equivariance vs weak equivariance]
    If a frame is equivariant, then it is also weakly equivariant. Moreover, the two definitions are equivalent at points $v$ with trivial stabilizer.
\end{remark}

Though our definitions allow for general probability measures, we will only be interested in measures with finite support. Thus, the main difference from %the frames of 
% changed from "only difference" b/c we are consolidating the weakly equivariant definition here too
\citet{puny2021frame} is that we allow group elements to have non-uniform weights. 

The goal of using general weighted frames is to provide a more efficient BEC operator than the full Reynolds operator. To quantify this, we define the \textit{cardinality} of the frame $\mu$ as
\begin{align*}
%&\sup_{v\in V} |\{C\in G/\stab{v}:~\mu_v(C) > 0\}|. \\
    &\sup_{v\in V} |\{g \in G:~\mu_v(g) > 0\}|.
\end{align*}
%Here, $G_v$ is the stabilizer of the element $v\in V$. 
Note that the cardinality is the worst-case number of evaluations of $f$ required to evaluate $\Iw$ or $\Ew$. % 
%for the worst possible inputs $v\in V$.
% \nd{seems inconsistent with the rest of definitions.  I would define cardinality at $v$ as the size of the support of $\mu_v$, and the cardinality of $\mu$ as the maximum cardinality over all $v$.}

%Before we define equivariant frames 

% Equivalently, if $f$ is the indicator function of $A$, or any other Borel measurable function defined on $G$, then 
% $$\int_G f(g')dg_*\mu(g')=\int_G f(gg')d\mu(g') .$$

% \begin{defn}[Weighted equivariant frame]
% 	% Let $(V,G)$ be a module.  
%  A weighted frame $\mu$ is equivariant if 
% $$\mu_{gv}=g_*\mu_v, \quad \forall v\in V, g\in G  .$$
% %We say that $\mu$ is an equivariant weighted frame on $V$ if it is an equivariant weighted frame at every $v\in V$. %\hl{added in ``weighted'' here, but maybe we want to say explicitly that from this point on, we mean weighted when we say frame??} \nd{maybe we should define Puny's frames as "unweighted frames" and ours as "frames". with due apologies}
% \end{defn}

 Our primary motivation in defining \emph{weakly} equivariant weighted frames is computational. Indeed, note that for every $s\in \stab{v}$, equivariance of a frame implies that $\mu_v=\mu_{sv}=s_*\mu $. This implies that the support of $\mu_v$ must be at least the size of the stabilizer of $v$, which can be problematic when the stabilizer is large. % (as the complexity of frame-averaging scales with the support size of $\mu$). 
    Weak equivariance bypasses this issue, drawing on the intuition that if $s \in \stab{v}$, then $v=sv$ implies $f(sv)=f(v)$. Therefore, in terms of \emph{invariant} averaging, equivalence up to $\stab{v}$ should not affect the projection operator. We now make this precise. % in the following proposition. 
    %\hl{we should maybe spend more time explaining this / even have a figure} 

Both equivariant and weakly equivariant weighted frames can be used to define invariant projection operators. Namely,
\begin{align*}
	\Iw[f](v) := \int_G f(g^{-1}v)d\mu_v(g).
\end{align*}
$\Ew$ is the same, but with integrand $gf(g^{-1}v)$.
% The equivariant operator $\Ew$ is the same as $\Iw$, but with integrand $gf(g^{-1}v)$. % In the following proposition, we 

\begin{restatable}{prop}{invariantprop}
\normalfont{(Invariant frame averaging})\label{prop:weak_equiv_frame_avg_invar}
%Let $(G,V)$ and $(G,W) $ be modules and 
Let $\muF$ be a weighted frame. 
% , and define $$\Iw[f](v) := \int_G f(g^{-1}v)d\mu_v(g).$$ Define $\Ew$ in the same way, but with integrand $gf(g^{-1}v)$. %well-defined, 
% If $\muF$ is \emph{equivariant}, then $\Iw$ and $\Ew$ are both bounded projection operators. 
%bounded projection operators. %, respectively.
%However, we will need additional conditions on $\muF$ to preserve continuity. 
If $\muF$ is \emph{weakly equivariant}, then $\Iw$ %\Iw:F(V,\RR)\to F(V,\RR)$ 
is a %well-defined, 
bounded, invariant projection operator. %(However, $\Ew$ may not be an equivariant projection operator; we discuss ways to address this in Section \ref{sec:inv2eq}.) In all cases, 
\end{restatable}
To preserve continuity, we will need a natural notion of continuity of  $\muF$. We discuss this in the next section. 

\begin{remark}[Equivariant frame averaging]\label{rem:equisucks}
    If $\muF$ is \emph{equivariant}, then $\Ew$ is a bounded projection operator. However, if $\muF$ is only \emph{weakly} equivariant, then $\Ew$ may not be an equivariant projection operator; we discuss possible solutions to this in Section \ref{sec:inv2eq}.
\end{remark}

\subsection{Continuous weighted frames}
In this subsection, we define a natural notion of a \emph{continuous weakly equivariant weighted frame}. We will show that with this notion, the operator $\Iw$ preserves continuity. 

A natural definition of continuity would be to require that whenever $v_n\rightarrow v$, $\mu_{v_n}\rightarrow \mu_v$ in the weak topology %\hl{add to appendix a reminder of what this means!} 
on probability measures on $G$ (i.e. for any continuous function $f:G\to \RR$, % we will have that 
$\int fd\mu_{v_n}\rightarrow \int fd\mu_v $). However, we allow the following even weaker notion of continuity, which takes into account the ambiguity at points with non-trivial stabilizers.   
\begin{defn}[Continuous weighted frames]\label{defn:continuity-of-weighted-frame}
 %Let $(G,V)$ be a module. 
 Let $\mu$ be a weakly equivariant weighted frame. We say that $\mu$ is continuous at the point $v$ if for every sequence $v_k$ converging to $v$, we have $\langle \mu_{v_k}\rangle_v \rightarrow  \langle \mu_{v}\rangle_v=\bar{\mu}_{v}$ 
 % \begin{equation}\label{continuity-conditions-definition}
 %    \langle \mu_{v_k}\rangle_v \rightarrow  \langle \mu_{v}\rangle_v=\bar{\mu}_{v}
 % \end{equation}
 in the weak topology. $\mu$ is continuous if it is continuous at every $v \in V$. %If the frame is continuous at every point $v$, we say it is continuous.
%\hl{What is the difference from saying that $\bar{\mu}_{v_n} \rightarrow \bar{\mu}_v$? }
%\nd{think about what happens in the $(\CC^n,S^1)$ case when $v=0$, we have $\bar\mu_v=0$}
\end{defn}

We will use the term \textbf{\CWW frames} to refer to continuous, weakly equivariant, weighted frames.
\begin{restatable}{prop}{cwwbec}
%Let $(G,V)$ be a module and 
$\Iw[f] $ is a BEC operator iff $\mu$ is a \CWW frame. % (a bounded, invariant, continuity-preserving projection). % operator which preserves continuity).
\end{restatable}
%\hl{converse?}
%\hl{GOAL: everything before this in the first four pages}ֿ
\section{Weighted frames for permutations}\label{sec:weighted_frames_for_permutations}

In Corollary \ref{cor:snReynold}, we showed that the only (unweighted) frame which preserves continuity with respect to the action of the permutation group $S_n$ on $\Rdistinct$ is the full Reynolds frame. In \ref{subsec:robust_perm_frames_distinct}, %the first subsection 
we will show there is a \emph{weighted frame} which preserves continuity in this case, with cardinality of only $n(d-1)$. In \ref{subsec:robust_frames_perm_all}, %the second subsection 
we will discuss the harder %more challenging 
task of continuity preservation for the action of $S_n$ on all of $\RR^{d\times n}$ .

\subsection{Robust frames for permutations on $\Rdistinct$}\label{subsec:robust_perm_frames_distinct}
The frames we construct are %we will discuss in this subsection, and the next subsection, will be 
based on one-dimensional sorting. For $X\in \Rdistinct$ and $a\in \RR^d $, we say that $X$ is $a$-separated if there exists a unique permutation $\tau$ such that 
\begin{equation}\label{eq:tau}
a^Tx_{\tau(1)}< a^Tx_{\tau(2)}<\ldots < a^Tx_{\tau(n)}.\end{equation}
 Our goal is to find a finite, relatively small number of vectors $a_1,\ldots,a_m$, such that every  $X\in \Rdistinct$ is $a_i$ separated for at least one $i$. When this occurs, we call $a_1,\ldots,a_m$ a \emph{globally separated collection}.  The next theorem shows that this is possible %goal can be attained
 if and only if  $m\geq n(d-1)$.
\begin{restatable}{thm}{mon}\label{thm:separated}
Let $n,d>1$ be natural numbers.  
Then Lebesgue almost every $a_1,\ldots,a_{n(d-1)}\in \RR^d$
form a globally separated collection. Conversely, every globally separated collection must contain at least $n(d-1)$ vectors.
\end{restatable}
% A reviewer has pointed out that 
A similar result, with slightly higher cardinality, was obtained in \cite{ye2024widetildeon2}, where this result is used to  represent functions equivariant to the anti-symmetric group, for quantum chemistry simulations. 

We now explain how we can construct a \CWW frame $\mumon$ on $\Rdistinct$ using a globally separated collection %.
%Let 
$a_1,\ldots,a_m$ from Theorem \ref{thm:separated}. Our frame will be of the form $\mumon_X=\sum_{i=1}^m w_i(X)\delta_{g^{-1}_i(X)} $, where $g_i(X)$ is %defined to be 
a permutation $\tau$ satisfying \eqref{eq:tau}. Note that if $X$ is not separated in the direction $a_i$, then $\tau$ is not uniquely defined. In this case we choose $\tau$ arbitrarily, which is not a problem because we will define the weights $w_i(X)$ to be zero in this case: %. That is,
\begin{equation*}\label{CWW-separated-frame}
\begin{split}
\tilde{w}_i(X)&=\min_{s\neq t} |a_i^T(x_s-x_t)|, \:
w_i(X)=\frac{\tilde{w}_i(X)}{\sum_{j=1}^{m} \tilde{w}_j(X)}
\end{split}
\end{equation*}
Note that the division by $\sum_{j=1}^{m} \tilde{w}_j(X)$ %in the equation above 
is well-defined because $a_1,\ldots,a_m$ is a globally separated collection.

\begin{restatable}{lem}{mumonlem}
%The frame 
$\mumon$ %defined in \eqref{CWW-separated-frame} 
is a \CWW frame for $\Rdistinct$. %,S_n)$.
\end{restatable}

\subsection{\cCWW frames for all of $\RR^{d\times n} $}\label{subsec:robust_frames_perm_all}
The frame $\mumon$ is \CWW on $\Rdistinct$, but cannot be extended to a \CWW frame on all of $\RR^{d\times n}$. We now  provide a \CWW frame $\mubig$ for all of $(\RR^{d\times n},S_n)$ which averages over permutations obtained from all possible directions, rather than considering a fixed collection of directions. For a given $X$, $\mubig_X$ is defined by assigning to each $g^{-1} \in S_n$ the weight %. T
$$w_{g^{-1}}(X)=\mathbb{P}_{a\sim S^{d-1}}\left[g=\mathrm{argsort}(a^TX)\right].$$
Here the probability $\mathbb{P}$ is over directions $a$ distributed uniformly on $S^{d-1}$, and $\mathrm{argsort}(w)$ is the unique permutation $g\in S_n$ which sorts $w$ while preserving the ordering of equal entries. %, i.e. $i < j~\text{and}~ g(j) < g(i) \iff w_i < w_j$.
% Another way of describing $\mu^{S_n}$ is through its projection operator $\Iw$, which is given by

% \begin{equation}\label{projection-operator-mu-S_n}
%     \Iw[f](X) = \mathbb{E}_{a\sim S^{d-1}}(f(\text{sort}(a,X))),
% \end{equation}
% where $\text{sort}(a,X) = gX$ for a permutation $g\in S_n$ satisfying
% $$
%     a^Tx_{g(1)}\leq a^Tx_{g(2)}\leq\ldots \leq a^Tx_{g(n)}.
% $$
% In other words, we choose a random direction to sort along and then average the resulting value of $f$. Observe that $\text{sort}(a,X)$ is unambiguous for almost every $a\in S^{d-1}$ and so the above expoectation is well-defined.
%The following %Proposition 
%shows

We then have that $\mu^{S_n}$ is a \CWW frame of moderate size. %, and bounds its cardinality.
\begin{restatable}{prop}{mubigprop}
$\mubig$ is a \CWW frame for $S_n$ acting on $\RR^{d\times n}$, %the module $(\RR^{d\times n},S_n)$ 
with cardinality bounded by  
$
    2\cdot \sum_{k=0}^{d-1}\binom{\frac{n^2 - n - 2}{2}}{k}.
$
\end{restatable}
In most applications, $d << n$ (e.g. $d=3$), in which case the bound above is $O(n^{2(d-1)})$. This is significantly worse than $\mumon$, %our frame for $\Rdistinct$ 
which had cardinality of $n(d-1)$, but also a  significant improvement over the $(n!)$-sized Reynolds operator. %, which has cardinality $n!$. %In most cases 
%We anticipate that 
The frame $\mu^{S_n}$ may be too large to compute exactly, but can be implemented in an augmentation-like style by randomly drawing $a\in S^{d-1}$ to sort along. 

Our final result for $S_n$ acting on $\RR^{d\times n}$ is a lower bound (Proposition \ref{lb}) on the cardinality of any \CWW frame. When $d\ll n$, the lower bound is $\sim (d-1)(n/2)$.
% \hl{ref the proof in appendix? or, we could just keep the sentences above and say see appendix for complete statement}
% \begin{restatable}{prop}{lb}
%  Any \CWW frame for the module $(\RR^{d\times n},S_n)$ with $d>1$ has cardinality at least $$k(d,n) := (d-1)\lfloor n/2\rfloor + 1 - \sum_{i=1}^{\lfloor n/2\rfloor}(d-1-2^i)_+.$$
% \end{restatable}

\section{Weighted frames for rotations % on point clouds
}\label{sec:weighted_frames_for_rotations}
We now show how to define robust frames for the action of $SO(d)$ on $\RR^{d\times n}$ with cardinality $n(n-1)\cdots (n-d+2)$. We note that our $d=3$ %robust frame 
construction is very similar to the ``smooth frames'' introduced in  \citet{pozdnyakov2023smooth}; in a sense, our contribution is %just 
to generalize their frames to all dimensions, and to formally define and prove their robustness. Using essentially the same idea, we also construct robust frames for the action of $O(d)$ on $\RR^{d\times n}$ with cardinality $2n(n-1)\cdots (n-d+2)$ in Appendix \ref{app:so3}.

\subsection{The case of $d=2$}\label{sub:so2}
We use the identification $\RR^{2\times n} \equiv \CC^n $ and work in complex notation for simplicity. We denote a vector in $\CC^n$ as $Z=(z_1,\ldots,z_n)$, and define a weighted frame of the form 
$$\muTwo_Z=\sum_{i=1}^n w_i(Z)\delta_{g_i(Z)}.$$
The group elements $g_i(Z) = z_i/\|z_i\|\in S^1$ are defined to be the phase of the $i$-th entry, so that multiplying by $g_i^{-1}$ will rotate $Z$ so that $z_i$ is real and positive. This is not well-defined when $z_i=0$, in which case we somewhat arbitrarily set %use a somewhat arbitrary definition 
$g_i(Z)=1$.
% $$g_i(Z)=\twopartdef{\frac{z_i}{\|z_i\|}}{z_i\neq 0}{1}{z_i=0} $$
%We can use 

The weight functions are defined as follows.  We fix $\eta \in (0,1)$ and a continuous function $\phi_\eta$ which is zero on $(-\infty,\eta)$, one on $[1,\infty)$, and
satisfies $0\leq \phi_\eta(t)\leq 1 $ elsewhere. For $Z\neq 0$, we set
$\tilde w_i(Z)=\phi_\eta\left(\frac{\|z_i\|}{\max_j\|z_j\|} \right).$
We then define $w_i(Z)=\frac{1}{n}$ for $Z=0_n$, and % w_i(Z)=
$\frac{\tilde{w}_i(Z)}{\sum_j \tilde{w}_j(Z)}$ otherwise. % via
%$$w_i(Z)=\twopartdef{\frac{\tilde{w}_i(Z)}{\sum_j \tilde{w}_j(Z)}}{Z\neq 0}{\frac{1}{n}}{Z=0}. $$

The functions $w_i$ are $S^1$ invariant, non-negative, and sum to one everywhere. They are  continuous on $\CC^n \setminus \{0_n \} $, and they ensure that at a point $Z$ with some zero and some non-zero coordinates, only the non-zero coordinates will be ``active''. The  frame $\muTwo$ does have a singularity at $0_n$, where all coordinates are zero. However, this singularity is ``harmless'' because the stabilizer %of $0_n$ 
is the whole group $S_1$. %In other words
\begin{restatable}{prop}{mutwo}
%The frame 
$\muTwo$ 
%for $(\CC^n,S^1) $ 
%for $S^1$ acting on $\CC^n$ 
is a \CWW frame.
\end{restatable}
\begin{proof}[Proof idea]
We explain why $\muTwo$ is weakly equivariant, leaving the full proof for the appendix. At points $Z\neq 0_n$ this follows from the invariance of the weight functions $w_i$ and the equivariance of $g_i(Z)$ at all points, except for points with $Z_i=0, X\neq 0_n$ for which $w_i(Z)=0$. For %the point 
$Z=0_n$, weak equivariance follows from the fact that $\stab{0_n}$ is all of % the whole group 
$S^1$. Thus, for any distribution %probability measure 
$\mu$ on $S^1$, the average measure $\langle\mu \rangle_{0_n}$ is the same: the Haar measure on $S^1$.  
\end{proof}

As a corollary of our general results on \CWW frames, we deduce that projecting a dense set of continuous functions $Q$ using the invariant operator $\Iw$ induced from $\muTwo$, will give a dense set of continuous invariant functions $\Iw(Q) $. Explicitly, for a given $q\in Q$, the function $\Iw[q]$ will be of the form 
$$\Iw[q](Z)=\sum_{i=1}^n w_i(Z)q \left(\frac{\bar{z}_i}{\|z_i\|}\cdot Z \right) $$
\subsection{The case of $d\geq3$} \label{sub:so3}
 Generalizing the cases of $d=2$ and $d=3$, which essentially appear in \cite{pozdnyakov2023smooth}, to higher dimensions requires a much more involved construction and proof, which is given in detail in Appendix \ref{app:so3}. The basic idea is to associate a weight to every sequence of $r$ columns of $X\in \RR^{d\times n}$, where $r = \min(d-1,\text{rank}(X))$. For each such sequence of columns, a rotation can be obtained by mapping the column vectors into a standard position, determined by a Gram-Schmidt orthogonalization with the given ordering. By carefully choosing the weights to vanish whenever this rotation is not uniquely defined (up to an element of $G_X$), we can ensure that the resulting frame is robust.
 
\section{From Invariant to Equivariant Projections}\label{sec:inv2eq}

We saw that in the invariant setting, \CWW frames induced a BEC projection operator $\Iw$. In the equivariant setting, the situation is more complex. First, note that weighted frames $\muF$ which are  \emph{fully} equivariant do induce well-defined equivariant projection operators $\Ew$. However, this comes at a computational cost for inputs $v$ where $|\stab{v}|$ is large. 
%However, we feel that it is not likely to construct BEC operators with fully  equivariant frame with reasonable cardinality, and BEC operators require using our weak equivariance notion which enables exploiting non-trivial stabilizers.
When considering \CWW (and therefore only \emph{weakly} equivariant) frames, % in the equivariant setting, 
the natural ``equivariant'' operator $\Ew$ may not produce equivariant functions
(Example \ref{ex:so2_equiv}). However, we can remedy this %We will show that this can be remedied 
by requiring that the backbone architecture parametrize only \emph{stable functions}, which remain equivariant under $\Ew$. In Appendix \ref{app:stableFrame}, we also define \emph{stable frames} as an alternative approach, which can be applied to an arbitrary backbone architecture. In Appendix \ref{app:so3} and Appendix \ref{app:stableFrame} we show how both % and the other is to define a more restrictive notion of \emph{stable frames}. 
%We show how 
  of these ideas can be implemented efficiently to achieve \emph{equivariant}, continuous universal models for $(\RR^{d\times n},SO(d))$ for $d=2,3$. Below we will discuss only stable functions for $d=2$ . 

%, focusing on the former below and the latter in Appendix \hl{add ref}.

%\subsection{Stable functions}

We call a function $f:V\to W$ \emph{stable} if 
$\stab{v}\subseteq \stab{f(v)}$ $\forall v \in V$.
%$f$ is a stable function if it is stable at every $v\in V$. 
Note that any equivariant function $f:V\to W$ is stable, since for every $s\in \stab{v}$ we have $f(v)=f(sv)=sf(v)$. 
\begin{restatable}{prop}{stable}\label{prop:stable}
%Let $(V,G)$ and $(W,G)$ be modules, and l
Let $\muF$ be a \CWW frame. The restriction of $\Ew$ %:F(V,W)\to F(V,W)$ 
to \emph{stable} input functions is a continuity-preserving, bounded, equivariant projection. 
\end{restatable}

\begin{example}[$SO(2)$ equivariance]\label{ex:so2_equiv}
Let us return to the \CWW frame $\muTwo$ from Subsection \ref{sub:so2} for the action of $S^1\cong SO(2) $ on $\CC^n \cong \RR^{2\times n} $. %, where this time we are interested in equivariant functions $f:\CC^n \to \CC^n$ rather than invariant functions. 
Note that $G_Z=\{e\} \forall Z\in \CC^n$ %all points $Z\in \CC^n$ have a trivial stabilizer, 
except $Z=0_n$ (whose stabilizer is all of $S^1$). Thus %a function 
$f:\CC^n \to \CC^n$ is stable if and only if $f(0_n)=0_n$. 

We first note that applying $\Ew$ induced from $\muTwo$ to a function which is not stable %(that is, a function with $f(0_n)\neq 0_n$), 
gives a function $\Ew[f]$ which is also not stable, and therefore not equivariant: %. This is because %, from the definition of $\muTwo$,
$$\Ew[f](0)=\frac{1}{n}\sum_{j=1}^n f(0)=f(0)\neq 0. $$

By Proposition \ref{prop:stable}, this problem can be avoided if we apply $\Ew$ only to stable functions, which here %in this scenario just 
means functions $q$ satisfying $q(0_n)=0_n$. This condition is easily enforced: if $Q$ is a dense space of continuous functions $q:\CC^n \to \CC^n$, then we obtain a dense set of stable functions via %by considering 
$\hat q(Z)=q(Z)-q(0_n)$. Applying $\Ew$ to these $\hat q$ functions yields a continuous, universal, equivariant model.
\end{example}

% \subsection{Stable frames}
% Stable frames are an alternative strategy, where we add additional restrictions on \CWW frames such that $\Ew$ is an equivariant projection even when applied to non-stable functions. %which allow to define an equivariant projection operator which can applied to general functions (not only stable functions). 

% We give a formal definition of stable frames in Appendix \ref{app:stableFrame}. Here we just give an example for the 2-dimensional case. We construct from $\muTwo$ a new frame with doubled cardinality by setting 
% $$g_{i,+}(Z)=g_i(Z) \text{ and } g_{i,-}(Z)=-g_i(Z) $$
% $$w_{i,+}(Z)=w_{i,-}(Z)=\frac{1}{2}w_i(Z), $$
% we obtain an equivariant projection operator
% $$Q_2(f)(Z)=\sum_{i=1}^n \sum_{s\in \{-1,1\}}  w_{i,s}(Z) g_{i,s}(Z)\cdot  f(g_{i,s}^{-1}(Z)\cdot Z) $$
% with only twice the support size of $\muTwo$. Intuitively, the addition of this frame is that for any $f$ we will have that $Q_2[f](0_n)=0_n $, so this frame automatically produces stable functions. In Appendix \ref{app:stableFrame} we will show that this frame is a \emph{stable } \CWW frame and hence defines a BEC equivariant projection operator.
% This only doubles the size of the weighted frame $\muTwo$, 

%\section{Combining permutation, rotation, and translation}

\section{Experiments}
In this section, we provide experimental evidence showing the advantage of preserving both continuity and invariance using robust frames. We consider the action of the permutation group $S_n$ on two dimensional point clouds and leave the investigation of other group actions to future work. In Appendix \ref{app:subsec_verifying_discontinuities}, we also experimentally verify the presence of discontinuities in a trained canonicalization pipeline for point clouds from the equiadapt library \citep{kaba2022learned, mondal2023adaptation}.

\subsection{Comparison of $S_n$-frames}
In this experiment\footnote{Code for reproducing this experiment can be found at \href{https://github.com/jwsiegel2510/Sn-invariant-weighted-frames}{https://github.com/jwsiegel2510/Sn-invariant-weighted-frames}}, we tested permutation invariant frames on the following classification toy problem. 
Starting from the MNIST dataset, we processed the image of each digit into a two-dimensional point cloud containing 100 points, ordered randomly. We then trained a standard multi-layer perceptron (MLP) to classify the corresponding digit from this collection of point clouds, with invariance enforced in one of five ways: no invariance, invariance using a discontinuous canonicalization (sorting along the x-axis), invariance using each of the two robust frames introduced in Section \ref{sec:weighted_frames_for_permutations}, and invariance using the Reynolds operator (i.e. averaging over the entire group). Due to their size, each of the weighted frames (including the Reynolds operator) was implemented using empirical averaging, with one randomly drawn sample in each train step and $1$, $5$, $10$, or $25$ samples for inference (i.e. testing). All models were trained for 60 epochs using SGD with momentum 0.9 and a step size of 0.01, dropping to 0.001 after 30 epochs. The network was an MLP with 3 hidden layers of sizes 150, 100, and 50, with an input size of 200 and output of size 10. The results of this experiment are shown in Table \ref{tab:experiment_Sn}.
\begin{table}[]
    \centering
    \begin{tabular}{c|c}
         Invariance Method &	Test Accuracy (\%) \\
         \hline
No Invariance	&25.5  \\
Discontinuous Canonicalization	& 85.6  \\
Robust Frames (Sec. 4.1) &	75.5 / 85.6 / 87.1 / 88.4 \\
Robust Frames (Sec. 4.2)&	74.2 / 85.9 / 87.6 / 88.7 \\
Reynolds Operator	&21.0 / 22.4 / 22.6 / 22.6
    \end{tabular}
    \caption{Comparison between permutation canonicalization and various frames. The right hand column shows 1/5/10/25 samples drawn during testing for the weighted frames.}
    \label{tab:experiment_Sn}
\end{table}

From these results, we draw the following conclusions. First, no invariance and the (sampled) Reynolds operator do not work well, since the permutation group is so large that both of these are essentially not enforcing any permutation invariance and the dataset is far too small to enable learning without the permutation symmetry enforced. Second, a discontinuous canonicalization performs much better than the prior two methods without canonicalization (since permutation invariance is now enforced), but lack of continuity still hurts the test accuracy relative to the robust frames, which enforce both continuity and invariance. We also see that enforcing continuity on the entire input space $\RR^{d\times n}$ performs slightly better than only enforcing continuity at $\Rdistinct$. However, when empirically implementing robust frames, we do need to average quite a few samples from the frame during inference to obtain a good result. 

\section{Conclusion and open questions}
In this work, we %called attention to/
illuminated a critical problem with group canonicalization: %namely, that 
it can destroy the continuity of the function being canonicalized. %, resulting in a learned equivariant function that is highly non-robust. 
Moreover, even frames may have this problem if they aren't sufficiently large. As a solution, we introduced \emph{robust} frames, which are not only weighted but also continuity-preserving. Robust frames also deal intelligently with self-symmetric inputs, a facet that has not to our knowledge been previously analyzed. Finally, we construct several examples of robust frames. As frames (whether learned or deterministic) become more prevalent in the world of equivariant learning, % as methods for enforcing equivariance, 
we hope that our results will provide a guiding light %set of principles  
%\hl{set of principles? compass? north star?} 
for practitioners. % deciding which frames to use and what they can hope to achieve. 

%Although we have laid the theoretical groundwork for a careful analysis of frames, many interesting directions remain. 
Our work leaves open a few questions, such as stronger lower bounds on the cardinality of robust $SO(d)$ frames, and whether a continuous canonicalization exists for unordered point clouds. More broadly, %we analyzed only the continuity of frame-averaged functions; 
one may ask under what conditions (and frame sizes) one can expect stronger notions of smoothness, such as bounded Lipschitz constants. %Similarly, %although frame-averaging is a projection operator, 
%frame-averaging is not an \emph{orthogonal} projection like the Reynolds operator; one may ask how far from orthogonal this operator is. %, perhaps in terms of the size of the frame. 
In the equivariant case, it also remains to develop stable frames and/or %parametrize stable 
functions for a wider variety of groups. Finally, \citet{kim2023probabilistic}
 and \citet{mondal2023adaptation} \emph{probabilistically} sample from their weighted frames; to analyze this setting, %. This randomization is currently absent from our analysis, but 
 one might imagine concentration bounds replacing cardinality as the relevant measure of a frame's efficiency. %becoming relevant (rather than frame cardinality).  

% Future directions:
% \begin{enumerate}
%     \item Lower bound on frame size for $SO(d)$
%     \item Finer grained analysis of smoothness than continuity, e.g. Lipschitz constant, probably dependent on support size of frame. maybe having to do with norm of projection operator (since not orthogonal)?
%     \item Stable frames in general
%     \item Unordered point clouds ($S_n \times SO(d)$)
%     \item Analysis of randomness of probabilistic frames: might involve finer grained analysis than cardinality, something like concentration
%     \item Experiments (although not sure if we should draw attention to this) 
% \end{enumerate}

 %There are many potential societal consequences of our work, none which we feel must be specifically highlighted here. \nd{sounds a bit cheeky no? How about: instead of "there are many" we do not forsee any direct negative societal impact resulting from this work.}

\paragraph{Acknowledgements} ND is supported by Israel Science
Foundation grant no. 272/23. HL is supported by the Fannie and John Hertz Foundation and the NSF Graduate Fellowship under Grant No. 1745302. JWS is supported by the National Science Foundation (DMS-2424305 and CCF-2205004) as well as the Office of Naval Research (MURI ONR grant N00014-20-1-2787).

\section*{Impact Statement}
This paper presents work whose goal is to advance the field of Machine Learning. We do not foresee any direct negative societal impact resulting from this work
\newpage
\bibliography{refs}
\bibliographystyle{icml2024}
.

%%%%%%%%%%%%%%%%%%%%%%%%%%%%%%%%%%%%%%%%%%%%%%%%%%%%%%%%%%%%%%%%%%%%%%%%%%%%%%%
%%%%%%%%%%%%%%%%%%%%%%%%%%%%%%%%%%%%%%%%%%%%%%%%%%%%%%%%%%%%%%%%%%%%%%%%%%%%%%%
% APPENDIX
%%%%%%%%%%%%%%%%%%%%%%%%%%%%%%%%%%%%%%%%%%%%%%%%%%%%%%%%%%%%%%%%%%%%%%%%%%%%%%%
%%%%%%%%%%%%%%%%%%%%%%%%%%%%%%%%%%%%%%%%%%%%%%%%%%%%%%%%%%%%%%%%%%%%%%%%%%%%%%%
\newpage
\appendix
\onecolumn
\section*{Appendix structure}
%\hl{TO ADD: outline of the appendix or table of contents (with links)!}
The structure of the appendix is as follows. Appendix \ref{sec:additional_background} details some additional mathematical assumptions and background. 

We then discuss results which were not fully stated in the paper: Appendix \ref{app:low} proves the existence of continuous canonicalizations for $SO(d)$ when $n<d$, and for $O(d)$ when $n \leq d$.
Appendix \ref{app:lb} proves a lower bound on the cardinality of \CWW frames for the module $(\RR^{d\times n},S_n)$ with $d>1$.
Appendix \ref{app:so3} constructs robust frames for the actions of $SO(d)$ and $O(d)$, and Appendix \ref{app:stableFrame}
 discusses stable frames.

 The last and largest appendix \ref{proofs-appendix} contains the proofs of all claims in the paper, listed in chronological order.
\section{Additional Background}\label{sec:additional_background}
As stated in the main text,  throughout the paper we considered compact groups $G$ acting linearly and continuously on (typically) finite dimensional real vector spaces $V$ and $W$, or else on subsets of $V$ and $W$  closed under the action of $G$. These are common assumptions, essentially the same as the $G$-modules used in the definition of a module in \cite{yarotsky2022universal}. In this appendix we lay out in more detail what a module $(V,G)$ means: 
\begin{enumerate}
\item We assume $V$ is a finite dimensional real Hilbert space, i.e. a finite dimensional vector space endowed with a positive definite inner product. 
\item  $G$ is a compact group. That is, $G$ is a group endowed with a topology under which $G$ is a compact Hausdorff space, and moreover the multiplication and inverse operations are continuous.
\item $G$ acts on $V$ and for every fixed $g$, the map $g:V \to V $ is a linear transformation, and  the map $(g,v)\mapsto gv $ is continuous. Note that this is equivalent to a continuous group homomorphism $G\rightarrow GL(V)$, where $GL(V)$ denotes the general linear group of $V$, i.e. the group of all invertible linear transformations of $V$.
\end{enumerate}

\begin{defn}
    If $G$ acts on a set $V$ and $V'\subseteq V$, then we say that $V'$ is a \textbf{$G$-stable set}, or that $V'$ is closed under the action of $G$, if 
    $$gV' := \{gv,~v\in V'\} \subset V'.$$
\end{defn}
 Note that a group action on a set $V$ induces a well-defined group action any $G$-stable subset $V'\subset V$. In our discussion in the paper, in the $(V, G)$ pairs we discuss, we allow $V$ to either be the whole vector space or a $G$ stable subset. 
\newpage
\section{Continuous orthogonal canonicalizations}\label{app:low}
In this section, we show that Proposition \ref{prop:SOno_and_Ono} is sharp, i.e. that there exists a continuous canonicalization for $SO(d)$ when $n  < d$ and for $O(d)$ when $n \leq d$. The key to this is the following construction.
\begin{prop}
    The module $(\RR^{d\times d},O(d))$  has a continuous canonicalization.
\end{prop}
\begin{proof}
    We first recall the classical fact that the positive semi-definite Gram matrix $X^TX$ is a complete invariant for the action of $O(d)$ on $X\in \mathbb{R}^{d\times d}$. We define a continuous canonicalization $y:\RR^{d\times d}\rightarrow \RR^{d\times d}$ by
    \begin{equation}
        y_X = (X^TX)^{1/2}.
    \end{equation}
    Here the square root is the standard square root of a positive semi-definite matrix, i.e. $A^{1/2}$ is the unique positive semi-definite matrix $B$ such that $B^2 = A$. It is clear that this is a canonicalization by construction, since $y_X$ is symmetric so that $(y_X)^Ty_X = (y_X)^2 = X^TX$. 
    
    The continuity follows immediately from the continuity of the matrix square root on the set of positive semi-definite matrices. This fact is elementary, and can be proven for instance using the Taylor series expansion
    \begin{equation}\label{sqrt-series}
        M^{1/2} = \sum_{n=0}^{\infty} (-1)^{n-1} \frac{(2n)!}{4^n(n!)^2(2n-1)}(M-I_d)^n.
    \end{equation}
    Using Sterling's formula, we easily see that the coefficients satisfy
    $$
        \frac{(2n)!}{4^n(n!)^2(2n-1)} = O\left(\frac{1}{n^{3/2}}\right),
    $$
    which implies that the series \eqref{sqrt-series} converges absolutely whenever $\|M - I_d\| \leq 1$ (here $\|\cdot\|$ denotes the operator norm so that $\|X^k\| \leq \|X\|^k$). This implies that the matrix square root is a continuous function for all positive semi-definite matrices $M$ such that $0 \preceq M\preceq 2I_d$, since for such matrices we clearly have $\|M - I_d\| \leq 1$. Finally, the homogeneity of the matrix square root extends this continuity to all positive semi-definite matrices.
\end{proof}
By appending zero columns to $X$, this immediately implies that $O(d)$ has a continuous canonicalization whenever $n \leq d$. We can also use it in a straightfoward manner to obtain a continuous canonicalization for $SO(d)$ acting on $\RR^{d\times n}$ when $n < d$.
\begin{cor}\label{SO-canonicalization-corollary}
The module $(\RR^{d\times (d-1)},SO(d))$  has a continuous canonicalization.
\end{cor}
\begin{proof}
    This follows since the orbits of $\RR^{d\times (d-1)}$ under the action of $SO(d)$ and $O(d)$ are the same. Indeed, if $X,Y$ are in the same orbit under the action of $SO(d)$ they are clearly in the same orbit under the action of $O(d)$ since $SO(d)\subset O(d)$. On the other hand, suppose that $X,Y$ are in the same orbit under the action of $O(d)$, i.e. that there exists a $U\in O(d)$ such that $Y = UX$. Let $R_Y$ be a reflection which leaves the space spanned by $Y$ invariant (this exists since $Y$ consists of $d-1$ vectors). Then $Y = R_YUX$ and either $U\in SO(d)$ or $R_YU\in SO(d)$. Thus $X,Y$ are in the same orbit under $SO(d)$.

    Since the orbits of $\RR^{d\times (d-1)}$ under the action of $SO(d)$ and $O(d)$ are the same, the canonicalization for $O(d)$ acting on $\RR^{d\times (d-1)}$ gives a canonicalization for $SO(d)$ as well.
\end{proof}

\newpage
\section{Lower bound on permutation robust frames}\label{app:lb}
In this appendix, we state and prove the precise lower bound on the cardinality of a \CWW frame for $(\RR^{d\times n},S_n)$ with $d>1$ mentioned in Section \ref{sec:weighted_frames_for_permutations}. 
% was previously restatable
\begin{prop}\label{lb}
 Any \CWW frame for the module $(\RR^{d\times n},S_n)$ with $d>1$ has cardinality at least $$k(d,n) := (d-1)\lfloor n/2\rfloor + 1 - \sum_{i=1}^{\lfloor n/2\rfloor}(d-1-2^i)_+.$$
\end{prop}
\begin{proof}
    Suppose without loss of generality that $n$ is even. If not, we restrict to the set where the first vector is $0$. 
    
    Suppose that we are given a robust frame $\mu: \mathbb{R}^{d\times n}\rightarrow \Pi(S_n)$, where $\Pi(S_n)$ denotes the space of probability measures on the symmetric group $S_n$. For an element $X\in \RR^{d\times n}$, we let $\tilde{\mu}_X$ denote the pushforward of $\mu_X$ under the coset map $G\rightarrow G/\stab{X}$. Thus $\tilde{\mu}_X$ is a measure on the set of cosets $G/\stab{X}$.
    We will show that there exists a point $X\in \mathbb{R}^{d\times n}$ such that $$|\text{supp}(\tilde\mu_{X})| \geq \frac{(d-1)n}{2} + 1 - \sum_{i=1}^{n/2}(d-1-2^i)_+.$$
    
    Let $x_1,...,x_{n/2}\in \mathbb{R}^d$ be distinct and consider the point \begin{equation}
        X_0 = (x_1,x_1,x_2,x_2,\cdots,x_{n/2},x_{n/2})\in \mathbb{R}^{dn}.
    \end{equation}
    We will inductively construct a sequence of points $X_1,...,X_{n/2}$ of the form
    \begin{equation}
        X_i = (x_1',x_1^*,\cdots x_i',x_i^*,x_{i+1},x_{i+1},\cdots,x_{n/2},x_{n/2}),
    \end{equation}
    where $x_i'\neq x_i^*$ are close to $x_i$ and such that $|\text{supp}(\tilde\mu_{X_i})| \geq |\text{supp}(\tilde\mu_{X_{i-1}})| + \min\{|\text{supp}(\tilde\mu_{X_{i-1}})|, d-1\}$. Since $|\text{supp}(\tilde\mu_{X_{0}})| \geq 1$, we get $$|\text{supp}(\tilde\mu_{X_{n/2}})| \geq \frac{(d-1)n}{2} + 1 - \sum_{i=1}^{n/2}(d-1-2^i)_+,$$
    as desired.

    Suppose that the point $X_i$ can been constructed and let $m = \min\{|\text{supp}(\tilde\mu_{X_{i}})|, d-1\}$. Consider points of the form
    \begin{equation}
        X(v_\epsilon) = (v_1',v_1^*,\cdots v_i',v_i^*,v_{i+1}+v_\epsilon,v_{i+1}-v_\epsilon,\cdots,v_{n/2},v_{n/2})
    \end{equation}
    for a vector $v_\epsilon \in \mathbb{R}^d$. By definition, the continuity of the frame implies that
    \begin{equation}\label{weak-convergence-equation}
        \lim_{v_\epsilon\rightarrow 0} \langle\mu_{X(v_\epsilon)}\rangle_{X_i} = \bar{\mu}_{X_i}
    \end{equation}
    in the weak topology, which coincides with pointwise convergence of the probabilities since $S_n$ is a finite group. 

    Let $C_1,...,C_m\in S_n/\stab{X_i}$ be distinct cosets in the support of $\tilde\mu_{X_i}$, i.e. we have $\mu_{X_i}(C_j) > 0$. Equation \eqref{weak-convergence-equation} implies that for sufficiently small $v_\epsilon$, we will have $\mu_{X(v_\epsilon)}(C_j) > 0$ for all $j = 1,...,m$.

    Observe that for any $v_\epsilon \neq 0$, the stabilizer $H:=\stab{X(v_\epsilon)}$ is independent of $v_\epsilon$ and has index $2$ in $\stab{X_i}$. This means that each coset $C_j$ splits into two cosets of the smaller subgroup $H$, which we denote by $C^+_j, C^-_j\in G/H$. Define the following function $f:\mathbb{R}^d\rightarrow \mathbb{R}^m$
    \begin{equation}
        f(v_\epsilon)_j = \frac{\mu_{X(v_\epsilon)}(C^+_j)}{\mu_{X(v_\epsilon)}(C_j)}.
    \end{equation}
    The function $f$ is well-defined for sufficiently small inputs $v_\epsilon$, since by the previous remark the denomintor is then $ > 0$. Moreover, the continuity of the frame $\mu$ implies that $f$ is a continuous function of $v_\epsilon \neq 0$ in such a sufficiently small neighborhood (since the stabilizer $H$ of $X(v_\epsilon)$ is constant in this neighborhood and so the continuity condition %\eqref{continuity-conditions-definition} 
    becomes continuity of the averaged frame $\bar\mu_{X(v_\epsilon)}$). In addition, the invariance of the frame means that 
    \begin{equation}\label{inversion-relation-frame-equation}
    f(-v_\epsilon)_j = 1 - f(v_\epsilon)_j
    \end{equation}
    since transposing the $2i+1$ and $2i+2$ elements of $X$ maps $X(v_\epsilon)$ to $X(-v_\epsilon)$ and also swaps the cosets $C_j^+$ and $C_j^-$.

    We complete the proof by noting that since $m \leq d-1$, the Borsuk-Ulam Theorem \cite{borsuk1933drei} implies that there must be a point $v_\epsilon$ such that $f(v_\epsilon) = f(-v_\epsilon)$. Combined with \eqref{inversion-relation-frame-equation} this means that $f(v_\epsilon)_j = 1/2$ for all $j$, and so $$\mu_{X(v_\epsilon)}(C_j^+) = \mu_{X(v_\epsilon)}(C_j^-) = 1/2 > 0.$$
    Choosing $X_{i+1} = X(v_\epsilon)$ then completes the inductive step.
\end{proof}

\newpage
\section{Frames for $SO(d)$ and $O(d)$ acting on $\RR^{d\times n}$}\label{app:so3}
In this appendix, we construct \CWW frames for $(\RR^{d\times n},SO(d))$ and for $(\RR^{d\times n},O(d))$ using similar ideas to those laid out in \cite{pozdnyakov2023smooth}. Our contribution is to generalize this construction to all dimensions and to rigorously prove that it preserves continuity.

We begin first with the action of $SO(d)$. For a point cloud $0 \neq X\in \RR^{d\times n}$, the frame $\mu := \mu^{SO(d)}$ will be of the form
\begin{equation}\label{definition-of-mu-SO-d}
    \mu_X = \sum_{i_1 = 1}^n w_{i_1}(X)\sum_{i_2 = 1}^n w_{i_1i_2}(X)\cdots \sum_{i_r = 1}^nw_{i_1i_2\cdots i_r}(X)\delta_{g_{i_1i_2\cdots i_r}(X)},
\end{equation}
where $r = \min(\text{rank}(X), d-1)$, $w_{i_1i_2\cdots i_t}(X)$ for $t$ is a weight associated to the sequence of columns $i_1,i_2,\cdots i_t$ for $t \leq r$, and $g_{i_1i_2\cdots i_r}(X)\in SO(d)$ is a rotation associated to the sequence of columns $i_1,i_2,\cdots i_r$. When $X = 0$, the frame $\mu_X$ can be chosen arbitrarily. 

We proceed to describe the weight functions $w_{i_1\cdots i_t}(X)$ for $t=1,...,r$ and the rotations $g_{i_1\cdots i_r}(X)$.

The rotations $g_{i_1\cdots i_r}(X)$ are defined by 
\begin{equation}\label{equation-for-g-s}
    g_{i_1\cdots i_r}(X)^{-1}\begin{pmatrix}
        x_{i_1},x_{i_2},...,x_{i_r}
    \end{pmatrix} = A,
\end{equation}
where $A$ is an upper triangular $d\times r$ matrix with non-negative diagonal entries which satisfies
\begin{equation}\label{equation-for-gram-matrix-A}
    A^TA = \begin{pmatrix}
        \langle x_{i_1},x_{i_1}\rangle & \cdots & \langle x_{i_1},x_{i_r}\rangle\\
        \vdots & \ddots & \vdots\\
        \langle x_{i_r},x_{i_1}\rangle & \cdots & \langle x_{i_r},x_{i_r}\rangle
    \end{pmatrix}.
\end{equation}
The matrix $A$ is uniquely determined if the columns $x_{i_1},...,x_{i_r}$ are linearly independent. If not, we simply choose one such $A$, and remark that in this case the choice will not matter because the corresponding weights will be equal to $0$.

Another way of thinking about the matrix $A$ is that it is determined by performing Gram-Schmidt orthogonalization on the column vectors $x_{i_1},x_{i_2},...,x_{i_r}$ to obtain an orthonormal set $\hat{x}_{i_1},...,\hat{x}_{i_r}$. If the vectors $x_{i_1},x_{i_2},...,x_{i_r}$ are linearly dependent, we must modify Gram-Schmidt as follows. If $x_{i_t}\in \text{span}(x_{i_1},...,x_{i_{t-1}})$, then in the $t$-th step of Gram Schmidt we simply choose $\hat{x}_{i_t}$ to be an arbitrary unit vector orthogonal to $\hat{x}_{i_1},...,\hat{x}_{i_{t-1}}$ (this is where the non-uniqueness comes in).

If we perform this modified Gram-Schmidt orthogonalization procedure on the columns $x_{i_1},...,x_{i_r}$ to obtain an orthonormal set $\hat{x}_{i_1},...,\hat{x}_{i_r}$, which we then complete to an orthonormal basis, then the columns of $A$ correspond to the representation of $x_{i_1},...,x_{i_r}$ with respect to this basis. From this it becomes clear that the first $t$ columns of $A$ only depend upon $x_{i_1},...,x_{i_t}$, and that the first $t$ columns of $A$ are a continuous function of $x_{i_1},...,x_{i_t}$ on the set where these vectors are linearly independent. These facts will become important later in the proof of continuity.

Since the vectors $x_{i_1},...,x_{i_r}$ have the same inner products as the first $r$ columns in $A$, there exists an orthogonal transformation satisfying \eqref{equation-for-g-s}. Moreover, because $r\leq d-1$ this orthogonal transformation can be chosen to lie in $SO(d)$ (by reflecting across the plane spanned by the columns of $A$ if necessary). Thus a rotation $g_{i_1\cdots i_r}(X)$ satisfying \eqref{equation-for-g-s} always exists, although it is only uniquely defined up to left multiplication by the stabilizer of the columns $x_{i_1},x_{i_2},...,x_{i_r}$. In defining $g_{i_1\cdots i_r}(X)$ we simply choose any rotation satisfying \eqref{equation-for-g-s}. As we will see, the weights $w_{i_1,...,i_t}$ will be chosen so that if $x_{i_1},x_{i_2},...,x_{i_r}$ have a larger stabilizer than the whole point cloud $X$, then the total weight corresponding to $g_{i_1\cdots i_r}(X)$ will be $0$.

Next, we describe the weight functions. If $x_{i_1},...,x_{i_{t-1}}$ are linearly independent, the weights $w_{i_1\cdots i_t}(X)$ are defined by
\begin{equation}\label{weight-definition-lin-indep-case}
    w_{i_1\cdots i_t}(X) = \frac{\Delta(x_{i_1},...,x_{i_t})}{\sum_{j=1}^n \Delta(x_{i_1},...,x_{i_{t-1}}, x_j)},
\end{equation}
where 
$$
    \Delta(v_1,...,v_t) = \sqrt{\left|\det{\begin{pmatrix}
        \langle v_1,v_1\rangle & \cdots & \langle v_1,v_t\rangle\\
        \vdots & \ddots & \vdots\\
        \langle v_t,v_1\rangle & \cdots & \langle v_t,v_t\rangle
    \end{pmatrix}}\right|}
$$
is the area of the parallelopiped spanned by $v_1,...,v_t$. We remark that if $x_{i_1},...,x_{i_{t-1}}$ are linearly independent, then since $X$ has rank at least $t$ there is some $j$ such that $\Delta(x_{i_1},...,x_{i_{t-1}}, x_j) > 0$ so that the weight in \eqref{weight-definition-lin-indep-case} is well-defined. We remark that we could have used a cutoff function $\phi_\eta$ as in Section \ref{sub:so2} and \cite{pozdnyakov2023smooth} in the definition of the weights \eqref{weight-definition-lin-indep-case}. However, for simplicity of presentation we stick with the raw areas, although a cutoff function may be desirable in a practical implementation. The following construction and proof carry over easily to the case where a cutoff function is used with minor modifications. 

If $x_{i_1},...,x_{i_{t-1}}$ are linearly dependent, then we simply set $w_{i_1\cdots i_{t-1}j}(X) = 1/n$ for $j=1,...,n$. From these definitions, it is clear that all weights are non-negative, and for any indices $i_1,...,i_{t-1}$ we have
\begin{equation}
    \sum_{i_t = 1}^nw_{i_1i_2\cdots i_t}(X) = 1.
\end{equation}
From this it follows that
\begin{equation}
    \sum_{i_1 = 1}^n w_{i_1}(X)\sum_{i_2 = 1}^n w_{i_1i_2}(X)\cdots \sum_{i_r = 1}^nw_{i_1i_2\cdots i_r}(X) = 1,
\end{equation}
so that $\mu$ is a well-defined frame. Moreover, the cardinality of $\mu$ is equal to the maximum number of sequences of indices $i_1,...,i_r$ for $r\leq d-1$ such that $x_{i_1},...,x_{i_r}$ are linearly independent. Clearly in this case we can have no repeated indices so that the cardinality of $\mu$ is $n(n-1)\cdots (n-d+2)$. The main result of this Section is that the frame $\mu$ defined in this way is a robust frame.

\begin{restatable}{prop}{mud}
The frame $\mu^{SO(d)}$ is a \CWW frame.
\end{restatable}
\begin{proof}
    We first verify that the frame $\mu := \mu^{SO(d)}$ is weakly equivariant. Let $g\in SO(d)$ and $X\in \RR^{d\times n}$.
    Since the action of $SO(d)$ preserves both inner products and the rank of $X$, it follows that all of the weights in \eqref{definition-of-mu-SO-d} are invariant under the action of $g$. Also, since inner products are preserved it follows from \eqref{equation-for-g-s} and \eqref{equation-for-gram-matrix-A} that
    \begin{equation}\label{g-ij-relation-1467}
        g_{i_1\cdots i_r}(X)^{-1}\begin{pmatrix}
        x_{i_1},x_{i_2},...,x_{i_r}
        \end{pmatrix} = g_{i_1\cdots i_r}(gX)^{-1}g\begin{pmatrix}
        x_{i_1},x_{i_2},...,x_{i_r}\end{pmatrix}.
    \end{equation}
    Observe also that if we define total weights via
    \begin{equation}
        W_{i_1,...,i_r}(X) = w_{i_1}(X)w_{i_1i_2}(X)\cdots w_{i_1\cdots i_r}(X),
    \end{equation}
    then \eqref{weight-definition-lin-indep-case} implies that $W_{i_1,...,i_r}(X) > 0$ iff $G_X = G_{(x_1,...,x_r)}$, i.e. if the whole point cloud $X$ and the columns $(x_1,...,x_r)$ have the same stabilizer. Indeed, $W_{i_1,...,i_r}(X) > 0$ iff $x_{i_1},...,x_{i_r}$ span a subspace of dimension $r$, which must coincide with the span of $X$ if $X$ has rank $ < d$. In this case, the stabilizer of $X$ and $(x_{i_1},...,x_{i_r})$ consist of all rotations that fix this subspace. If $X$ has rank $d$, then its stabilizer is trivial, and $x_{i_1},...,x_{i_r}$ span a space of dimension $d-1$ so that the stabilizer consists of all rotations which fix this $(d-1)$-dimensional subspace. But any rotation fixing a $(d-1)$-dimensional subspace must be trivial so that the stabilizer of $(x_{i_1},...,x_{i_r})$ is also trivial in this case.
    
    Utilizing this, we see that \eqref{g-ij-relation-1467} implies that
    \begin{equation}\label{g-ij-coset-relation-1479}
        g_{i_1\cdots i_r}(X)^{-1}G_X = g_{i_1\cdots i_r}(gX)^{-1}gG_X\implies G_Xg_{i_1\cdots i_r}(X) = G_Xg^{-1}g_{i_1\cdots i_r}(gX)
    \end{equation}
    for every sequence of indices $(i_1,...,i_r)$ for which $W_{i_1,...,i_r}(X) > 0$. 
    
    Plugging the invariance of the weights into \eqref{definition-of-mu-SO-d}, we see upon averaging over the stabilizer that
\begin{equation}
    \bar{\mu}(gX) = \int_{G_{gX}}\sum_{i_1,...,i_r = 1}^n W_{i_1,...,i_r}(X)\delta_{sg_{i_1\cdots i_r}(gX)} ds.
\end{equation}
Now, we use that the stabilizers satisfy $G_{gX} = gG_{X}g^{-1}$ and the relation \eqref{g-ij-coset-relation-1479} (which holds whenever $W_{i_1,...,i_r}(X) > 0$) to rewrite this as
\begin{equation}
    \bar{\mu}(gX) = \int_{G_{X}}\sum_{i_1,...,i_r = 1}^n W_{i_1,...,i_r}(X)\delta_{gsg^{-1}g_{i_1\cdots i_r}(gX)} ds = \int_{G_{X}}\sum_{i_1,...,i_r = 1}^n W_{i_1,...,i_r}(X)\delta_{gsg_{i_1\cdots i_r}(X)} ds = g^*\bar{\mu}(X),
\end{equation}
as desired.

Next, we prove that the frame $\mu$ is continuous. To do this, fix $X\in \RR^{d\times n}$ and suppose that $X$ has rank $r$. If $r = 0$, i.e. if $X = 0$, then the stabilizer of $X$ is all of $SO(d)$ and so continuity at $X$ follows trivially from Definition \ref{defn:continuity-of-weighted-frame} since we are averaging over the whole group. 

So suppose that $X$ has rank $r > 0$. We first observe that for sufficiently small $\epsilon > 0$ (depending upon $X$), $|Y - X| < \epsilon$ implies that if $x_{i_1},...,x_{i_r}$ are linearly independent, then $y_{i_1},...,y_{i_r}$ are also linearly independent. In particular, $|Y - X| < \epsilon$ implies that $\text{rank}(Y) \geq r$. For such a $Y$, the `marginal' weights
\begin{equation}
    W_{i_1,...,i_r}(Y) = w_{i_1}(Y)w_{i_1i_2}(Y)\cdots w_{i_1\cdots i_r}(Y) = \sum_{i_{r+1}=1}^n\sum_{i_R=1}^nW_{i_1,...,i_R}(Y),
\end{equation}
where $R = \text{rank}(Y)$ are well-defined. We first claim that
\begin{equation}\label{limit-equation-weights-1520}
    \lim_{Y\rightarrow X} W_{i_1,...,i_r}(Y) = W_{i_1,...,i_r}(X)
\end{equation}
for any sequence of indices $i_1,...,i_r$.
To prove this, suppose first that $W_{i_1,...,i_r}(X) > 0$. In this case, $x_{i_1},x_{i_2},...,x_{i_r}$ are linearly independent, so that by \eqref{weight-definition-lin-indep-case} each of the weight functions $w_{i_i,...,i_t}(X)$ for $1\leq t\leq r$ is continuous in a neighborhood of $X$. This implies \eqref{limit-equation-weights-1520}. If on the other hand $W_{i_1,...,i_r}(X) = 0$, then let $t$ be the first index such that $w_{i_i,...,i_t}(X) = 0$. Then $x_{i_1},...,x_{i_{t-1}}$ are linearly independent and \eqref{weight-definition-lin-indep-case} shows that $w_{i_i,...,i_t}(X)$ is continuous in a neighborhood of $X$. This means that
$$
    \lim_{Y\rightarrow X} w_{i_i,...,i_t}(Y) = 0.
$$
Since the weight functions are all bounded, we get
$$
    \lim_{Y\rightarrow X} W_{i_1,...,i_r}(Y) = 0 = W_{i_1,...,i_r}(X)
$$
as desired.

The final ingredient we need to prove continuity is to observe that if $R = \text{rank}(Y) \geq r$, then for any indices $i_1,...,i_R$ the product
\begin{equation}\label{special-orthogonal-product-1535}
    g_{i_1,...,i_R}(Y)^{-1}\begin{pmatrix}
        y_{i_1},y_{i_2},...,y_{i_r}
    \end{pmatrix}
\end{equation}
is independent of final indices $i_{r+1},...,i_R$. This is due to the fact (mentioned earlier) that the first $r$ columns of the matrix $A$ in \eqref{equation-for-gram-matrix-A} only depend upon the first $r$ vectors $y_{i_1},y_{i_2},...,y_{i_r}$. Moreover, as mentioned earlier we also have that the product in \eqref{special-orthogonal-product-1535} (i.e. the first $r$ columns of the matrix $A$ in \eqref{equation-for-gram-matrix-A}) is a continuous function of $y_{i_1},...,y_{i_r}$ on the set where $y_{i_1},...,y_{i_r}$ are linearly independent.

We can now complete the proof of continuity. Let $i_1,...,i_r$ be a sequence of indices such that $W_{i_1,...,i_r}(X) > 0$. This means that $x_{i_1},...,x_{i_r}$ are linearly independent and if $|Y - X| < \epsilon$, then $y_{i_1},...,y_{i_r}$ are linearly independent as well. The continuity of the Gram-Schmidt procedure (assuming linear independence) implies that for any set of indices $i_{r+1},...,i_R$ where $R$ is the rank of $Y$, we have
\begin{equation}
    \lim_{Y\rightarrow X}g_{i_1,...,i_R}(Y)^{-1}\begin{pmatrix}
        y_{i_1},y_{i_2},...,y_{i_r}
    \end{pmatrix} = g_{i_1,...,i_r}(X)^{-1}\begin{pmatrix}
        x_{i_1},x_{i_2},...,x_{i_r}
    \end{pmatrix}.
\end{equation}
Note that here the rank $R$ may depend upon $Y$ in the above limit.
Since $y_i\rightarrow x_i$ and every $g\in SO(d)$ is an isometry, this implies that
\begin{equation}
    \lim_{Y\rightarrow X} [g_{i_1,...,i_R}(Y)^{-1} - g_{i_1,...,i_r}(X)^{-1}]\begin{pmatrix}
        x_{i_1},x_{i_2},...,x_{i_r}
    \end{pmatrix} = 0.
\end{equation}
Here the left term is viewed as a matrix in $\RR^{d\times d}$. 

If $r < d-1$, then since $x_{i_1},...,x_{i_r}$ is a basis for the range of $X$, we get that
\begin{equation}\label{limit-equation-1557}
    \lim_{Y\rightarrow X}[g_{i_1,...,i_R}(Y)^{-1} - g_{i_1,...,i_r}(X)^{-1}]v = 0
\end{equation}
uniformly for $v$ in any compact subset of $\text{range}(X)$. If $r = d-1$, then \eqref{limit-equation-1557} actually holds uniformly for $v$ in any compact subset of $\RR^d$, since both matrices on the left hand side above are in $SO(d)$. 

In either case, by averaging over $G_X$ we get that
\begin{equation}
    \lim_{Y\rightarrow X}\int_{h\in G_X}[g_{i_1,...,i_R}(Y)^{-1} - g_{i_1,...,i_r}(X)^{-1}]hv = 0
\end{equation}
uniformly on compact subsets of the whole space $\RR^d$. This holds since if $r < d-1$, then $$\int_{h\in G_X} hv\in \text{range}(X)$$ for all $v\in \RR^d$, while if $r = d-1$ then \eqref{limit-equation-1557} already holds for all $v\in \RR^d$. Putting this together, we get that
\begin{equation}
    \lim_{Y\rightarrow X} \int_{h\in G_X}\delta_{h^{-1}g_{i_1,...,i_R}(Y)}dh = \int_{h\in G_X}\delta_{h^{-1}g_{i_1,...,i_r}(X)}dh
\end{equation}
in the weak topology on $SO(d)$. Since this holds for all sets of indices $i_1,...,i_R$ for which $W_{i_1,...,i_r}(X) > 0$, and we have $\lim_{Y\rightarrow X} W_{i_1,...,i_r}(Y) = W_{i_1,...,i_r}(X)$ for all indices $i_1,..,i_r$, we finally get using the definition \eqref{definition-of-mu-SO-d} that
\begin{equation}
    \langle\mu_Y\rangle_X = \int_{h\in G_X}h^*\mu_Y\rightarrow \int_{h\in G_X}h^*\mu_X = \langle\mu_X\rangle_X
\end{equation}
as $Y\rightarrow X$, in the weak topology on $SO(d)$, which proves the continuity of the frame $\mu$.
\end{proof}

We remark that essentially the same construction can be used to obtain a frame for the action of $O(d)$ on $\RR^{d\times n}$ which has cardinality $2n(n-1)\cdots (n-d+2)$. Specifically, setting $r = \text{rank}(X)$ in \eqref{definition-of-mu-SO-d}, and repeating the exact same argument we get a robust frame for $O(d)$. This frame has cardinality $2n(n-1)\cdots (n-d+2)$, since the orthogonal transformation $g_{i_1\cdots i_d}(X)$ is determined up to a reflection through the plane spanned by $x_{i_1},...,x_{i_{d-1}}$ when the columns $x_{i_1},...,x_{i_d}$ are linearly independent.

Finally, we remark that we do not know whether these constructions are optimal. In specific instances, such as when $n = d-1$ for $SO(d)$, we know they are sub-optimal since in this case a canonicalization exists (see Corollary \ref{SO-canonicalization-corollary}. Determining the smallest possible weighted frames in these cases is an interesting further research direction.

\subsection{$SO(3)$ and $O(3)$ equivariance via stable functions}
In this section, we specialize to the case of $\RR^3$ and discuss how to endorce equivariance using the previously constructed frames.
To achieve continuous, equivariant, universal models for functions $f:\RR^{3\times n}\to \RR^{3\times n}$ with respect to the action of $O(3)$ or $SO(3)$, we need to characterize the space of stable functions. This can be done using results from \citet{scalars}. Namely, in the $O(3)$ case we will have that $\stab{f(X)}\subseteq \stab{X}$ if and only if all columns of $f(X)$ are in the linear space spanned by the points of $X$. Thus, the stable functions in the $O(3)$ case can be parameterized as 
\begin{equation}\label{eq:O3stable}
f_k(X)=\sum_j f^k_j(X)x_j .
\end{equation}
Moreover, the space of functions of this form with \emph{continuous} coefficients $f^k_j$ is dense (over compact sets) in the space of continuous equivariant functions (follows from Proposition 4 in \cite{scalars})

In the $SO(3)$ case, stable functions have a slightly more complex parameterization, in which each column of $f(X)$ is of the form
\begin{equation}\label{eq:SO3stable}
f(X)_k=\sum_j f^k_{j}(X)x_j+\sum_{i,j}f^k_{i,j}(X)(x_i \times x_j),
\end{equation}
where $f^k_j(X)$ and $f^k_{i,j}(X)$ are arbitrary, scalar-valued functions. Moreover the space of functions of this form with \emph{continuous} coefficients $f^k_j$ and $f^k_{i,j}(X)$  are dense (over compact sets) in the space of continuous equivariant functions (follows from Proposition 5 in \cite{scalars})

 Continuous equivariant universal models for $O(3)$ can thus be obtained by applying $\Ew$ %induced from a $O(3)$ CWW frame 
 to functions of the form \eqref{eq:O3stable}, where the $f_j$ are taken from some dense function space $Q$. Continuous equivariant universal models for $SO(3)$ can be obtained analogously via \eqref{eq:SO3stable}. %\textbf{}in the same way by replacing \eqref{eq:O3stable} with the characterization in \eqref{eq:SO3stable}.

\section{Equivariant Projections via Stable Frames}\label{app:stableFrame}
To obtain BEC projection operators from robust frames in the equivariant setting, we add the requirement that, replacing $\mu_v$ with $\bar \mu_v$ in the definition of $\Ew$, will not affect the operator. This requirement is satisfied automatically in the invariant case, but in the equivariant case it is an additional condition we must enforce.

% \begin{defn}[stable measure]
	% \nd{mess}
	% $$F(g)=gf(g^{-1}v) $$
	% is, for a fixed $v$, a function $F:G\to W$ which is $\stab{v}$ equivariant. As opposed to 
	% $$H(g)=f(g^{-1}v) $$
	% which is $\stab{v}$ invariant.
	% \nd{bla}
	
	% Let $(W,G)$ be a module and let $S\leq G$ be a compact subgroup. We say that a Borel probability measure $\mu$ on $G$ is $S$-stable, if for every $S$ equivariant function $F:G\to W $,
	% $$\int_G F(g)d\mu(g)= \int_G F(g)d\langle\mu\rangle_S(g)ds $$
	% \end{defn}
% \hl{example/intuition?}
\begin{defn}[Stable robust frame]
	Let $(V,G)$ and $(W,G)$ be modules. 
	We say that a robust frame $\muF$ %$v\mapsto \mu_v$ 
	is  \emph{stable} at a point $v$,  if for every sequence $v_n \rightarrow v$ and for every
	$f:V\to W$
	\begin{align*}
	\int gf(g^{-1}v)d\mu_v(g)&=\int gf(g^{-1}v)d\bar{\mu}_v(g) \\
	\lim_{k\rightarrow \infty}\int gf(g^{-1}v)d\mu_{v_k}(g)&=\int gf(g^{-1}v)d{\mu}_v(g) \\
	\end{align*}
	We say  that $\muF$ is a stable robust frame if it is robust and stable at all points $v\in V$.
\end{defn}
\begin{remark}\label{rem:eqeq} Note that the second requirement above resembles the standard definition of weak convergence, but it requires only convergence of integrals $\int F(g)d\mu_{v_k} \rightarrow \int F(g) d\mu(g) $ for functions of the form $F(g)=gf(g^{-1}v) $. Not all functions on $G$ are of this form. Note for example that every function of this form is $\stab{v}$ equivariant since for all $s\in\stab{v}$ we have 
	$$F(sg)=sgf(g^{-1}s^{-1}v)=sgf(g^{-1}v)=sF(g) $$
In particular, to check whether a robust frame is stable, it is sufficient to check whether for every $v$, and every $\stab{v}$ equivariant $F:G\to W $, we have
\begin{align*}
	\int F(g)d\mu_v(g)&=\int F(g)d\bar{\mu}_v(g) \\
	\lim_{k\rightarrow \infty}\int F(g)d\mu_{v_k}(g)&=\int F(g)d{\mu}_v(g) 
\end{align*}
	for every sequence $v_k$ converging to $v$.
	\end{remark}

\begin{restatable}{prop}{stableFrame}
	Let $(V,G)$ and $(W,G)$ be modules, and  $\mu$  a stable robust frame. Then $\E:F(V,W)\to F(V,W)$ is a BEC operator.
\end{restatable}
\begin{proof}
Let $f:V\to W$ be a  function. We can now show that $\Ew[f]$ is an equivariant function because for every $h\in G$ and $v\in V$,
\begin{align*}
	\Ew[f](hv)&=\int_G gf(g^{-1}hv)d\mu_{hv}(g)  \\
	&=\int_G gf(g^{-1}hv)d\bar{\mu}_{hv}(g)\\
	&=\int_G gf(g^{-1}hv)dh_*\bar{\mu}_{v}(g)\\
	&=\int_G hgf((hg)^{-1}hv)d\bar{\mu}_{v}(g)\\
	&=h\left[\int_G gf(g^{-1}v)d\bar{\mu}_{v}(g) \right]\\
	&=h\left[\int_G gf(g^{-1}v)d\mu_{v}(g)\right]\\
	&=h\Ew[f](v).
\end{align*}
If $f$ is equivariant then $\Ew[f]=f$ since $gf(g^{-1}v)=f(v)$ from equivariance.

Boundedness also follows easily since $d\mu_v$ is a probability measure, so that on every compact $K\subseteq V$ which is also closed under the action of $G$, we have for every $v\in K$ that
\begin{equation}
	|\Ew[f](v)| \leq \left|\int_G gf(g^{-1}v)d\mu_v(g)\right| \leq \max_{g\in G}\|g\|\max_{w\in K}|f(w)|,
\end{equation}
where $\|g\|$ denotes the operator norm of the linear operator $g$ (which is bounded, since $G$ is compact and acting linearly and continuously). 

\textbf{Continuity}
Let $f:V\to \RR$ be a continuous functions. Let $v_k$ be a sequence converging to $v\in V$. We need to show that $\Ew[f](v_k)$ converges to $\Ew[f](v)$. We observe that
\begin{align*}
	&|\Ew[f](v_k)-\Ew[f](v)|\\
	&=\left|\int gf(g^{-1}v_k)d\mu_{v_k}(g)-\int gf(g^{-1}v)d\mu_{v}(g)\right|  \\
	&\leq\left|\int gf(g^{-1}v_k)d\mu_{v_k}(g)-\int gf(g^{-1}v)d\mu_{v_k}(g)\right| + \left|\int gf(g^{-1}v)d\mu_{v_k}(g)-\int gf(g^{-1}v)d\bar{\mu}_{v}(g)\right|.
\end{align*}
The second term tends to $0$ from the definition of a stable frame. The first term tends to zero because $gf(g^{-1}v_k) $ converges to $gf(g^{-1}v) $ uniformly in $g$ as $k$ tends to infinity.
\end{proof}
%NOTE TO SELF. COMPARE THIS PROOF TO THE INVARIANT CASE

We conclude this appendix with an explanation of how stable robust frames can be constructed for our $SO(2)$ and $SO(3)$ examples.
\begin{example}
	For the action of $S^1$ on $\CC^n $ we double the size of the weighted frame $\muTwo$ and define 
	$$g_{i,+}(Z)=g_i(Z) \text{ and } g_{i,-}(Z)=-g_i(Z) $$
	$$w_{i,+}(Z)=w_{i,-}(Z)=\frac{1}{2}w_i(Z) . $$
	We then obtain the projection operator
	$$Q_2(f)(Z)=\sum_{i=1}^n \sum_{s\in \{-1,1\}}  w_{i,s}(Z) g_{i,s}(Z)\cdot  f(g_{i,s}^{-1}(Z)\cdot Z) $$
	We need to check that this frame is a stable robust frame. At any non-zero point, since the stabilizer is trivial, the robustness of the original frame implies stable robustness since the notions are equivalent. 
	
	The main interesting point is at zero, where we have a non-trivial stabilizer. To show that the frame $\muF$ we defined is stable at $0_n$, it is sufficient to show, using Remark \ref{rem:eqeq}, that  for every $\stab{0_n}=S^1$ equivariant function $F:S^1 \to \CC^n $, we will have that 
	$$\int F(g)d\mu_0(g)=\int F(g)d\langle \mu_0 \rangle_0(g)=0 $$
	and that for every $Z_k \rightarrow 0$ we will have that 
	$$\lim_{k\rightarrow \infty}\int F(g)d\mu_{Z_k}(g)=\int F(g)d\mu_0(g)=0. $$
	Indeed, using the equivariance of $F$ we have
	$$\int F(g)d\mu_0(g)=\frac{1}{2}(F(1)+F(-1))=\frac{1}{2}(F(1)-F(1))=0, $$
	and 
	$$\int F(g)d\mu_{Z_k}(g)=\frac{1}{2}\sum_{i=1}^n w_i(Z_k)(F(g_i(Z_k))+F(-g_i(Z_k))=\frac{1}{2}\sum_{i=1}^n w_i(Z_k)(F(g_i(Z_k))-F(g_i(Z_k))=0.$$
	%\nd{ we'd like to say this is a stable frame. We'd need to %define it first...}
	%Note that
	%$$Q_2(f)(0)=\frac{1}{2}(f(0)-f(0))=0$$
	%so we have equivariance at zero. We can also check continuouity and equivariance everywhere else. 
\end{example}

\begin{example} We now define a stable weakly equivariant weighted frame for  the module  $(\RR^{3\times n},SO(3))$. This cardinality of the stable version  $\muThreeStable$ of $\muThree$ will be four times larger.  
	$$\muThreeStable(X)= \sum_{j\neq i}\sum_{k=1}^4 \frac{1}{4} w_i(X)w_{ij}(X)\delta_{g_{ij}d^{(k)}(X)} $$
	where $d^{(k)}, k=1,2,3,4$ are the four  diagonal matrices in $SO(3)$, that is 
	\begin{align*}
		d^{(1)}&=I_3, d^{(2)}=\begin{pmatrix}
			1 & 0 & 0\\
			0 & -1 & 0\\
			0 & 0 & -1
		\end{pmatrix}\\  
		d^{(3)}&=\begin{pmatrix}
			-1 & 0 & 0\\
			0 & -1 & 0\\
			0 & 0 & 1
		\end{pmatrix}
		d^{(4)}=\begin{pmatrix}
			-1 & 0 & 0\\
			0 & 1 & 0\\
			0 & 0 & -1
		\end{pmatrix}
	\end{align*}
	
	Let us explain why this frame is stable at points with non-trivial stabilizer. At $X=0$ the reasoning is the same as in the $SO(2)$ example. If $X$ has rank $1$, then all points $x_i$ with $x_i=0$ have weight $0$, and at all other points $x_i$ $g_{ij}^{-1} $ is a  rotation taking $x_i$ to $\|x_i\|e_1$, which means that $g_{ij}d^{(k)}g_{ij}^{-1}\in \stab{X}$ for $k=1,2$. Now, for any $\stab{X}$ equivariant function $F:SO(3)\to \RR^{3\times n}$ and $X$ with rank one, we obtain (denoting $g_{ij}=g_{ij}(X)$ and $\hat{w}_{ij}=w_i(X)w_{ij}(X) $ for simplicity)
	\begin{align*}
		\int_G F(g)d\mu_X(g)&=\frac{1}{4}\sum_{i: x_i\neq 0} \sum_j \hat{w}_{ij} \sum_{k=1}^4 F(g_{ij}d^{(k)})\\
		&=\frac{1}{4}\sum_{i: x_i\neq 0} \sum_j \hat{w}_{ij} ( F(g_{ij})+F(g_{ij}d^{(2)}g_{ij}^{-1}g_{ij}) \\
		&+F(g_{ij}d^{(2)}g_{ij}^{-1}g_{ij}d^{(4)})+F(g_{ij}d^{(4)}) )\\
		&=\frac{1}{4}\sum_{i: x_i\neq 0} \sum_j \hat{w}_{ij}( (d^{(1)}+g_{ij}d^{(2)}g_{ij}^{-1})F(g_{ij})\\
		&+(d^{(1)}+g_{ij}d^{(2)}g_{ij}^{-1})F(g_{ij}d^{(4)}) )
	\end{align*}
	The stability then follows from the fact that for all $w\in W$, 
	$$\int_{\stab{X}}s w ds=\frac{1}{2}g_{ij}(d^{(1)}+d^{(2)})g_{ij}^{-1} w ,$$
	so that the expression above is given by 
	$$\int_G F(g)d\mu_X(g)=\frac{1}{2}\int_{\stab{X}}s\left(\sum_{i: x_i\neq 0} \sum_j \hat{w}_{ij}( F(g_{ij})
	+F(g_{ij}d^{(4)}) ) \right)$$
and therefore is equal to 
$$\int_G F(g)d\bar{\mu}_X(g)=\int_{\stab{X}} \int_G F(sg)d\mu_X(g)= \int_{\stab{x}} s\left(\int_G F(g)d\mu_X(g) \right) .$$

\end{example}

\section{Experiments}\label{app:sec_experiments}
\subsection{Empirically verifying discontinuities}\label{app:subsec_verifying_discontinuities}
We also include a practical demonstration of a discontinuity in a canonicalization code library, equiadapt (associated with \citet{mondal2023adaptation} and \citet{kaba2022learned}). We consider their point cloud implementation, which is $O(3)$-equivariant. They use a vector neuron architecture composed with the Gram-Schmidt algorithm (which we together call $C$) to learn an orthogonal matrix, and compose it with a downstream non-equivariant network (which we call $f$). We train their network on the ModelNet40 dataset, and demonstrate that the resultant pipeline is discontinuous at certain points. (Our theoretical results show that, for the groups we consider, no such canonicalization method can always preserve continuity; this demonstration shows that they do not preserve continuity in practice, for the particular $f$ that is learned.) More specifically, we consider an input point cloud $x$, and a “bad” point cloud $b$ generated such that every point points in the same direction (i.e. the matrix of points is rank 1, so this point cloud has infinite stabilizer in $O(3)$). We predict that the pipeline $f(C(\cdot))$ is discontinuous at the input point cloud $b$, and verify this experimentally as follows. First, note that $C(b)$ itself yields a “NaN” output.

To further verify the discontinuity, we compute the average pairwise normalized $L_2$ distance between points close to $b$, which are generated as convex combinations $\epsilon p + (1-\epsilon)*b$ for many randomly chosen point clouds $p$ (generated such that they have trivial stabilizer with probability 1). For the sake of comparison, we also repeat this exact process, with a random asymmetric point cloud $g$ replacing $b$. If $f \circ C$ is discontinuous at $b$, we expect that there is no valid limiting value at $b$, and therefore the average pairwise distance nearby $b$ should be much larger than it is for $g$. As shown in Table \ref{tab:discontinuity_in_pracice}, this is indeed what we observe. We find that both the canonicalization $C$ and the composition $f \circ C$ show strong evidence of discontinuity.

\begin{table}[h]
    \centering
    \begin{tabular}{c|c|c}
        Pairwise Error Metric &	$x_1$, $x_2$ near a singularity $b$	& $x_1$, $x_2$ near a generic point $g$  \\
        \hline
        $ \frac{| | C(x_1) - C(x_2) ||}{ || C(x_1) ||}$ &	1.1088	&1.7035e-5 \\
        $ \frac{| |f(C(x_1)) - f(C(x_2))||}{ || f(C(x_1)) ||}$	& 0.0406	& 0.0009
    \end{tabular}
    \caption{Average distance between pairs of points, near a singular point cloud and near a random point cloud.}
    \label{tab:discontinuity_in_pracice}
\end{table}

It is also straightforward to understand the source of this behavior for the particular architecture used in equiadapt. In particular, the Gram-Schmidt procedure is very unstable when the vectors $v_1$ and $v_2$ it is provided are close to linearly dependent. And, since $v_1$ and $v_2$ are coming from an equivariant network applied to an input that is a perturbation away from having stabilizer isomorphic to $SO(2)$, they are nearly linearly dependent. This simple experiment verifies that discontinuity arising from canonicalization is not just a hypothetical, but a real problem for practically-used architectures after training.

\newpage
\section{Proofs}\label{proofs-appendix}
\projuniv*
\begin{proof}
It is clear directly from the definition that $\E(Q)$ contains only continuous functions, so we only need to prove the density. Let $K\subset V$ be a compact subset and $\epsilon > 0$ be given. We need to show that for every $f\in \Cequi(V,W)$, there exists a $q\in Q$ such that
$$
    \sup_{x\in K}|f(x) - \E(q)(x)| < \epsilon.
$$
Since $Q$ is dense in $C(V,W)$, there exists a $q\in Q$ (not necessarily equivariant), such that $\sup_{x\in K}|f(x) - q(x)| < \delta$ for any $\delta > 0$. Since $\E$ is bounded and $\E(f) = f$ (since $f$ is equivariant), we get the following bound
\begin{equation}
    \sup_{x\in K}|f(x) - \E(q)(x)| = \sup_{x\in K}|\E(f)(x) - \E(q)(x)| \leq \|\E\|\sup_{x\in K}|f(x) - q(x)| < \|\E\|\delta.
\end{equation}
Choosing $\delta$ sufficiently small completes the proof.
\end{proof}

\canon*
\begin{proof}
Let $v_n$ be an arbitrary convergent sequence in $V$, i.e. $\lim_{n\rightarrow \infty} v_n = v$. If $y$ is continuous, then $y_{v_n}\rightarrow y_v$, so that if $f$ is continuous, we have
\begin{equation}
    \lim_{n\rightarrow \infty} \Ican[f](v_n) = \lim_{n\rightarrow \infty} f(y_{v_n}) = f(y_v) = \Ican[f](v).
\end{equation}
The proves that $\Ican$ preserves continuity if the canonicalization $y$ is continuous.

For the reverse direction, suppose that $y$ is not continuous. This means that there exists a convergent sequence $v_n$ with $\lim_{n\rightarrow \infty} v_n = v$, but such that $y_{v_n}$ does not converge to $y_v$. By choosing an appropriate subsequence, we may assume that there exists an $\epsilon > 0$ such that $|y_{v_n} - y_v| > \epsilon$ for all $n$. Let $f$ be a continuous function such that $f(y_v) = 1$ and $f(w) = 0$ for all $w\in V$ satisfying $|w - y_v| > \epsilon$ (such a function exists by well-known considerations, for example the Urysohn Lemma). Then we will have
\begin{equation}
    \lim_{n\rightarrow \infty} \Ican[f](v_n) = \lim_{n\rightarrow \infty} f(y_{v_n}) = 0 \neq 1 = f(y_v) = \Ican[f](v).
\end{equation}
Hence the canonicalization of $\Ican[f]$ is not continuous and so $\Ican$ does not preserve continuity.
\end{proof}

We stated the following proposition in the main text:
\OtwoandSOtwo*
We in fact prove the following two slightly more precise, separate results for $SO(d)$ and $O(d)$, which subsume the previous result. In the proofs, we will also abuse notation slightly and simply write $y$ for both the canonicalization $V \rightarrow V$ and the induced map on the quotient space that we previously called $\ty:V/G \rightarrow V$.

\begin{prop}\label{prop:SOno}
   If $n \geq d \geq 2$, $SO(d)$ acting on $\RR^{d\times n}$ %the module $(\RR^{d\times n},SO(d)) $  
   does not have a continuous canonicalization.
\end{prop}

\begin{proof}
    Observe first that by considering the subspace
    \begin{equation}
        Y = \{(x_1,...,x_d,0,...,0),~x_i\in \mathbb{R}^d\}\subset \mathbb{R}^{d\times n}
    \end{equation}
    a continuous canonicalization for $\mathbb{R}^{d\times n}$ gives a continuous canonicalization for $\mathbb{R}^{d\times d}$, so it suffices to consider the case $n = d$.
    
    Consider first the case $d = 2$. Denote $G=SO(2)$. Consider the subspace $Y\subset \RR^{2\times 2}$ defined by
    \begin{equation}\label{sphere-subspace}
        Y = \{(x_1,x_2),~|x_1|^2 + |x_2|^2 = 1\}.
    \end{equation}
    It is clear that $Y$ is $G$-invariant and that $Y\eqsim S^{3}$. Moreover, we claim that $Y/G\eqsim S^2$. Indeed, consider the $G$-equivariant map $Y\rightarrow S^2$ given by
    \begin{equation}\label{hopf-fibration-map}
        H_f:(x_1,x_2)\rightarrow \begin{pmatrix}
            \sqrt{1-z^2}\cos(\theta) \\
            \sqrt{1-z^2}\sin(\theta) \\
            z \\
        \end{pmatrix}\in S^2,
    \end{equation}
    where $\theta := \theta(x_1,x_2)\in [0,2\pi)$ is the (counterclockwise) angle from $x_1$ to $x_2$ (when $x_1 = 0$ or $x_2 = 0$ this angle is not well-defined and we simply set $\theta(x_1,x_2) = 0$), and the height $z = z(x_1,x_2)$ is given by
    \begin{equation}\label{equation-for-z}
        z(x_1,x_2) = \begin{cases}
            \frac{2}{\pi}\arctan(\log(|x_2|) - \log(|x_1|)) & 0 < |x_1|,|x_2| < 1\\
            1 & |x_1| = 0\\
            -1 & |x_2| = 0.
        \end{cases}
    \end{equation}
    It is straightforward to verify that this map is both $G$-invariant and continuous. Indeed, both the angle $\theta$ and the lengths $|x_1|$ and $|x_2|$ are $G$-invariant. Since the map $H_f$ only depends upon these functions of the input it is clearly $G$-invariant. Moreover, continuity is evident away from the points where $x_1 = 0$ or $x_2 = 0$, since in this regime the map $H_f$ is a composition and product of continuous functions. Note that although $\theta$ is not a continuous function of $x_1,x_2\neq 0$, both $\sin(\theta)$ and $\cos(\theta)$ are continuous. Indeed, if we view $x_1$ and $x_2$ as elements of the complex plane, then
    $$
        \cos(\theta) + i\sin(\theta) = \frac{x_2|x_1|}{x_1|x_2|},
    $$
    which verifies continuity away from the set where $x_1 = 0$ or $x_2 = 0$.
    
    Next, we verify continuity when $x_1 = 0$ (the corresponding calculation when $x_2 = 0$ is completely analogous). We observe that for any unit vector $x_2$, we have 
    $$
        H_f(0,x_2) = \begin{pmatrix}
            0 \\
            0 \\
            1 \\
        \end{pmatrix}.
    $$
    Given any sequence $(x_1^n,x_2^n)\rightarrow (0,x_2)$ we note that since $x_1^n\rightarrow 0$, we have that $z(x_1^n,x_2^n) \rightarrow (2/\pi)\arctan(\infty) = 1$. Plugging this into \eqref{hopf-fibration-map}, we get that
    $$
        H_f(x_1^n,x_2^n)\rightarrow \begin{pmatrix}
            0 \\
            0 \\
            1 \\
        \end{pmatrix},
    $$
    verifying continuity at $(0,x_2)$. Thus, $H_f$ does indeed define a continuous map $Y/G\rightarrow S^2$. 
    
    Finally, we must verify the surjectivity and injectivity of $H_f$. These follow since given any $z\in [-1,1]$ we can solve for $z$ in \eqref{equation-for-z} uniquely for the two lengths $|x_1|$ and $|x_2|$ satisfying $|x_1|^2 + |x_2|^2 = 1$ (since $\arctan$ and $\log$ are both increasing functions). If $z\neq \pm 1$, these lengths will both be non-zero and the angle $\theta$ is well-defined and uniquely determined up to a rotation of both $x_1$ and $x_2$. At the poles where $z = \pm 1$ (and $\theta$ becomes irrelevant), one of the vectors $x_i = 0$ and the other is a unit vector, and all of these configurations are equivalent up to a rotation as well. In fact, the map $H_f$ is easily seen to be the well-known Hopf fibration \cite{hopf1931abbildungen}. 
    
    We can now complete the proof that no continuous canonicalization can exist when $d = 2$. Indeed, if there were such a canonicalization $y$, then the induced map $\ty$ and the quotient map $q$ would satisfy
    \begin{equation}
        S^2\xrightarrow{\ty} S^3\xrightarrow{q} S^2
    \end{equation}
    whose composition is the identity. Algebraic topology provides a variety of obstructions to such a scenario. Perhaps the simplest comes from the homology groups (see \cite{hatcher2002}, Chapter 2). Indeed, the homology groups are $H_2(S^2) = \mathbb{Z}$ and $H_2(S^3) = \{0\}$. The induced maps on homology groups would have to satisfy
    \begin{equation}
        \mathbb{Z}\xrightarrow{\ty^*} \{0\}\xrightarrow{q^*} \mathbb{Z}
    \end{equation}
    with composition equal to the identity which is clearly impossible.

    Next, we consider the case where $d > 2$. We will prove by induction on $d$ that no canonicalization exists for $(\RR^{d\times d},SO(d))$. The case $d=2$ forms the base case.
    
    Suppose that there exists a continuous canonicalization $y$ for $(\RR^{d\times d},SO(d))$. Consider the element
    $$
        X_0 = \begin{pmatrix}
            0 & 0 & 0 & \cdots & 0\\
            0 & 0 & 0 & \cdots & 0\\
            0 & 0 & 0 & \cdots & 0\\
            \vdots & \vdots & \vdots & \ddots & \vdots\\
            0 & 0 & 0 & \cdots & 1
        \end{pmatrix}\in \RR^{d\times d}
    $$
    whose first $d-1$ columns are $0$ and whose remaining column consists of the basis vector $e_d$.
    Note that we can compose $y$ with any element of $SO(d)$ to obtain a new continuous canonicalization. Thus, we can assume without loss of generality that $y_{X_0} = X_0$ by composing with a rotation which moves the last (and only non-zero) column of $y_{X_0}$ to the standard basis vector $e_d$. 
    
    We will use $y$ to construct a continuous canonicalization $\hat{y}$ for $(\RR^{(d-1)\times (d-1)},SO(d-1))$, which cannot exist by the inductive hypothesis. Consider the norm on $\RR^{d\times d}$ defined by (note that this is not the usual $\ell^\infty$ norm of a vector)
    \begin{equation}\label{infinity-norm-definition}
        \|X\|_\infty := \max_i \|x_i\|_2,
    \end{equation}
    i.e. the maximum length of the columns of $X$.
    Let $\epsilon > 0$ be chosen so that $\|y_X - y_{X_0}\|_{\infty} = \|y_X - X_0\|_{\infty} < 1/2$ whenever $\|X - X_0\|_{\infty} \leq \epsilon$ (this can always be done by the continuity of the canonicalization $y$). Let
    \begin{equation}
        B_\infty^{d-1} = \{X\in \RR^{(d-1)\times (d-1)}:~\|X\|_{\infty} = 1\} 
    \end{equation}
    denote the unit ball in $\RR^{(d-1)\times (d-1)}$ with respect to the norm \eqref{infinity-norm-definition}. We will first construct the canonicalization $\hat{y}$ on the set $B_\infty^{d-1}$ and then extend it homogeneously to all of $\RR^{(d-1)\times (d-1)}$.
    
    We define a (continuous) map $i_0:B_\infty^{d-1}\rightarrow \RR^{d\times d}$ by 
    \begin{equation}
        i_0(X) = \begin{pmatrix}
            \epsilon X & 0\\
            0 & 1
        \end{pmatrix}\in \RR^{d\times d}.
    \end{equation}
    In otherwords, $i_0$ simply puts $\epsilon$ times $X$ into the upper $(d-1)\times (d-1)$ block of $X_0$. We clearly have $$\|i_0(X) - X_0\|_{\infty} \leq \epsilon$$ for every $X\in B_\infty^{d-1}$. 
    
    Applying the canonicalization $y$ to $i_0(X)$ gives a matrix
    \begin{equation}
        y_{i_0(X)} = U(X)\begin{pmatrix}
            \epsilon X & 0\\
            0 & 1
        \end{pmatrix} = \begin{pmatrix}
            A & v\\
            w^T & r        \end{pmatrix},
    \end{equation}
    for a matrix $U(X)\in SO(d)$ (since $y$ is a canonicalization and thus maps to the same orbit). Here the submatrices $A\in \RR^{(d-1)\times (d-1)}$, $w\in \RR^{d-1}$, $v\in \RR^{d-1}$, and $r\in \RR$ are all continuous functions of $X$ since $y$ and $i_0$ are continuous (we have suppressed this dependence for notational simplicity). Finally, we observe that since the last columns of $i_0(X)$ is $e_d$, which has norm $1$, we have $|v|^2 + r^2 = 1$.
    
    Since $\|i_0(X) - X_0\|_{\infty} \leq \epsilon$ we have that $\|y_{i_0(X)} - X_0\|_{\infty} < 1/2$, which means that $r \geq 1/2$. We define the matrix
    \begin{equation}\label{V-d-formula-1804}
        V(X) := I_d + (r-1)\left[e_de_d^T + \frac{\bar{v}\bar{v}^T}{1-r^2} \right] + \left[e_d\bar{v}^T - \bar{v}e_d^T\right],
    \end{equation}
    where $I_d$ is the $d\times d$ identity matrix and $\bar{v}$ is $v$ augmented with a $0$ in the $d$-th coordinate. This is clearly a continuous function of both $r > -1$ and $v$ since $(r-1)/(1-r^2) = -1/(r+1)$, and thus it is also a continuous function of $X$. Moreover, we claim that since $r^2 + |v|^2 = 1$, the matrix $V(X)\in SO(d)$ and
    \begin{equation}
        V(X)\begin{pmatrix}
            v\\
            r
        \end{pmatrix} = e_d.
    \end{equation}
    Indeed, in the case $v=0$, $r=1$ we clearly have $V(X) = I_d$. If $v\neq 0$, then the space orthogonal to the vectors $e_d$ and $\bar{v}$, which we denote by $W^\perp$, is clearly invariant under $V(X)$, while we calculate
    \begin{equation}
        V(X)e_d = e_d + (r-1)e_d - \bar{v} = re_d - \bar{v}
    \end{equation}
    and
    \begin{equation}
        V(X)\bar{v} = \bar{v} + (r-1)\frac{|v^2|}{1 - r^2}\bar{v} + |v|^2e_d = r\bar{v} + |v|^2e_d,
    \end{equation}
    since $|v|^2 = 1-r^2$. Thus, with respect to the orthonormal basis consisting of $e_d, \bar{v}/|v|$ and an orthonormal basis for $W^\perp$, the matrix $V(X)$ consists of a $(d-2)\times (d-2)$ identity matrix block (corresponding to the orthonormal basis of $W^\perp$), and a $2\times 2$ block of the form
    \begin{equation}
        \begin{pmatrix}
            r & -|v|\\
            |v| & r
        \end{pmatrix}
    \end{equation}
    corresponding to $e_d$ and $\bar{v}/|v|$. This implies that $V(X)\in SO(d)$. We also calculate
    \begin{equation}
        V(X)\begin{pmatrix}
            v\\
            r
        \end{pmatrix} = rV(X)e_d + V(x)\bar{v} = r^2e_d - r\bar{v} + r\bar{v} + |v|^2e_d = (r^2 + |v|^2)e_d = e_d
    \end{equation}
    as desired. Geometrically, $V(X)$ is the rotation through the plane spanned by $e_d$ and $\bar{v}$ which maps the last column of $y_{i_0(X)}$ to $e_d$. Algebraically, this is given by the formula \eqref{V-d-formula-1804}.

    We now define the canonicalization $\hat{y}:B_\infty^{d-1}\rightarrow \RR^{(d-1)\times (d-1)}$ via
    \begin{equation}
        V(X)y_{i_0(X)} = \begin{pmatrix}
            \epsilon\hat{y}_X & 0\\
            0 & 1
        \end{pmatrix}.
    \end{equation}
    Since both $V(r,v)$ and $y_{i_0(X)}$ are continuous functions of $X$, it is clear that $\hat{y}$ is continuous. 
    
    We claim that $\hat{y}$ is a canonicalization for $(\RR^{(d-1)\times (d-1)},SO(d-1))$. Note that since $y$ is $SO(d)$ invariant, we have that $\hat{y}$ is $SO(d-1)$ invariant (recall that $V(X)$ only depends upon $y_{i_0(X)}$ and is thus also invariant). Further, since
    \begin{equation}
        \begin{pmatrix}
            \epsilon\hat{y}_X & 0\\
            0 & 1
        \end{pmatrix} = V(X)U(X)\begin{pmatrix}
            \epsilon X & 0\\
            0 & 1
        \end{pmatrix},
    \end{equation}
    where by construction the last column of $V(X)U(X)\in SO(d)$ is $e_d$, we have that $\hat{y}_X = \hat{O}(X)X$ for $\hat{O}(X)\in SO(d-1)$ being the upper right $(d-1)\times (d-1)$-block of $V(X)U(X)$. This implies that $\hat{y}$ is a canonicalization as desired.
    
    We finally extend $\hat{y}$ $1$-homogeneously to all of $\RR^{(d-1)\times (d-1)}$ via
    $$
        \hat{y}_X = \begin{cases}
            \|X\|_\infty \tilde{y}_{X/\|X\|_{\infty}} & X\neq 0\\
            0 & X = 0.
        \end{cases}
    $$
    This gives a continuous canonicalization for $(\RR^{(d-1)\times (d-1)},SO(d-1))$, which is impossible by the inductive hypothesis.
\end{proof}

\begin{prop}\label{prop:Ono}
   If $n > d \geq 1$, then $O(d)$ acting on $\RR^{d\times n}$ %$(\RR^{d\times n},O(d))$ 
   does not have a continuous canonicalization.
\end{prop}
\begin{proof}
    Observe that be considering the subspace
    \begin{equation}
        Y = \{(x_1,...,x_{d+1},0,...,0),~x_i\in \mathbb{R}^d\}\subset \mathbb{R}^{d\times n}
    \end{equation}
    a continuous canonicalization for $\mathbb{R}^{d\times n}$ gives a continuous canonicalization for $\mathbb{R}^{d\times (d+1)}$, so it suffices to consider the case $n = d+1$.

    We first address the case $d=1$, in which case $O(d) = \{\pm 1\}$.
    Consider the subspace 
    $$Y = S^1 = \{(x_1,x_2):~|x_1|^2 + |x_2|^2 = 1\}\subset \RR^{1\times 2}.$$
    The quotient $Y/G \eqsim S^1$ (identifying antipodal points on a circle again gives a circle) and the quotient map $q:S^1\rightarrow S^1$ is given by identifying antipodal points which corresponds to the map $t\rightarrow 2t$ (here we are viewing $S^1$ as $\mathbb{R}/\mathbb{Z}$). Suppose that $y$ is a continuous canonicalization and $\ty:Y/G\rightarrow Y$ is the corresponding map on the quotient space. We consider the induced maps on the fundamental group $\pi_1(S^1) = \mathbb{Z}$ corresponding to the quotient map $q$ and map $\ty$. These would satisfy $q_*(z) = 2z$ and $q^*\circ \ty^* = I$. This is clearly impossible since $q_*$ is not surjective.

    We now prove the result for general $d$ by induction on $d$ with $d=1$ being the base case. In particular, a continuous canonicalization for $(\RR^{d\times (d+1)},O(d))$ can be used to obtain a continuous canonicalization for $(\RR^{(d-1)\times d},O(d-1))$ utilizing exactly the same construction as in the proof of Proposition \ref{prop:SOno}.
\end{proof}

% \bootstrapping*
% \nd{I have an argument which I think is correct but need to work out the details. This is the argument for the $O(d)$ case and $n=3,d=2$:} Assume by contradiction the canonicalization $c$ exists. We will show this induces a continuous canonicalization for $\RR^{d\times n}, O(d)$ for $n=2,d=1$, which we know does not exist. This Canonicalization is realized as a sequence of mappings 
% $$(x,y)\in \RR^2\overset{\phi}{\mapsto}
% \begin{pmatrix}
% 0 & 0 & 1\\
% x & y & 0\\
% \end{pmatrix}
% \overset{c}{\mapsto}
% [z_1,z_2,z_3]
% \mapsto
% M^T(z_3)[z_1,z_2,z_3]=
% \begin{pmatrix}
% 0 & 0 & 1\\
% \hat x & \hat y & 0
% \end{pmatrix}
% \overset{\pi}\mapsto
% (\hat x,\hat y)\in \RR^2
% $$
% Where the definition of  $\phi$ and the projection $\pi$ is obvious from the equation above, and for $z=[a,b]$ we define $M_z$ to be an orthogonal matrix which takes $e_1$ to $z$, that is
% $$M(z)=\begin{pmatrix}
% a & -b\\
% b & a
% \end{pmatrix}$$
% We see that the mapping  defined above is continuous and that for all $x,y$
% $$ M^T_{z_3}\circ c\begin{pmatrix}
% 0 & 0 & 1\\
% x & y & 0\\
% \end{pmatrix}
% =
% \begin{pmatrix}
% 1 & 0
% 0 & s
% \end{pmatrix}
% \begin{pmatrix}
% 0 & 0 & 1\\
% x & y & 0\\
% \end{pmatrix}
% $$
% for $s\in \{-1,1\}$. 
\onlyReynold*
\begin{proof}

We begin by showing that the identity element $e\in G$ is in $\F(v_0)$ for some $v_0\in V_{free}$. Indeed, choose some $v\in V_{free}$, and choose some $g\in \F(v)$. Then we can set $v_0=g^{-1}v $. This point  is in $V_{free}$ and the frame equivariance implies that 
$$e=g^{-1}g\in g^{-1}\F(v)=\F(g^{-1}v) $$

Next, we will show that $e\in \F(v)$ \emph{for all} $v\in V_{free}$. Since $V_{free}$ is connected, it is sufficient to show that the sets 
$$V_0=\{v\in V_{free}| \quad e \in \F(v) \} \text{ and } V_1= \{v\in V_{free}| \quad e \not \in \F(v) \}$$
are both open in $V_{free}$. This implies that one of these sets must be empty and the other all of $V_{free}$. Since we say $V_0$ is not empty we will get that it is all of $V_{free}$. 

We now show that $V_0$ and $V_1$ are open in $V_{free}$. Note that the continuity of the group action and the finiteness of $G$ implies that $V_{free}$ is open, so that we need to prove that $V_0$ and $V_1$ are open in the original topology of $V$. 

Fix some $v\in V_{free}$. Since $v$ has a trivial stabilizer we have that $gv\neq v$  for all distinct $g\in G$. By continuity of the action and the finiteness of the group there exists an open neighborhood $U$ of $v$ so that $g^{-1}U\cap U=\emptyset $ for all distinct $g\in G$. Thus, for any fixed $v\in V_{free}$ we can define a continuous function $f_{v} $ which is identically $1$ on $U$ and zero on $g^{-1}U$ for all $g\neq e$. 

It follows that for all $y\in U$ we have that $f_{v}(g^{-1}y)$ is one if $g=e$ and is zero otherwise. In particular
$$\Iframe [f_{v}] (y)=\frac{1}{|\F(y)|}\sum_{g\in \F(y)} f_v(g^{-1}y)=\twopartdef{1/|\F(y)|}{e\in \F(y)}{0}{e \not \in \F(y)} , \quad \forall y \in U. $$
It follows that if $v\in V_0$, so that $e\in \F(v)$, then by continuity of $\Iframe[f_v] $  we will have that $e \in \F(y)$ (and also $|\F(y)|=|\F(v)| $) for all $y\in U$. This shows that $V_0$ is open. Similarly, if $v\in V_1$, so that $e \not \in \F(v)$, then by continuity of $\Iframe[f_v] $  we will have that $e \not \in \F(y)$ for all $y\in U$ so that $V_1$ is open. We have thus showed that $V_0=V_{free}$.

We have shown that the identity element is in $\F(v)$ for all $v\in V_{free}$. The same is true for all group elements: let $g\in G$ and $v\in V_{free}$. We want to show that $g\in \F(v)$. Note that $g^{-1}v$ also is in $V_{free}$. Therefore $e\in \F(g^{-1}v)$, and by equivariance 
$$g \in g\F(g^{-1}v)=\F(gg^{-1}v)=\F(v). $$

We now proved the claim for all points in $V_{free}$ and we need to extend the claim to the closure of $V_{free}$. 

Let $v$ be a point in the closure of $V_{free}$, and let $g\in G$. We need to show again that $g\in \F(v)$. Note that $h^{-1}v=g^{-1}v$ if and only if $g=sh$ for some  $s$ in the stabilizer $\stab{v}$. In this case we will say that $g,h$ are $\stab{v}$ equivalent and denote $g \sim_v h$.  

Similarly to earlier in the proof, we can find an open set $U\subseteq V$, and  a continuous function $f_{v,g}:V \to \RR$, such that for all $y\in U$ and $h\in G$, we have that $f_{v,g}(hy)=1 $ if $h \sim_v g$, and otherwise $f_{v,g}(hy)=0 $.

It follows that the continuous function $\Iframe [f_{v,g}] (y)$  is equal to 

$$\Iframe [f_{v,g}] (y)=\frac{1}{|\F(y)|}\sum_{g\in \F(y)} f_v(g^{-1}y)=\frac{|\stab{v}\cdot g\cap \F(y)| }{|\F(y)|} , \quad \forall y \in U. $$

Since $v$ is in the closure of $V_{free}$, there is a sequence $v_n$ of elements in $V_{free}\cap U$  converging to $v$, and since $\F(v_n)=G$ we obtain 

$$\frac{|\stab{v}|}{|G|}=\lim_{n \rightarrow \infty} \Iframe [f_{v,g}] (v_n) =\Iframe [f_{v,g}] (v)=\frac{|\stab{v}\cdot g\cap \F(v)| }{|\F(v)|} $$
so that the intersection $\stab{v}\cdot g\cap \F(v)$ is not empty: there exists some $s\in \stab{v}$ such that $sg\in \F(v)$. By equivariance of the frame 
$$g=s^{-1}sg\in s^{-1}\F(v)=\F(s^{-1}v)=\F(v). $$
This concludes the proof of the theorem.
\end{proof}

\snReynold*
\begin{proof}
The matrices $X\in \RR^{d\times n}$ which have trivial stabilizers are the set $\Rdistinct$ of matrices $X$ whose columns are pairwise distinct. Clearly $\Rdistinct$ is dense in $\RR^{d \times n}$, and therefore, 
due to Theorem \ref{thm:onlyReynold}, it is sufficient  to show that $\Rdistinct$ is connected. 

We will show that $\Rdistinct$ is path connected and hence connected. Let $X^{(0)}$ be any point with a trivial stabilizer. We will show that it can be connected by two straight lines to the point $X^{(2)}$ defined by $X^{(2)}_{ij}=j $, that is the matrix (clearly also with trivial stabilizer) whose rows are all identical and given by
\begin{equation}
\begin{pmatrix}\label{eq:one2n}
1 & 2 & \ldots & n    
\end{pmatrix} 
\end{equation}
As a first step, we find a permutation $\tau$ which sorts the first row of $X^{(0)}$, so that 
$$X^{(0)}_{1\tau(1)}\leq X^{(0)}_{1\tau(2)}\leq \ldots \leq X^{(0)}_{1\tau(n)}  .$$
We then choose a point $X^{(1)}$ whose first row satisfies this same inequality strictly 
$$X^{(1)}_{1\tau(1)}< X^{(1)}_{1\tau(2)}< \ldots < X^{(1)}_{1\tau(n)}  .$$
and whose remaining rows are equal to the rows \eqref{eq:one2n} of $X^{(2)}$. We note that all points in the straight line between $X^{(0)}$ and $X^{(1)}$ have a trivial stabilizer. For $X^{(0)}$ this is by assumption, and for all other points this is because the first row, sorted by $\tau$, is strictly separated. 

Next, we observe that all points in the straight line between $X^{(1)}$ and $X^{(2)}$ also have a trivial stabilizer. This is because all but the first row are equal to \eqref{eq:one2n}. Thus we have shown that $\Rdistinct$ is connected.

\end{proof}

% \equidefined*
% \begin{proof}
% Special case of later theorem \nd{add ref}. 
% \end{proof}

\finiteframeSOthm*
\begin{proof}
    By considering the subspace
    \begin{equation}
        Y = \{(x_1,x_2,0,...,0):~x_1,x_2\in \mathbb{R}^2\}
    \end{equation}
    it suffices to consider the case $n=2$.
    
    We will in fact prove something a bit stronger. Let $G$ be a group acting on a vector space $V$ and define an unweighted frame of size $N$ to be a map
    \begin{equation}
        \F_N:V\rightarrow G^N/S_N
    \end{equation}
    from $V$ to the set $G^N/S_N$ of unordered $N$-tuples of elements of $G$ (potentially with repetition). The corresponding invariant projection operator is given by
    \begin{equation}
        \Ican[f](v) := \frac{1}{N}\sum_{g\in \F_N(v)} f(g^{-1}v).
    \end{equation}
    Given a frame $\F$ such that $|\F(v)| \leq M$ for all $v\in V$, we can construct a finite unweighted frame of size $M!$ by repeating each element of $\F(v)$ $k_v$ times where $k_v = M!/|\F(v)|$. However, the notion of an unweighted frame of size $N$ is more general since it allows different weights through repetition, although the weights must all be divisible by $1/N$. 
    
    We proceed to show that a continuity-preserving finite unweighted frame $\F_N$ cannot exist for $SO(2) = S^1$ acting on $\RR^{2\times 2}$ for any finite $N$ (here $S^1$ denotes the unit circle of rotations). Suppose to the contrary that $\F_N$ is such a frame. Consider the subspace $Y$ defined in \eqref{sphere-subspace} in the proof of Theorem \ref{prop:SOno_and_Ono}, and the subspace
    \begin{equation}\label{open-sphere-subspace}
        Y^o := \{(x_1,x_2),~|x_1|^2 + |x_2|^2 = 1,~x_1\neq 0,~x_2\neq 0\}
    \end{equation}
    which is the same as $Y$ but with the points where $x_1 = 0$ and $x_2 = 0$ removed. Since all points of $Y$ have trivial stabilizer, the map $\F_N$ restricted to $Y$ must be continuous and equivariant under the action of $S^1$.

    Next, consider the space $S^1_{N,u} := (S^1)^N/S_N$ of unordered $N$-tuples of rotations. Observe that the group $S^1$ of rotations naturally acts componentwise on $S^1_{N,u}$ via the map
    \begin{equation}\label{group-unordered-tuples-action}
        g\cdot (g_1,...,g_N) = (gg_1,...,gg_N).
    \end{equation}
    In addition, since $S^1$ is abelian, there is a natural continuous map $S^1_{N,u}\rightarrow S^1$ given by multiplying all of the rotations together, i.e.
    \begin{equation}\label{multiplication-map-rotations}
        (g_1,...,g_N)\rightarrow g_1g_2\cdots g_N\in S^1.
    \end{equation}

    Consider the following map defined on $Y^o$
    \begin{equation}
        \arg_1(X) = \frac{x_1}{|x_1|}\in S^1,
    \end{equation}
    which gives the angle of the first vector (well-defined and continuous since we have removed the points where $x_1 = 0$). We now define the continuous map $\tilde{\F}_N:Y^o\rightarrow S^1_{N,u}$ by
    \begin{equation}\label{tilde-F-definition-S-1-1981}
        \tilde{\F}_N(X) = \arg_1(X)^{-1}\cdot\F_N(X)\in S^1_{N,u},
    \end{equation}
    which gives the unordered collection of angles that the first vector is rotated to under the frame $\F_N$ (here the multiplication in equation \eqref{tilde-F-definition-S-1-1981} is the action described in \eqref{group-unordered-tuples-action}). The map $\tilde{\F}_N$ is invariant under the action of $S^1$ and continuous on $Y^o$ (but does not have a continuous extension to $Y$). Thus, it induces a map $\tilde{\F}_N:Y^o/S^1\rightarrow S^1_{N,u}$.

    Observe that $Y^o/S^1\cong (0,1)\times S^1$ via the parameterization map
    \begin{equation}
        (t,\theta) \rightarrow \left(\begin{pmatrix}
            \sqrt{1 - t^2}\\
            0
        \end{pmatrix},\begin{pmatrix}
            t\cos(\theta)\\
            t\sin(\theta)
        \end{pmatrix}\right),
    \end{equation}
    where $t$ denotes the length of $x_2$ and $\theta$ denotes the counterclockwise angle from $x_1$ to $x_2$. Using this parameterization, we define a map $\mathcal{G}_N:(0,1)\times S^1\rightarrow S^1_{N,u}$ via
    \begin{equation}\label{rewriting-F-N}
        \mathcal{G}_N(t,\theta) = \tilde{F}_N\left(\begin{pmatrix}
            \sqrt{1 - t^2}\\
            0
        \end{pmatrix},\begin{pmatrix}
            t\cos(\theta)\\
            t\sin(\theta)
        \end{pmatrix}\right).
    \end{equation}

    Next, we claim that the map $\mathcal{G}_N$ extends to a continuous map $[0,1]\times S^1\rightarrow S^1_{N,u}$. Indeed, we define
    \begin{equation}\label{definition-of-G-N}
        \mathcal{G}_N(0,\theta) = \F_N\left(\begin{pmatrix}
            1\\
            0
        \end{pmatrix},\begin{pmatrix}
            0\\
            0
        \end{pmatrix}\right),~\mathcal{G}_N(1,\theta) = \theta\cdot\F_N\left(\begin{pmatrix}
            0\\
            0
        \end{pmatrix},\begin{pmatrix}
            1\\
            0
        \end{pmatrix}\right),
    \end{equation}
    and check that the resulting map is continuous on $[0,1]\times S^1$. Clearly, continuity must only be checked at points of the form $(0,\theta)$ and $(1,\theta)$. 
    
    Suppose first that $(t_n,\theta_n)\rightarrow (0,\theta)$. Continuity in $\theta$ is clear when $t = 0$, so we may assume that $0 < t_n < 1$. We then see that
    \begin{equation}
        \mathcal{G}_N(t_n,\theta_n) = \tilde{\F}_N\left(\begin{pmatrix}
            \sqrt{1 - t_n^2}\\
            0
        \end{pmatrix},\begin{pmatrix}
            t_n\cos(\theta_n)\\
            t_n\sin(\theta_n)
        \end{pmatrix}\right) = \F_N\left(\begin{pmatrix}
            \sqrt{1 - t_n^2}\\
            0
        \end{pmatrix},\begin{pmatrix}
            t_n\cos(\theta_n)\\
            t_n\sin(\theta_n)
        \end{pmatrix}\right) \rightarrow \F_N\left(\begin{pmatrix}
            1\\
            0
        \end{pmatrix},\begin{pmatrix}
            0\\
            0
        \end{pmatrix}\right).
    \end{equation}
    Here we have used the continuity of $\F_N$ on all of $Y$ to get the final convergence, and that
    $$
        \arg_1\begin{pmatrix}
            \sqrt{1 - t_n^2}\\
            0
        \end{pmatrix} = e_1
    $$
    corresponds to $\theta = 0$ to get the middle equality.

    Next, suppose that $(t_n,\theta_n)\rightarrow (1,\theta)$. We may again assume that $0 < t_n < 1$ since $\mathcal{G}_N$ is continuous in $\theta$ when $t = 1$. We similarly calculate
    \begin{equation}
        \mathcal{G}_N(t_n,\theta_n)\F_N\left(\begin{pmatrix}
            \sqrt{1 - t_n^2}\\
            0
        \end{pmatrix},\begin{pmatrix}
            t_n\cos(\theta_n)\\
            t_n\sin(\theta_n)
        \end{pmatrix}\right) \rightarrow \F_N\left(\begin{pmatrix}
            0\\
            0
        \end{pmatrix},\begin{pmatrix}
            \cos(\theta)\\
            \sin(\theta)
        \end{pmatrix}\right).
    \end{equation}
    Now the equivariance of $\F_N$ implies that
    \begin{equation}
        \F_N\left(\begin{pmatrix}
            0\\
            0
        \end{pmatrix},\begin{pmatrix}
            \cos(\theta)\\
            \sin(\theta)
        \end{pmatrix}\right) = \theta\cdot \F_N\left(\begin{pmatrix}
            0\\
            0
        \end{pmatrix},\begin{pmatrix}
            1\\
            0
        \end{pmatrix}\right) = \mathcal{G}_N(1,\theta).
    \end{equation}
    This verifies the continuity of $\mathcal{G}_N$ on all of $[0,1]\times S^1$. Thus $\mathcal{G}_N$ defines a homotopy between the constant loop $\mathcal{G}(0,\cdot)$ and the loop $\mathcal{G}(1,\cdot)$ defined in \eqref{definition-of-G-N} in the space $S^1_{N,u}$. 
    
    We complete the proof by showing that these loops cannot be homotopic. For this, we use the multiplication map defined in \eqref{multiplication-map-rotations}. Composing $\mathcal{G}_N$ with this map gives a continuous map $\mathcal{H}_N:[0,1]\times S^1\rightarrow S^1$. From \eqref{definition-of-G-N} we see that $\mathcal{H}_N(0,\cdot)$ is a constant map, while $\mathcal{H}_N(1,\cdot):S^1\rightarrow S^1$ loops around the circle $N$ times. Thus $\mathcal{H}_N$ gives a homotopy between these two loops. This is impossible, since the fundamental group $\pi_1(S^1)\cong \mathbb{Z}$ and $\mathcal{H}_N(0,\cdot)$ represents the zero element while $\mathcal{H}_N(1,\cdot)$ represents the element $N$ (see \cite{hatcher2002}, Chapter 1).
\end{proof}

\invariantprop*
\begin{proof}
Let $f:V\to \RR$ be a function. We need to show that $\Iw[f]$ is an invariant function. We first show that in general $\mu_v$ can be replace with $\bar{\mu}_v$ in the defininition of $\Iw[f]$. That is 
\begin{align}
\Iw[f](v)&=\int_G f(g^{-1}v)d\mu_v(g)  \nonumber\\
&=\int_{\stab{v}} \int_G f((sg)^{-1}v)d\mu_v(g)ds  \nonumber\\
&=\int_{\stab{v}} \int_G f(g^{-1}v)ds_*\mu_v(g)ds  \nonumber\\
&=\int_G f(g^{-1}v)d\bar{\mu}_v(g) \label{eq:equiv}
\end{align}
Using this we can now show that $\Iw[f]$ is an invariant function because for every $h\in G$ and $v\in V$,
\begin{align*}
\Iw[f](hv)&=\int_G f(g^{-1}hv)d\mu_{hv}(g)  \\
&=\int_G f(g^{-1}hv)d\bar{\mu}_{hv}(g)\\
&=\int_G f(g^{-1}hv)dh_*\bar{\mu}_{v}(g)\\
&=\int_G f((hg)^{-1}hv)d\bar{\mu}_{v}(g)\\
&=\int_G f(g^{-1}v)d\bar{\mu}_{v}(g)\\
&=\int_G f(g^{-1}v)d\mu_{v}(g)\\
&=\Iw[f](v).
\end{align*}
Boundedness also follows easily since $d\mu_v$ is a probability measure, so that
\begin{equation}
    |\Iw[f](v)| \leq \left|\int_G f(g^{-1}v)d\mu_v(g)\right| \leq \sup_{w}|f(w)|.
\end{equation}
\end{proof}

\cwwbec*
\begin{proof}
Let $f:V\to \RR$ be a continuous functions. Let $v_k$ be a sequence converging to $v\in V$. For the `if' direction, we need to show that $\Iw[f](v_k)$ converges to $\Iw[f](v)$. We observe that
\begin{align*}
&|\Iw[f](v_k)-\Iw[f](v)|\\
&=\left|\int f(g^{-1}v_k)d\mu_{v_k}(g)-\int f(g^{-1}v)d\mu_{v}(g)\right|  \\
&\leq\left|\int f(g^{-1}v_k)d\mu_{v_k}(g)-\int f(g^{-1}v_k)d\langle\mu_{v_k}\rangle_v(g)\right| + \left|\int f(g^{-1}v_k)d\langle\mu_{v_k}\rangle_v(g)-\int f(g^{-1}v)d\bar{\mu}_{v}(g)\right|.
\end{align*}
Now, by assumption $\langle\mu_{v_k}\rangle_v \rightarrow \bar{\mu}_{v}$ weakly so that the second term tends to $0$ (since $f$ and the action of $G$ are continuous). Moreover, since $v$ is by definition invariant under the action of $\stab{v}$, we have
\begin{equation}
    \int f(g^{-1}v)d\mu_{v_k}(g)-\int f(g^{-1}v)d\langle\mu_{v_k}\rangle_v(g) = 0.
\end{equation}
Hence, we get
\begin{align*}
    &\left|\int f(g^{-1}v_k)d\mu_{v_k}(g)-\int f(g^{-1}v_k)d\langle\mu_{v_k}\rangle_v(g)\right|\\
    &\leq \int |f(g^{-1}v_k) - f(g^{-1}v)|d\mu_{v_k}(g) + \int |f(g^{-1}v_k) - f(g^{-1}v)|d\langle\mu_{v_k}\rangle_v(g)
\end{align*}
Since $v_k\rightarrow v$, the group $G$ is compact, $f$ is continuous, and the integrals are both against probability measures, we finally get
\begin{equation}
    \left|\int f(g^{-1}v_k)d\mu_{v_k}(g)-\int f(g^{-1}v_k)d\langle\mu_{v_k}\rangle_v(g)\right|\rightarrow 0,
\end{equation}
as desired.

For the converse, assume that $\mu$ is not robust. This means that there exists a sequence $v_k$ converging to $v$ such that $\langle\mu_{v_k}\rangle_v$ does not converge to $\bar{\mu}_{v}$ weakly. Thus there is a continuous function $\phi$ defined on the group $G$, such that
\begin{equation}\label{non-convergence-of-averaged-measures}
    \int_{G}\phi(g)d\langle\mu_{v_k}\rangle_v(g) \nrightarrow \int_{G}\phi(g)\bar{\mu}_{v}(g).
\end{equation}
Further, we observe that by definition of the averaged measure $\langle\cdot\rangle_v$, for any measure $\mu$ we have
\begin{equation}
    \int_{G}\phi(g)d\langle\mu\rangle_v(g) = \int_{G_v}\int_G \phi(g)d(s^*\mu)(g)ds = \int_{G_v}\int_G \phi(sg)d\mu(g)ds = \int_{G} \bar{\phi}(g)d\mu(g),
\end{equation}
where the averaged function $\bar{\phi}$ is defined by
\begin{equation}
    \bar{\phi}(g) := \int_{G_v}\phi(sg)ds
\end{equation}
and is evidently $G_v$ invariant. Thus $\bar{\phi}$ defines a function on $G/G_v$ and \eqref{non-convergence-of-averaged-measures} gives
\begin{equation}\label{limits-not-equal-eq-1093}
    \lim_{k\rightarrow \infty} \int_{G/G_v}\bar{\phi}(g)d\langle\mu_{v_k}\rangle_v(g) \neq \int_{G/G_v}\bar{\phi}(g)\bar{\mu}_{v}(g),
\end{equation}
where $\langle\mu_{v_k}\rangle_v$ and $\bar{\mu}_{v}$ are viewed as measures on $G/G_v$.

Observe that since the action of $G$ is continuous, the orbit $Gv = \{gv:~g\in G\}$ is homeomorphic to $G/G_v$. Thus, we can view $\bar{\phi}$ as a function on the orbit $Gv$ and extend it to a continuous function $f$ on the whole space $V$ by the Tietze extension theorem \cite{urysohn1925machtigkeit} (since $G$ and thus $Gv$ is a compact set). For this function $f$ we have (using compactness of $Gv$ and that $v_k\rightarrow v$)
\begin{equation}
    \lim_{k\rightarrow \infty} \Iw[f](v_k) = \lim_{k\rightarrow \infty} \int_G f(gv_k)d\mu_{v_k} = \lim_{k\rightarrow \infty} \int_G f(gv)d\mu_{v_k} = \lim_{k\rightarrow \infty} \int_{G/G_v}\bar{\phi}(g)d\langle\mu_{v_k}\rangle_v(g).
\end{equation}
On the other hand, we have
\begin{equation}
    \Iw[f](v) = \int_G f(gv)d\mu_{v} = \int_{G/G_v}\bar{\phi}(g)\bar{\mu}_{v}(g).
\end{equation}
Then \eqref{limits-not-equal-eq-1093} implies that $\lim_{k\rightarrow \infty} \Iw[f](v_k) \neq \Iw[f](v)$ so that $\Iw[f]$ is not continous and so $\Iw$ is not a BEC operator.
\end{proof}

\mon*
\begin{proof}
\textbf{Proof of generic global monotonicity}
First note that $X=(x_1,\ldots,x_n)$ is $a$ separated if and only if  $\bar X=(0,x_2-x_1,\ldots,x_n-x_1) $ is $a$ separated, which in turn will be $a$ separated if and only if 	$\|\bar X\|^{-1} \bar X $ is $a$ separated. Thus is it sufficient to consider $X$ in the set 
$$\MM=\{X=(0,x_2,\ldots,x_n)\in \Rdistinct| \quad \|X\|=1   \} $$
which can be identified with the unit circle $S^{n(d-1)-1} $ which is an $n(d-1)-1$ dimensional semi-algebraic set. 

Note that $X$ is $a$ separated  unless $X,a$ is a zero of the polynomial
$$F(X;a)=\prod_{i\neq j} a^T(x_i-x_j) .$$
The proof is based on the finite witness theorem from \cite{amir2023neural}. Namely
\begin{thm}[Special case of Theorem A.2 in \cite{amir2023neural}]
Let $\mathbb{M}\subseteq \RR^p$ be a semi-algebraic set of  dimension $D$. Let $F:\RR^{p} \times \RR^q \to \RR$ be a polynomial. Define 
$$\mathcal{N}=\{ X\in \mathbb{M}| \quad F(X;a)=0, \forall a \in \RR^q \} $$ 
Then for generic $a_1,\ldots,a_{D+1}$,
$$\mathcal{N}=\{X\in \mathbb{M}| \quad F(X;a_i=0, \forall i=1,\ldots,D+1 \} $$
\end{thm}
Applying to this theorem to our setting with $D=n(d-1)-1$, since for every $X\in \MM$ there exists a direction $a$ for which $X$ is separated, we see that $\N$ is the empty set. According to the theorem we have for generic $a_1,\ldots,a_{D+1}\in \RR^d$ that the set
\begin{align*}
	\{X\in \mathbb{M}| \quad F(X;a_i=0, \forall i=1,\ldots,D+1 \}
\end{align*}
is equal to $\N$ and thus is empty. Equivalently,  every $X$ will be separated with respect to at least one of the $a_i$.

\textbf{Remark:} We add a more elementary proof that a finite number of $a_i$ is
sufficient, but this proof has a worse cardinality, quadratic in $n$: Assume that we 
are given $m$ vectors $a_i,i=1,\ldots,m$ in $\RR^d$ and assume that they are "full 
spark", meaning that any $d \times d$ matrix formed by choosing $d$ of these vectors 
is full rank. Note that Lebesgue almost every choice of $a_i,i=1,\ldots,m$ in $\RR^d$ will be full spark. Assume that the number of vectors $m$ is strictly larger than ${n \choose 2}(d-1) $. Now let $X\in \RR^{d\times n}$ and assume that $g(X;a_i) $ is 
not uniquely defined for any $i=1,\ldots,m$. This means the for each such $i$ there 
exists $s<t$ such that $a_i^T(x_s-x_t)=0 $. Since the number of $(s,t)$ pairs is 
${n \choose 2}$ and $m>(d-1){n \choose 2}$, by the pigeon hole principle there exists 
a pair $(s,t)$ for which $a_i^T(x_s-x_t)=0 $ for $\geq d$ different indices $i$. 

By the full spark assumption it follows that $x_s-x_t=0$ so $X$ has a non-trivial 
stabilizer.  

\textbf{Optimality}
We now show that, if $a_1,\ldots,a_k\in \RR^d$ and $k< n(d-1)$, then $a_1,\ldots,a_k$ are not globally separated.

    Note that by adding zeros to the sequence $a_1,...,a_k$ we may assume without loss of generality that $k = d(n-1) - 1$. 
    
    We partition the $a_i\in \mathbb{R}^d$ as $\{a_1,...,a_s\},\{a^1_1,....,a^1_d\},...,\{a^t_1,...,a^t_d\}$, where each block $\{a^q_1,...,a^q_d\}$ is linearly dependent (i.e. contained in a subspace of dimension $d-1$), and the vectors $a_1,...,a_s$ have the property that any subset of size $d$ is linearly independent (note this is vacuous if $s < d$). We can do this inductively by starting with all of the $a_i$ and removing linearly dependent subsets of size $d$ until there are none left. We note that $s + td = k$ and thus $s\equiv d-1\pmod{d}$ (so that in particular $s \geq d-1$).

    We now construct the $x_1,...,x_n\in \mathbb{R}^d$ as follows. First set $x_1 = 0$. Next, set $x_2 \neq 0$ satisfying the condition that $a_1\cdot x_2 = a_2\cdot x_2 = \cdots = a_{d-1}\cdot x_2$. Such an $x_2$ can be found since $a_1,...,a_{d-1}$ cannot span the whole space $\mathbb{R}^d$. 
    
    Now, for $p=1,...,v$, where $s = d-1 + dv$, we set $x_{p+2}$ to be the unique vector satisfying
    \begin{equation}\label{linear-system-974}
        a_{pd}\cdot x_{p+2} = 0,~\cdots~,a_{pd+d-2}\cdot x_{p+2} = 0,~a_{pd + d-1}\cdot (x_{p+2} - x_2) = 0.
    \end{equation}
    This vector exists and is unique since by construction $a_{pd}, ...,a_{pd+d-1}$ are linearly independent so that the linear system in \eqref{linear-system-974} is non-singular.

    At this stage, we verify that for each $l=1,...,s$ there exist $1\leq i\neq j\leq v+2$ such that $a_l\cdot x_i = a_l\cdot x_j$. This follows easily from the construction. Indeed, if $l = 1,...,d-1$ we can take $i=1,j=2$. If $l=pd + k$ with $0\leq k < d-1$ and $p \geq 1$ we can take $i=1$ and $j=p+2$. Finally, if $l=pd + d-1$ with $p \geq 1$ we can take $i=2$ and $j=p+2$.

    Next, we verify that all of the vectors $x_1,...,x_{v+2}$ are distinct. We first claim that $x_i \neq x_1 = 0$ if $i > 1$. This follows by construction for $i=2$, while if the solution to \eqref{linear-system-974} were $0$ for some $p$, then $a_{pd+d-1}\cdot x_2 = 0$. However, this would imply that $a_{pd+d-1}$ lies in the same plane as $a_1,...,a_{d-1}$, which contradicts the linear independence of any size $d$ subset of $a_1,...,a_s$. Next, we claim that $x_i\neq x_j$ if $i\neq j \geq 2$. Indeed, note that \eqref{linear-system-974} implies that $x_{p+2}$ is orthogonal to $a_{pd},...,a_{pd+d-2}$, while $x_2$ is orthogonal to $a_1,...,a_{d-1}$. If any of these vectors were equal, then the planes spanned by their corresponding sets of $d-1$ vectors would be the same, again contradicting the linear independence property of $a_1,...,a_s$. 
    %if this were the case, then \eqref{linear-system-974} would imply that $a_{pd}\cdot x_2 = 0$, which would imply that $a_{pd}$ lies in the same plane as $a_1,...,a_{d-1}$, again contradicting the linear independence property of $a_1,...,a_s$. Finally, suppose that $2 < i\neq j$. Since $x_i,x_j\neq 0$, if $x_i = x_j$, the equation \eqref{linear-system-974} would imply that $a_{pd},...,a_{pd+d-2}$ and $a_{qd},...,a_{qd+d-2}$ span the same plane for $p=i-2$ and $q=j-2$. This would again contradict the linear independence property of $a_1,...,a_s$.

    To complete the construction, we now choose for each block $\{a^q_1,....,a^q_d\}$ a vector $x_{v+2+q}$ which is orthogonal to $a^q_1,...,a^q_d$ and is not equal to any of the previously chosen vectors. Such a vector can be chosen since the block $a^q_1,....,a^q_d$ is linearly dependent (and thus lies in a $(d-1)$-dimensional subspace) by construction. It is clear that for an element $a^q_k$ in block $q$ we can now choose $i=1$ and $j=v+2+q$ to ensure that $a^q_k\cdot x_i = a^q_k\cdot x_j$.

    To complete the proof, we count the number of vectors produced by this construction, which is $v+2+t$. Since $s = d-1+dv$ and $s+td = k = d(n-1)-1$, we see that $ (v+t+1)d - 1 = d(n-1) - 1$, from which it easily follows that $n = v+2+t$ as required.

\end{proof}

\mumonlem*
\begin{proof}
    We begin by noting that by definition the stabilizer of any $X\in \Rdistinct$ is trivial. In addition it is clear by construction that the frame $\mumon$ is equivariant. 
    
    To verify continuity, we let $X^k\rightarrow X$ be a convergent sequence in $\Rdistinct$, with limit $X\in \Rdistinct$. Since the stabilizer of $X$ is trivial, we must show that $\mumon_{X^k}\rightarrow \mumon_{X}$ in the weak topology. Since these are probability distributions on a finite set, this means that for each permutation $\tau\in S_n$ we have $\mumon_{X^k}(\tau)\rightarrow \mumon_{X}(\tau)$.

    First, we note that the denominator in \eqref{CWW-separated-frame} satisfies
    \begin{equation}
        \sum_{j=1}^{m} \tilde{w}_j(X) > 0
    \end{equation}
    since the collection $a_i$ is globally separated.
    It is clear that $\tilde{w}_i(X)$ are continuous functions of $X$, so we have that in a sufficiently small neighborhood of $X$
    \begin{equation}\label{denominator-lower-bound-1689}
        \sum_{j=1}^{m} \tilde{w}_j(X) > \epsilon
    \end{equation}
    for some $\epsilon > 0$.
    
    Let $\tau\in S_n$ be arbitrary and consider the set
    \begin{equation}
        I(\tau,X) = \{i:~a_i^Tx_{\tau(1)} < a_i^Tx_{\tau(2)} < \ldots < a_i^Tx_{\tau(n)}\}
    \end{equation}
    where $x_i$ are the columns of $X$. $I(\tau,X)$ is the set of indices $i$ such that $\tau = g_i(X)$ and $X$ is strictly separated in the direction $a_i$.
    %First, suppose that there $\mumon_{X}(\tau) > 0$. This means that there exists at least one direction $a_i$ such that
    %\begin{equation}\label{condition-on-a-i-1695}
    %    a_i^Tx_{\tau(1)} < a_i^Tx_{\tau(2)} < \ldots < a_i^Tx_{\tau(n)},
    %\end{equation}
    %where $x_i$ are the columns of $X$, i.e. such that $\tau = g_i(X)$ for a direction $a_i$ in which $X$ is strictly separated. Thus the set
    
    %is non-empty. 
    From the definition of $\mumon$ it follows that
    \begin{equation}
        \mumon_{X}(\tau^{-1}) = \sum_{i\in I(\tau,X)}w_i(X) = \frac{\sum_{i\in I(\tau,X)}\tilde{w}_i(X)}{\sum_{j=1}^{m} \tilde{w}_j(X)}.
    \end{equation}
    Since $X^k\rightarrow X$, there exists a sufficiently large $N$ such that for all $k > N$ and all $i\in I(\tau,X)$ we have
    \begin{equation}
        a_i^Tx^k_{\tau(1)} < a_i^Tx^n_{\tau(2)} < \ldots < a_i^Tx^k_{\tau(n)},
    \end{equation}
    where $x_i^k$ are the columns of $X^k$,
    i.e. $I(\tau,X)\subset I(\tau,X^k)$. Thus, for $n > N$ we have
    \begin{equation}
        \mumon_{X^k}(\tau^{-1}) = \sum_{i\in I(\tau,X^k)}w_i(X^k) \geq\sum_{i\in I(\tau,X)} w_i(X^k) = \frac{\sum_{i\in I(\tau,X)}\tilde{w}_i(X^k)}{\sum_{j=1}^{m} \tilde{w}_j(X^k)}
    \end{equation}
    Combined with the continuity of $\tilde w_i$ and \eqref{denominator-lower-bound-1689}, this implies that we have
    \begin{equation}
        \lim\inf_{k\rightarrow \infty}\mumon_{X^k}(\tau^{-1}) \geq \lim_{k\rightarrow \infty} \frac{\sum_{i\in I(\tau,X)}\tilde{w}_i(X^k)}{\sum_{j=1}^{m} \tilde{w}_j(X^k)} = \frac{\sum_{i\in I(\tau,X)}\tilde{w}_i(X)}{\sum_{j=1}^{m} \tilde{w}_j(X)} = \mumon_{X}(\tau^{-1}).
    \end{equation}

    Since $\mumon_{X^k}$ and $\mumon_X$ are probability distributions, i.e. we have
    $$
        \sum_{\tau\in S_n} \mumon_{X^k}(\tau) = 1 = \sum_{\tau\in S_n} \mumon_{X}(\tau)
    $$
    it follows that we must actually have $\lim_{k\rightarrow \infty}\mumon_{X^k}(\tau) = \mumon_{X}(\tau)$ for all $\tau$.
\end{proof}

\mubigprop*
\begin{proof}

We first prove the stated bound on the cardinality of the frame $\mu^{S_n}$. Let $X\in \mathbb{R}^{d\times n}$. We can remove the measure $0$ set of directions $a$ for which $a^Tx_i = a^Tx_j$ for some pair of distinct columns $x_i$ and $x_j$ of $X$ . Thus, we need to bound the cardinality of the set
\begin{equation}\label{set-of-permutations-considered}
    \{\text{argsort}(a^TX):~a\in S(X)\},
\end{equation}
where the set $S(X)$ is given by
$$
    S(X) = \{a\in S^{d-1},~a^Tx_i = a^T x_j\implies x_i = x_j\}.
$$
Note that $\text{argsort}(a,X) \neq \text{argsort}(a',X)$ for $a,a'\in S(X)$ implies that there are two distinct columns $x_i\neq x_j$ of $X$ such that $$a\cdot (x_i - x_j) < 0 < a'\cdot (x_i - x_j).$$
Thus, the cardinality of the set in \eqref{set-of-permutations-considered} is equal to the number of regions that the hyperplanes $(x_i - x_j)\cdot a = 0$ divide the space $\RR^d$ into (where $x_i,x_j$ run over all distinct pairs of columns $x_i\neq x_j$ of $X$). 

The total number of these hyperplanes is at most $\binom{n}{2}$. Next, we claim that given $N$ hyperplanes in $\mathbb{R}^d$ passing through the origin, the number of regions that the space $\RR^d$ is divided into is bounded by
\begin{equation}\label{binomial-sum}
    2\sum_{k=0}^{d-1}\binom{N-1}{k}.
\end{equation}
We prove this by induction on both $N$ and $d$. It is clearly true for $N = 1$, since in this case the sum in \eqref{binomial-sum} is $2$ (we interpret $\binom{n}{k} = 0$ if $k > n$) and a single hyperplane divides the space into two pieces. Also, if $d=1$, then the sum in \eqref{binomial-sum} is also always $2$, and in one dimension there is only one `hyperplane' through the origin (i.e. the points $0$ itself), which divides the space into two pieces.

Now, consider the case $N > 1$ and $d > 1$. Suppose that the first $N-1$ hyperplanes divide $\RR^d$ into $M$ pieces. The number of these pieces which is divided into two by the $N$-th hyperplane is equal to the number of pieces that the $N$-th hyperplane (which is a $d-1$ dimensional space) is divided into by its intersections with the first $N-1$-hyperplanes. Thus, letting $M(N,d)$ denote the maximum number of pieces that $N$ hyperplanes can divide $\RR^d$ into, we obtain the recursive relation
$$
    M(N,d) \leq M(N-1,d) + M(N-1,d-1).
$$
Combining this with the base cases discussed above gives the bound \eqref{binomial-sum}. Plugging in $N = \binom{n}{2}$ gives the claimed bound on the cardinality of $\mu^{S_n}$.

Finally, we prove that $\mu^{S_n}$ is a \CWW frame.  We need to show that this frame is weakly equivariant and continuous. For both of these computations we will first prove that the averaged measure $\bar  \mu^{S_n}_X$ is given by  
\begin{equation}\label{averaged-measure-representation}
    \bar \mu^{S_n}_X(g^{-1})= \frac{1}{|\stab{X}|}\mathbb{P}_{a\sim S^{d-1}}\left(a^Tx_{g(1)}\leq \cdots \leq a^Tx_{g(n)}\right).
\end{equation}
Indeed, for any $g\in S_n$ we have
$$
    \mathbb{P}_{a\sim S^{d-1}}\left(a^Tx_{g(1)}\leq \cdots \leq a^Tx_{g(n)}\right) = \mathbb{P}(\exists h\in G_X,~gh = \text{argsort}(a^TX)).
$$
Moreover, an $h\in G_X$ such that $gh = \text{argsort}(a^TX))$, if it exists, must be unique by construction. This means that
\begin{equation}
    \bar \mu^{S_n}_X(g^{-1}) = \frac{1}{|G_X|}\sum_{h\in G_X} \mu^{S_n}_X(h^{-1}g^{-1}) = \frac{1}{|G_X|}\mathbb{P}(\exists h\in G_X,~gh = \text{argsort}(a^TX)),
\end{equation}
which proves \eqref{averaged-measure-representation}.

It is clear from \eqref{averaged-measure-representation} that $\bar \mu^{S_n}_{hX}=h^*\bar \mu^{S_n}_X$ for all $h\in S_n$, since 
$$
\bar \mu^{S_n}_{hX}(g^{-1}) = \frac{1}{|\stab{hX}|}\mathbb{P}_{a\sim S^{d-1}}\left(a^Tx_{gh(1)}\leq \cdots \leq a^Tx_{gh(n)}\right),
$$
while
$$
h^*\bar \mu^{S_n}_X(g^{-1}) = \bar \mu^{S_n}_X(h^{-1}g^{-1}) = \frac{1}{|\stab{X}|}\mathbb{P}_{a\sim S^{d-1}}\left(a^Tx_{gh(1)}\leq \cdots \leq a^Tx_{gh(n)}\right),
$$
and $G_X$ and $G_{hX}$ are conjugate and thus have the same order. All that remains is to verify the continuity.

Let $X\in \mathbb{R}^{d\times n}$ and suppose that $X^k\rightarrow X$. Let $\stab{X}$ denote the stabilizer of $X$. By definition, we  need to show that as $k\rightarrow \infty$
$$
    \left\langle\mu^{S_n}_{X^k}\right\rangle_X\rightarrow \bar\mu^{S_n}_X
$$
weakly. Since $S_n$ is a finite group, this means that we must show that for every $g\in S_n$, we have
\begin{equation}\label{need-to-prove-1916}
    \left\langle\mu^{S_n}_{X^k}\right\rangle_X(g)\rightarrow \bar\mu^{S_n}_X(g)
\end{equation}
as $k\rightarrow \infty$. 

Let $\delta_X > 0$ denote the smallest distance between two non-equal columns of $X$. We observe that if $|Y-X| < \delta_X/2$, then any two columns which are non-equal in $X$ must also be non-equal in $Y$, so that $\stab{Y}\subset \stab{X}$. Thus, for sufficiently large $k$, we have $\stab{X^k}\subset \stab{X}$. This implies that
$$
    \left\langle\mu^{S_n}_{X^k}\right\rangle_X = \left\langle\bar\mu^{S_n}_{X^k}\right\rangle_X
$$
for sufficiently large $k$, so that in \eqref{need-to-prove-1916} we can replace $\langle\mu^{S_n}_{X^k}\rangle_X$ by $\langle\bar\mu^{S_n}_{X^k}\rangle_X$.

Finally, since both $\langle \bar{\mu}^{S_n}_{X^k} \rangle_X$ and $\bar\mu^{S_n}_X$ are probability distributions, it suffices to show that for every $g\in S_n$ and $\epsilon > 0$
\begin{equation}
    \langle \bar{\mu}^{S_n}_{X^k} \rangle_X(g) \geq \bar\mu^{S_n}_X(g) - \epsilon
\end{equation}
for sufficiently large $k$.

Using the definition of $\mu_{X^k}^{S_n}$ and averaging over the stabilizers $\stab{X}$ and $\stab{X_k}$ we get
\begin{equation}\label{equation-for-averaged-mu-S_n}
    \langle \bar{\mu}^{S_n}_{X^k} \rangle_X(g^{-1}) = \frac{1}{|\stab{X^k}||\stab{X}|} \sum_{h\in \stab{X}} \mathbb{P}_{a\sim S^{d-1}}\left(a^Tx^k_{gh(1)}\leq \cdots \leq a^Tx^k_{gh(n)}\right),
\end{equation}
while
\begin{equation}\label{equation-for-averaged-mu-X-2123}
    \bar\mu^{S_n}_X(g^{-1}) = \frac{1}{|\stab{X}|}\mathbb{P}_{a\sim S^{d-1}}\left(a^Tx_{g(1)}\leq \cdots \leq a^Tx_{g(n)}\right).
\end{equation}
For any $\delta > 0$ consider the event
\begin{equation}
    A(\delta,X,g) := \left\{a\in S^{d-1}:~a^Tx_{g(i)} + \delta \leq a^Tx_{g(j)},~\forall i > j~\text{and}~x_{g(i)} \neq x_{g(j)}\right\}.
\end{equation}
In words, $A(\delta,X,g)$ is the set of directions $a\in S^{d-1}$ for which $a^T(gX)$ is sorted and all non-equal columns of $X$ differ by at least $\delta$. Observe that
\begin{equation}
    \left\{a\in S^{d-1}:~a^Tx_{g(1)}\leq \cdots \leq a^Tx_{g(n)}\right\} - \bigcup_{\delta > 0} A(\delta,X,g)
\end{equation}
contains only directions $a$ such that $a^Tx_i = a^Tx_j$ for non-equal columns $x_i\neq x_j$ of $X$, and thus is a set of measure $0$. This means that
\begin{equation}
    \lim_{\delta\rightarrow 0}\mathbb{P}(A(\delta,X,g)) = \mathbb{P}_{a\sim S^{d-1}}\left(a^Tx_{g(1)}\leq \cdots \leq a^Tx_{g(n)}\right).
\end{equation}
Thus, in light of \eqref{equation-for-averaged-mu-S_n} and \eqref{equation-for-averaged-mu-X-2123} it thus suffices to show that for any $\delta > 0$ we have
\begin{equation}\label{probability-sum-eq-1768}
    \frac{1}{|\stab{X^k}|} \sum_{h\in G_X} \mathbb{P}_{a\sim S^{d-1}}\left(a^Tx^k_{gh(1)}\leq \cdots \leq a^Tx^k_{gh(n)}\right) \geq \mathbb{P}(A(\delta,X,g))
\end{equation}
for sufficiently large $k$. To show this, suppose that $a\in A(\delta,X,g)$, i.e. $a\in S^{d-1}$ satisfies
\begin{equation}\label{condition-on-a-1772}
    a^Tx_{g(i)} + \delta \leq a^Tx_{g(j)},~\forall i > j~\text{and}~x_{g(i)} \neq x_{g(j)}.
\end{equation}
Choose $k$ large enough so that for all columns $i$, we have $|x^k_i - x_i| < \delta / 4$. Then we will also have
\begin{equation}\label{modified-condition-on-a}
    a^Tx^k_{g(i)} + \delta/2 \leq a^Tx^k_{g(j)},~\forall i > j~\text{and}~x_{g(i)} \neq x_{g(j)}.
\end{equation}
This means that by permuting only columns in $X^k$ which are equal in $X$, we can sort $a^T(gX^k)$ (since any columns which are not equal in $X$ are already sorted in $a^T(gX^k)$ by \eqref{modified-condition-on-a}), so that there exists an $h\in \stab{X}$ such that
\begin{equation}
    a^Tx^k_{gh(1)}\leq \cdots \leq a^Tx^k_{gh(n)}.
\end{equation}
Moreover, $h$ is obviously only unique up to multiplication by an element of $\stab{X^k}$ (which is contained in $\stab{X}$ since by \eqref{modified-condition-on-a}, $x_i\neq x_j$ implies $x^k_i \neq x^k_j$). So each $a\in A(\delta,X,g)$ is contained in at least $|\stab{X^k}|$ of the sets
\begin{equation}
    \left\{a\in S^{d-1}:~a^Tx^k_{gh(1)}\leq \cdots \leq a^Tx^k_{gh(n)}\right\}~\text{for $h\in \stab{X}$.}
\end{equation}
This implies \eqref{probability-sum-eq-1768} and completes the proof.

\end{proof}

\mutwo*
\begin{proof}
To prove weak equivariance we need to show that $\mu=\muTwo$ satisfies
$$\bar{\mu}_{gZ}=g_*\bar{\mu}_Z, \forall g\in S^1, Z\in \CC^n $$
For $Z=0$ both sides of this equation are the zero measure. For $Z \neq 0$ we note that $w_i(gZ)=w_i(Z)$ for all $g\in S^1$, that  $g_i(gZ)=g\cdot g_i(Z)$ when $z_i \neq 0$, while when $z_i=0$ we have $w_i(Z)=0$, so that overall we obtain
\begin{align*}\mu_{gZ}&=\sum_{i: z_i \neq 0} w_i(gZ)\delta_{g_i(gZ)}\\ 
&=\sum_{i: z_i \neq 0} w_i(Z)\delta_{g\cdot g_i(Z)}\\
&=g_*\mu_Z
\end{align*}

 Finally we need to show that $\mu=\muTwo$ is a continuous weakly equivariant weighted frame. If $Z_n \rightarrow Z$ and $Z\neq 0$ it is straightforward to see that $\mu_{Z_n}\rightarrow \mu_Z $. If $Z_n\rightarrow 0$ then we note that for any probability measure $\mu$ on $G$ we have that $\langle \mu \rangle_0=0=\langle \mu_0 \rangle_0$, and so in particular  $\langle \mu_{Z_n} \rangle_0=\langle \mu_0 \rangle_0$ and $\mu_{Z_n}-\mu_{Z_n}=0$ converges weakly to $0$ so we are done.
\end{proof}

\stable*

\begin{proof}
Let $f:V\to W$ be a stable function. The stability assumption essentially allows us to reconstruct the proof from the invariant case, replacing $\Iw$ with $\Ew$. The critical observation is that if $f$ is stable, then $s$ stabilizes $gf(g^{-1}v) $ which occurs in the definition of $\Ew$. This is because  in general 
$$\stab{gx}\subseteq g \stab{x}g^{-1} $$
so
$$\stab{gf(g^{-1}v)}= g\stab{f(g^{-1}v)}g^{-1}\supseteq  g\stab{g^{-1}v}g^{-1}=\stab{v}.$$
as a result, we can replace $\mu_v$ with  $\bar{\mu}_v$ in the definition of $\Ew[f]$, because
\begin{align*}
\Ew[f](v)&=\int_G gf(g^{-1}v)d\mu_v(g)  \nonumber\\
&=\int_{\stab{v}} \int_G sgf((sg)^{-1}v)d\mu_v(g)ds  \nonumber\\
&=\int_{\stab{v}} \int_G gf(g^{-1}v)ds_*\mu_v(g)ds  \nonumber\\
&=\int_G gf(g^{-1}v)d\bar{\mu}_v(g) 
\end{align*}

Using this we can now show that $\Ew[f]$ is an equivariant function because for every $h\in G$ and $v\in V$,
\begin{align*}
\Ew[f](hv)&=\int_G gf(g^{-1}hv)d\mu_{hv}(g)  \\
&=\int_G gf(g^{-1}hv)d\bar{\mu}_{hv}(g)\\
&=\int_G gf(g^{-1}hv)dh_*\bar{\mu}_{v}(g)\\
&=\int_G hgf((hg)^{-1}hv)d\bar{\mu}_{v}(g)\\
&=h\left[\int_G gf(g^{-1}v)d\bar{\mu}_{v}(g) \right]\\
&=h\left[\int_G gf(g^{-1}v)d\mu_{v}(g)\right]\\
&=h\Ew[f](v).
\end{align*}
If $f$ is equivariant then $\Ew[f]=f$ since $gf(g^{-1}v)=f(v)$ from equivariance.

 Boundedness also follows easily since $d\mu_v$ is a probability measure, so that on every compact $K\subseteq V$ which is also closed under the action of $G$, we have for every $v\in K$ that
 \begin{equation}
     |\Ew[f](v)| \leq \left|\int_G gf(g^{-1}v)d\mu_v(g)\right| \leq \max_{g\in G}\|g\|\max_{w\in K}|f(w)|,
 \end{equation}
where $\|g\|$ denotes the operator norm of the linear operator $g$ which is bounded due to the fact that $G$ is compact acting linearly and continuously. 

 \textbf{Continuity}
 Let $f:V\to \RR$ be a continuous functions. Let $v_k$ be a sequence converging to $v\in V$. We need to show that $\Ew[f](v_k)$ converges to $\Ew[f](v)$. We observe that
\begin{align*}
&|\Ew[f](v_k)-\Ew[f](v)|\\
&=\left|\int gf(g^{-1}v_k)d\mu_{v_k}(g)-\int gf(g^{-1}v)d\mu_{v}(g)\right|  \\
&\leq\left|\int gf(g^{-1}v_k)d\mu_{v_k}(g)-\int gf(g^{-1}v)d\mu_{v_k}(g)\right| + \left|\int gf(g^{-1}v)d\mu_{v_k}(g)-\int gf(g^{-1}v)d\bar{\mu}_{v}(g)\right|.
\end{align*}
Since $f$ is stable we have
\begin{equation}
    \int gf(g^{-1}v)d\mu_{v_k}(g)-\int gf(g^{-1}v)d\langle\mu_{v_k}\rangle_v(g) = 0.
\end{equation}
and by assumption $\langle\mu_{v_k}\rangle_v \rightarrow \bar{\mu}_{v}$ weakly  the second term tends to $0$ (since $f$ and the action of $G$ are continuous). The first term tends to zero because $gf(g^{-1}v_k) $ converges to $gf(g^{-1}v) $ uniformly in $g$ as $k$ tends to infinity.

%NOTE TO SELF. COMPARE THIS PROOF TO THE INVARIANT CASE

\end{proof}

\end{document}